\documentclass[11pt]{article}
\usepackage[margin=1in]{geometry}
\usepackage{multirow}
\usepackage{makecell}
\usepackage[utf8]{inputenc} 
\usepackage[T1]{fontenc}    
\usepackage[colorlinks,
linkcolor=red,
anchorcolor=blue,
citecolor=blue
]{hyperref}
\makeatletter
\newcommand*{\rom}[1]{\expandafter\@slowromancap\romannumeral #1@}
\makeatother

\usepackage{url}            
\usepackage{booktabs}       
\usepackage{nicefrac}       
\usepackage{microtype}      
\usepackage{xcolor}         
\usepackage{amsmath,amsfonts,amsthm,amssymb,bm,verbatim,dsfont,mathtools}
\usepackage{algorithmic}
\usepackage[ruled,vlined,linesnumbered]{algorithm2e}
\usepackage{color,graphicx,appendix}
\usepackage{bbm}
\usepackage{subfig}
\usepackage{etoolbox}
\usepackage{tikz}
\usepackage{xr,xspace}
\usepackage{todonotes}
\usepackage{paralist}
\usepackage{tabularx}
\usepackage[labelfont=bf,format=plain,justification=centering,singlelinecheck=false,]{caption}
\usepackage{soul}

\newcommand{\iiddistr}{{\stackrel{\text{\iid}}{\sim}}}

\newcommand{\reals}{{\mathbb{R}}}

\newcommand{\red}[1]{\textcolor{red}{{#1}}}

\newcommand{\nbr}[1]{\red{\sf[#1]}}
\newcommand{\nbo}[1]{{\color{orange}{\sf[#1]}}}
\newcommand{\ls}[1]{\nbo{LS: #1}}

\newcommand{\var}{\mathrm{Var}}
\newcommand{\Cov}{\text{Cov}}

\newcommand{\ie}{i.e.\xspace}
\newcommand{\iid}{i.i.d.\xspace}

\newcommand{\expect}[1]{\mathbb{E}\left[ #1 \right]}

\newcommand{\expects}[2]{\mathbb{E}_{#2}\left[ #1 \right]}

\newcommand{\diag}[1]{\mathsf{diag} \left( {#1} \right) }

\newcommand{\RN}[1]{%
  \textup{\uppercase\expandafter{\romannumeral#1}}%
}

\DeclareFontFamily{U}{mathx}{}
\DeclareFontShape{U}{mathx}{m}{n}{<-> mathx10}{}
\DeclareSymbolFont{mathx}{U}{mathx}{m}{n}
\DeclareMathAccent{\widehat}{0}{mathx}{"70}
\DeclareMathAccent{\widecheck}{0}{mathx}{"71}

\newif\ifqedplaced

\ifdefined\proof
  \renewenvironment{proof}[1][Proof]{%
    \par\noindent{\textbf{#1.}\ }%
    \qedplacedfalse 
  }{%
    \ifqedplaced\else \hfill\qed\fi
    \\[2mm]%
  }
\else
  \newenvironment{proof}[1][Proof]{%
    \par\noindent{\textbf{#1.}\ }%
    \qedplacedfalse 
  }{%
    \ifqedplaced\else \hfill\qed\fi
    \\[2mm]%
  }
\fi

\newtheorem{theorem}{Theorem}[section]
\newtheorem{lemma}{Lemma}[section]
\newtheorem{proposition}{Proposition}[section]
\newtheorem{corollary}{Corollary}[section]
\theoremstyle{definition}
\newtheorem{definition}{Definition}[section]

\newtheorem{assumption}{Assumption}[section]

\newtheorem{example}{Example}[section]

\usepackage{style}

\title{
Learning with Shared Representations: Statistical Rates and Efficient Algorithms}
\author{Xiaochun Niu \and Lili Su \and Jiaming Xu \and Pengkun Yang \thanks{X. Niu and J. Xu are with
 the Fuqua School of Business, Duke University, \texttt{\{xiaochun.niu,jx77\}@duke.edu}. L. Su is with the Department of Electrical and Computer Engineering, 
Northeastern University,
\texttt{l.su@northeastern.edu}. P. Yang is with the Department of Statistics and Data Science,
Tsinghua University,
\texttt{yangpengkun@tsinghua.edu.cn}.}}

\date{}

\begin{document}

\maketitle

\begin{abstract}
Collaborative learning through latent shared feature representations enables heterogeneous clients to train personalized models with improved performance and reduced sample complexity. Despite empirical success and extensive study, the theoretical understanding of such methods remains incomplete, even for representations restricted to low-dimensional linear subspaces.
In this work, we establish new upper and lower bounds on the statistical error in learning low-dimensional shared representations across clients. Our analysis captures both statistical heterogeneity (including covariate and concept shifts) and variation in local dataset sizes, aspects often overlooked in prior work. We further extend these results to nonlinear models including logistic regression and one-hidden-layer ReLU networks.
 
Specifically, we design a spectral estimator that leverages independent replicas of 
local averages to approximate the non-convex least-squares solution and derive a nearly matching minimax lower bound. Our estimator achieves the optimal statistical rate when the shared representation is well covered across clients---i.e., when no direction is severely underrepresented. Our results reveal two distinct phases of the optimal rate: a standard parameter-counting regime and a penalized regime when the number of clients is large or local datasets are small. These findings precisely characterize when collaboration benefits the overall system or individual clients in transfer learning and private fine-tuning.   
\end{abstract}

\section{Introduction}\label{sect:intro}

Modern machine learning and data science often involve heterogeneous datasets collected from multiple related sources, such as different devices, organizations, or tasks. 
This has motivated extensive research on learning shared feature representations that capture the underlying common structure across these sources. 
By leveraging such shared representations, individual models can be trained with far fewer samples than if each were learned independently from scratch. This framework has many applications, including multi-task and transfer learning \citep{caruana1997multitask,duan2023adaptive}, private fine-tuning with public knowledge \citep{thaker2023leveraging}, and personalized federated learning \citep{fallah2020personalized,collins2021exploiting,even2022sample}.

Despite its success, the theoretical understanding of learning shared representations---particularly regarding statistical error rates---remains incomplete. Fundamental questions persist: How can we design optimal methods to learn shared representations from data, and under what conditions do they outperform independently trained models? These challenges remain open even for low-dimensional linear subspace models. 

We consider a widely studied model with $M$ clients (or tasks). Each client $i$ observes $n_i$ independent data samples $\{(x_{ij}, y_{ij})\}_{j=1}^{n_i}$, where $x_{ij} \in \R^{d}$ is the covariate and $y_{ij}\in\R$ is the response, generated from a linear model with local parameter $\theta^\star_i \in \R^d$:
\#\label{eq:model-sup}
\E[y_{ij}\given x_{ij}] = x_{ij}^\intercal \theta_{i}^\star \,.
\#
In many modern settings, the dimension $d$ far exceeds the local data size $n_i$. Such high dimensionality makes it impossible to reliably estimate $\theta_i^\star$ from local data alone.
To address this, 
we assume a shared low-dimensional structure. Let $\Gamma_i = \E[x_{ij}x_{ij}^\intercal]$ be the \emph{unknown} covariance matrix. 
Suppose there exists an orthonormal matrix $B^\star \in \cO^{d\times k}$ with $k\le d$, representing the shared low-dimensional subspace,
and client-specific $\alpha_i^\star \in \R^k$ such that
\#\label{eq:low-d-structure}
\Gamma_i\theta_i^\star = B^\star\alpha_i^\star.
\#
Here different $\Gamma_i$'s and $\theta_i^\star$ allow for potential covariate and concept shifts~\citep
{kairouz2021advances}. This is also closely related to the classic factor model in statistics and econometrics~\citep{bai2008large}. 
The clients aim to collaboratively learn the shared low-dimensional representation $B^\star$ from their datasets $\{\{(x_{ij}, y_{ij})\}_{j=1}^{n_i}\}_{i=1}^M$. 
If $B^\star$ is learned accurately, each $x_{ij}$ can be projected onto the shared subspace, which reduces the estimation problem from the $d$-dimensional $\theta_i^\star$ to the $k$-dimensional $\alpha_i^\star$; thus addresses high-dimensional challenges. 

It is worth noting that the learnability of $B^\star$ depends on the information contained in $\{(x_{ij}, y_{ij})\}$ for each of its $k$ columns. Intuitively, the data provide information about $B^\star$ along the directions $\alpha_i^\star$, with each direction supported by $n_i$ samples. Hence, a key factor determining the difficulty of learning $B^\star$ is the spectrum of the matrix 
\#\label{eq:diversity-matrix}
D = \frac{1}{N}\sum_{i=1}^M n_i \alpha_i^\star (\alpha_i^\star)^\intercal,
\# 
where $N=\sum_{i=1}^M n_i$ is the total number of samples. This matrix $D$, known as the \emph{client diversity matrix} \citep{du2021fewshot,tripuraneni2020theory,tripuraneni2021provable,collins2021exploiting,thekumparampil2021statistically,tian2023learning,zhang2024sample}, captures the diversity of the weighted client-specific parameters $\sqrt{n_i}\alpha_i^\star$ across clients. Let $\lambda_r$ be the $r$-th largest eigenvalue of $D$ for $r\in [k]$. 
Under the standard normalization~\citep{tripuraneni2021provable,duchi2022subspace,tian2023learning} $\|\alpha_i^\star\|=O(1)$ for all $i$, 
we have  $\sum_{r=1}^k\lambda_r = O(1)$. 
The problem of estimating $B^\star$ is considered \emph{well-represented} if the condition number $\lambda_1/\lambda_k=\Theta(1)$, meaning that ${\sqrt{n_i}\alpha_i^\star}$ are evenly distributed and the data provides sufficient information along all $k$ columns of $B^\star$. 
In contrast, the problem is ill-represented when $\sqrt{n_i}\alpha_i^\star$ are concentrated in a few directions and $D$ has a large condition number. Then the data provides limited information in other directions, making it hard to estimate all columns of $B^\star$.

Prior studies \citep{tripuraneni2021provable,du2021fewshot,collins2021exploiting,thekumparampil2021statistically,chua2021fine,duchi2022subspace,duan2023adaptive,tian2023learning,zhang2024sample} have analyzed the statistical error rates for estimating $B^\star$, measured by the principal angle distance (Definition~\ref{def:principal_angle}).
Except for \cite{tripuraneni2021provable}, these works focus on equal data partitions $n_i\equiv n$.
Table \ref{tab:bounds_comparison} summarizes the state-of-the-art (SOTA) bounds in both general and well-represented cases.\footnote{We focus on the specific well-represented case $\lambda_1$$=\Theta(\lambda_k)=\Theta(1/k)$, which naturally arises when $\|\alpha_i\| = \Theta(1)$ for all $i$, implying $\sum_{r=1}^k \lambda_r = \tr(D) = \Theta(1)$. For clarity, we ignore polylogarithmic factors and use big-$O$ notation in Section \ref{sect:intro}, while later sections use $\widetilde{O}$ to emphasize the hidden polylogarithmic factors.}   
However, substantial gaps remain between the best-known upper and lower bounds. 
And even in well-represented cases, 
existing results have suboptimal dependence on the subspace dimension $k$.  These gaps have been acknowledged in many works~\citep{tripuraneni2021provable,thekumparampil2021statistically,thaker2023leveraging,tian2023learning,zhang2024sample} as a challenging open problem: What is the optimal statistical rate to learn the low-dimensional representation $B^\star$?

\begin{table}[!ht]
\centering
\renewcommand{\arraystretch}{1.4} 
\setlength{\tabcolsep}{10pt} 
\begin{tabular}{|c|c|c|c|}
\hline
&  &  \textbf{\small SOTA}
 & \textbf{\small This Work} \\
\hline
\multirow{2}{*}{\textbf{\small General Cases}} 
&  \small Upper Bound & \small
$\sqrt{\frac{d}{N\lambda_k^2}}$ 
& \small $\sqrt{\frac{d\lambda_1}{N\lambda_k^2}} + \sqrt{\frac{Md}{N^2\lambda_k^2}}$ \\ 
\cline{2-4}
& \small Lower Bound & 
\small $\sqrt{\frac{1}{N\lambda_k}} + \sqrt{\frac{dk}{N}}$ 
&\small  $\sqrt{\frac{d}{N\lambda_k}} + \sqrt{\frac{Md}{N^2\lambda_k^2}}$ \\
\hline
\multirow{2}{*}{
 \small \makecell{\textbf{Well-Represented}\\
 $\lambda_1=\Theta(\lambda_k)=\Theta(1/k)$
 }
} 
& \small Upper Bound & \small $\sqrt{\frac{dk^2}{N}}$ & \small \multirow{2}{*}{$\sqrt{\frac{dk}{N}} + \sqrt{\frac{Mdk^2}{N^2}}$}  \\
\cline{2-3}
 & \small Lower Bound & 
 \small $\sqrt{\frac{dk}{N}}$
& {}  \\ 
\hline
\end{tabular}
\vspace{0.8em}
\caption{\small Comparison of estimation error bounds: the SOTA upper and lower bounds are established by \cite{tripuraneni2021provable} under the additional assumption that $x_{ij} \iiddistr N(0, I_d)$.}
\label{tab:bounds_comparison}
\end{table}

\subsection{Main Results}
Our work tackles this challenging open problem by improving both the upper and lower bounds on the error rate for estimating $B^\star$. 
Table \ref{tab:bounds_comparison} highlights our new results. 
The derived upper bound improves over the SOTA by $\sqrt{\min\{\lambda_1, M/N\}}$. Recall that $\lambda_1\le \sum_{r=1}^k \lambda_r = O(1)$ due to normalization. Compared with the SOTA, our lower bound achieves a gain of $\sqrt{d}$ in the first term and introduces a new second term. Notably, the lower and upper bounds differ only by a condition number $\sqrt{\lambda_1/\lambda_k}$ in the first term. Hence, our results successfully identify the optimal statistical rate in well-represented cases. Importantly, the well-represented condition is not overly restrictive and holds in many scenarios with heterogeneous data partitions, as shown in Example \ref{lem:well-cond}. These optimal rates also characterize when collaboration benefits the overall system or individual clients in transfer learning and private fine-tuning.   

Moreover, our results apply to heterogeneous $n_i$, relaxing the equal-partition assumption $n_i\equiv n$ commonly made in prior work. Our analysis requires only $n_i\ge 2$ for each client,\footnote{Note that in the extreme case where every client has only a single data point, \ie, $n_i\equiv 1$, the existing error bound $O(\sqrt{d/(N\lambda_k^2)})$ in~\cite{tripuraneni2021provable,duchi2022subspace} already matches our improved lower bound in Theorem \ref{cor:lb-sup}.  Hence, assuming 
$n_i\ge 2$ does not result in a significant loss of generality.}
highlighting the advantage of learning shared representations particularly when individual clients have limited data. Even with small $n_i$, clients can still collaboratively estimate $B^\star$ despite lacking sufficient samples to accurately learn $\theta_i^\star$ individually. This significantly relaxes the strict assumption $n_i\gg d$ required in \cite{du2021fewshot,chua2021fine,duan2023adaptive,tian2023learning,zhang2024sample}, where each client can already learn $\theta_i^\star$ independently. 
Finally, we extend our results to general non-linear models, such as logistic regression and one-hidden-layer ReLU networks, and show that the same error upper bound in Table~\ref{tab:bounds_comparison} holds for such models with standard Gaussian covariates.

To further illustrate the sharpness of our results, 
Figure \ref{fig:diag} plots the phase diagram for well-represented cases under a specific scaling parameterization: $n_i \equiv n=k^\beta$,  
$M=k^{\gamma+1}$, and $d=k^{\delta+1}$,
where $\beta,\gamma,\delta>0$ are fixed constants. We require $M \geq k$ to ensure that the client diversity matrix $D$ is full rank. 
Each region in the diagram indicates whether consistent estimation of $B^\star$ with vanishing error is achievable as $k\to\infty$.
\footnote{For the existing error upper bound $\sqrt{dk^2/N}$ to be $\Theta(1)$, we need $N=\Theta(dk^2)$.
This translates to $\beta + \gamma = \delta + 2$ with the specific parameterization. Similarly, the existing lower bound $\sqrt{dk/N}$ corresponds to $\beta + \gamma = \delta + 1$, and the second term in our derived error rate $\sqrt{Mdk^2/N^2}$ corresponds to $2\beta + \gamma = \delta + 2$.} Consistent estimation is impossible in light red Region \rom{1} and possible in light blue Region \rom{2}, as established in prior work. 
However, a wide gap remains between these two regions. Our contribution bridges this gap by identifying the optimal statistical rate, which delineates the boundary between Region \rom{4} and Regions \rom{1} and \rom{3}. Specifically, we prove that consistent estimation is impossible in dark red Region \rom{3} and possible in dark blue Region \rom{4}.

\begin{figure}[ht]
\centering\includegraphics[width=0.45\textwidth]{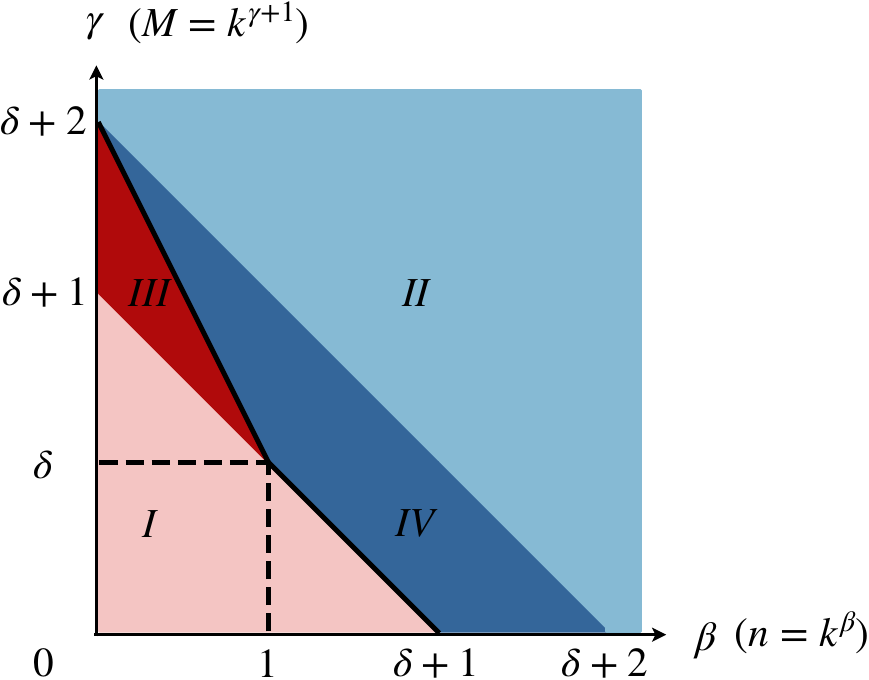}
       \caption{\small Phase diagram for estimating $B^\star$ in well-represented cases with a specific parameterization. 
   Regions \rom{1} and \rom{2} denote impossibility and possibility from prior work. Regions \rom{3} and \rom{4} represent our new results. 
The boundaries are given by
$\beta+\gamma=\delta + 1$ (between Region \rom{1} and Regions \rom{3} and \rom{4}),  $2\beta+\gamma = \delta+2$ (Regions \rom{3} and \rom{4}), and $\beta+\gamma = \delta+2$ (Regions \rom{2} and \rom{4}). 
       }
       \label{fig:diag}
   \end{figure}

The identified optimal rate $\Theta(\sqrt{dk/N} + \sqrt{Mdk^2/N^2})$ exhibits two phases. 
It matches the standard parameter-counting rate $\Theta(\sqrt{dk/N})$ (since $B^\star$ has $dk$ entries and $N$ samples are observed) along the boundary between Regions \rom{1} and \rom{4}. When $M=\Omega(d)$ or $n = O(k)$, the second term $\sqrt{Mdk^2/N^2}$ dominates, and the optimal rate defines the boundary between Regions \rom{3} and \rom{4}. 
 This second phase reveals a statistical penalty when there are too many clients or local datasets are too small. For example, fixing $d$, $k$ and $N$, as $M$ increases: 
if $M = O(N/k)$ (or $n =\Omega(k)$), the first term dominates and the error remains unchanged; if $M = \Omega(N/k)$ or $n =O(k)$, the second term dominates and the error grows with $M$. 

\subsection{Algorithmic and Analytical Innovations}

Several methods have achieved the SOTA upper bound in Table \ref{tab:bounds_comparison}, including nonconvex empirical risk minimization~\citep{du2021fewshot}, the
method-of-moments (MoM) estimator~\citep{tripuraneni2021provable} and its variant \citep{duchi2022subspace}, 
and 
alternating minimization initialized via MoM~\citep{collins2021exploiting,thekumparampil2021statistically,zhang2024sample}. 

We propose a new spectral estimator for $B^\star$ that achieves the improved error upper bound in Table \ref{tab:bounds_comparison}, making it the first to reach the optimal statistical rate in well-represented cases. 
The estimator is derived as an optimal solution to an approximate non-convex least-squares problem, 
which otherwise lacks a closed-form solution. It leverages two independent replicas of locally averaged cross-correlation vectors $\{y_{ij}x_{ij}\}_{j=1}^{n_i}$ at each client. Here local averaging reduces noise for a tighter bound and the replicas enable handling sub-Gaussian covariates with unknown covariances, 
unlike \cite{tripuraneni2021provable}, which assumes standard Gaussian covariates. Moreover, the estimator preserves privacy in federated learning, since clients share only local averages rather than any raw data to the server.

For the lower bound, we achieve a $\sqrt{d}$-improvement in the first term, using a packing set on the $d$-dimensional unit sphere to construct multiple hard-to-distinguish problem instances, rather than using only two instances as in \cite{tripuraneni2021provable}. 
We further derive a new second term by considering randomly generated $\alpha_i$'s instead of deterministic ones and carefully bounding the mutual information. This new term is critical for capturing  the statistical penalty that arises when the number of clients is large. 

\subsection{Notation and Organization} 
Let $\mathbb{S}^{d-1}$ be the unit sphere in $\R^d$. For matrix $M$, let $\lambda_r(M)$ be the $r$-th largest eigenvalue value of $M$ and $\|M\|$ be the spectral norm. 
We use notations $O$, $\Omega$, and $\Theta$ to hide absolute constants independent of model parameters, 
and use $\widetilde{O}$ to hide polylogarithmic factors.  %

The rest of the paper is organized as follows. Section \ref{sect:related-works} reviews further related works. Section \ref{sec:model} introduces the main model and assumptions. Section \ref{sec:up} presents our estimator and the error upper bound, and Section \ref{sec:low} establishes the minimax lower bound. Section \ref{sect:applications} discusses applications, while Section \ref{sect:non} extends the results to general non-linear models. Sections \ref{sec:proof-up} and \ref{sec:proof-low} contain the proofs of the upper and lower bounds, respectively. 
Section \ref{sect:exp} provides numerical studies. 
Section  \ref{sect:conclusion} concludes with future directions.

\section{Related Work}\label{sect:related-works}

\paragraph{Statistical Rates for Learning Shared Representations.}

Most relevant recent works \citep{du2021fewshot,collins2021exploiting,thekumparampil2021statistically,duchi2022subspace}, except \cite{tripuraneni2021provable}, focus on studying the statistical rate for learning $B^\star$ under model \eqref{eq:model-sup}  with equal data partition $n_i\equiv n$. The SOTA upper and lower bounds in these works are summarized in Table \ref{tab:bounds_comparison}. None of them identifies the optimal rate even in well-represented cases. In particular, \cite{tripuraneni2021provable} introduces a method-of-moments estimator (MoM) but only obtains the suboptimal error rate in Table \ref{tab:bounds_comparison} for standard Gaussian covariates. They also derive the suboptimal lower bound 
by applying Le Cam's two-point method to two hard-to-distinguish problem instances constructed given a deterministic ill-conditioned client diversity matrix. Concurrently, \cite{du2021fewshot} provides purely statistical guarantees by directly analyzing the optimal solution to the nonconvex empirical risk minimization, yet still obtains the same suboptimal upper bound, 
under the additional assumption $n \gg d.$
More recently, \cite{duchi2022subspace} presents a variant of MoM by excluding the diagonal terms. Their estimator achieves an optimal rate for simpler mean estimation problems, but its error upper bound for linear regression settings remains suboptimal, 
under the additional assumption that $n_i\equiv n.$
Furthermore, \cite{collins2021exploiting,thekumparampil2021statistically,zhang2024sample} study alternating minimization methods, which effectively reduce the error
rates when the noise level of $\xi_{ij}$, $\sigma^2$, is either zero or diminishes rapidly.
However, since the empirical
risk minimization for $B$ is non-convex, the algorithms require initialization sufficiently close to an optimal solution, typically starting with MoM. So their sample complexity requirements remain the same as that for MoM. 

In addition, several studies have explored variants of the model in \eqref{eq:model-sup}. For example, \cite{thaker2023leveraging} studies differentially private fine-tuning given the method-of-moments estimator. Moreover, \cite{duan2023adaptive} studies adaptive and robust multi-task learning, where, in a specific low-rank scenario, they assume $\theta_i^\star = B^\star \alpha_i^\star + v_i^\star$ with a bounded offset term $v_i^\star$. In addition, \cite{tian2023learning} considers the case when clients share similar representations such that $\theta_i^\star = B^\star_i \alpha_i^\star$, with the subspaces $B^\star_i$ constrained to certain angles, while also allowing outliers. However, \cite{thaker2023leveraging} and \cite{duan2023adaptive} achieve the suboptimal error rate in Table \ref{tab:bounds_comparison}. In addition, both \cite{duan2023adaptive} and \cite{tian2023learning} impose much stricter assumptions that $n_i\ge cd$, where clients can already independently learn $\theta_i^\star$ accurately. 

\paragraph{Multi-task and Transfer Learning.} As an important application for learning shared representations, the study of multi-task and transfer learning problems dates back to \cite{caruana1997multitask,thrun1998learning,baxter2000model,ando2005framework,maurer2016benefit}. Beyond the low-dimensional model introduced in \eqref{eq:model-sup}, other related models for theoretical studies include gradient-based meta-learning \citep{lee2018gradient,khodak2019adaptive} and non-parametric transfer learning \citep{hanneke2019value,cai2021transfer}.

\paragraph{Personalized Federated Learning.} Recent works have studied personalization in federated learning to handle statistical heterogeneity across clients. These approaches include multi-task learning \citep{smith2017federated}, gradient-based meta-learning \citep{fallah2020personalized}, and local fine-tuning \citep{singhal2021federated}. The shared low-dimensional representation model \eqref{eq:low-d-structure} provides an effective framework for personalization
\citep{collins2021exploiting}, and includes clustered federated learning~\citep{ghosh2020efficient,su2024global} as a special case.

\section{Model and Assumptions}\label{sec:model}
The shared linear subspace model in \eqref{eq:low-d-structure} generalizes those in \cite{tripuraneni2021provable,collins2021exploiting,thekumparampil2021statistically,duchi2022subspace} by allowing heterogeneous covariance matrices $\Gamma_i$ 
instead of assuming $\Gamma_i \equiv  I_d$ for all $i$. Our model setup also extends those in \cite{du2021fewshot,thekumparampil2021statistically,tian2023learning} by allowing heterogeneous data partitions $n_i$ rather than assuming $n_i \equiv n$ for all $i$.
In this section, we formally introduce our assumptions on the noise variables $
\xi_{ij} \triangleq y_{ij}-\E[y_{ij}\given x_{ij}]$, the covariate vectors $x_{ij}$, and the parameters $\theta_i^\star$. 
\begin{assumption}[Sub-gaussian noises]\label{ass:sub-gaussian-noise}
    The noise variables $\xi_{ij}$ are independent, 
    zero-mean, sub-gaussian with 
    variance proxy $\sigma^2=\Theta(1)$ and are independent of covariates 
    $x_{ij}$.\footnote{A random variable $\xi\in\R$ is sub-gaussian with variance proxy $\sigma^2$, denoted by $\xi\sim \text{subG}(\sigma^2)$, if 
    $\E[\exp(t(\xi-\E\xi))]\le \exp(\sigma^2t^2/2)$ for any $t\in\R$. A random vector $\xi\in\R^d$ is sub-gaussian with variance proxy $\sigma^2$, denoted by $\xi\sim \text{subG}_d(\sigma^2)$, if $u^\intercal \xi\sim \text{subG}(\sigma^2)$ for any $u\in \mathbb{S}^{d-1}$.} 
    \end{assumption}

\begin{assumption}[Sub-gaussian covariates]\label{ass:covariates}
    The covariate vectors $x_{ij}$ are independent, zero-mean, 
    sub-gaussian  with variance proxy $\gamma^2=\Theta(1)$. At each client $i$, let $\E[x_{ij}x_{ij}^\intercal]=\Gamma_i$
    be the same but unknown covariance matrix of $x_{ij}$ for all $j$, 
     satisfying $\lambda_1(\Gamma_i)/\lambda_d(\Gamma_i) =\Theta(1)$. 
    \end{assumption}
These sub-gaussian assumptions are standard in statistical learning for deriving tail bounds \citep{du2021fewshot,tian2023learning,duan2023adaptive,zhang2024sample}. In particular, Assumption \ref{ass:covariates} generalizes those in \cite{tripuraneni2021provable,collins2021exploiting,thekumparampil2021statistically,duchi2022subspace} by allowing non-identity covariance $\Gamma_i$. 

To capture the shared feature representations across clients, we assume that the weighted parameters $\Gamma_i\theta_i^\star$ have a common low-dimensional structure.
    \begin{assumption}[Low-dimensional structure]\label{ass:sup-low-rank} There exists a constant $k\le d$, an orthonormal matrix $B^\star \in \cO^{d\times k}$,
    and vectors $\alpha_i^\star \in \R^k$ such that $\Gamma_i\theta_i^\star = B^\star\alpha_i^\star$ for all $i$. 
    \end{assumption}
Here the columns of $B^\star \in \cO^{d\times k}$ form the shared low-dimensional subspace and $\alpha_i^\star \in \R^k$ is the client-specific parameter for client $i$. When $\Gamma_i \equiv I_d$ for all $i$, Assumption \ref{ass:sup-low-rank} reduces to the standard assumption $\theta_i^\star = B^\star\alpha_i^\star$, imposed by previous works such as \cite{tripuraneni2021provable,collins2021exploiting,thekumparampil2021statistically,duchi2022subspace}. We generalize this standard assumption to the case with non-identity covariance $\Gamma_i$, by requiring the cross-correlation vector $\E[y_{ij}x_{ij}]=\Gamma_i\theta_i^\star$ to share a common subspace.\footnote{Note that since the covariance matrices $\Gamma_i$ are unknown, one cannot apply the whitening procedure by writing $x_{ij}=\Gamma_i^{1/2} \widetilde{x}_{ij}$ so that $\widetilde{x}_{ij}$ have an identity covariance matrix. Moreover, it is difficult to accurately estimating $\Gamma_i$, as the size of the local dataset $n_i \ll d$.}



\begin{assumption}[Client 
normalization]\label{ass:client-diverse} 
Each $\alpha_i^\star$ satisfies $\|\alpha_i^\star\| = O(1)$ for all $i\in[M]$. 
\end{assumption}

The normalization is standard in the literature \citep{du2021fewshot,tripuraneni2021provable,duchi2022subspace,tian2023learning}. 
Recall that $\lambda_r$ is the $r$-th largest eigenvalue of 
$D$ in \eqref{eq:diversity-matrix} for $r\in[k]$. The normalization gives that $\sum_{r=1}^k \lambda_r= \tr(D) = \tr(\sum_{i=1}^M n_i \alpha_i^\star(\alpha_i^\star)^\intercal)/N = \sum_{i=1}^M n_i \|\alpha_i^\star\|^2/N = O(1)$, which further implies $k\lambda_k\le O(1)$ and $\lambda_1=O(1)$.

Given the model described in \eqref{eq:model-sup} and the assumptions introduced, the goal of the clients in these problems is to collectively estimate the shared representation $B^\star$. 
In particular, we define the following metric to measure the distance between 
two orthonormal matrices.
\begin{definition}[Principal angle distance] \label{def:principal_angle}
Let $B, B^\star\in\cO^{d\times k}$ be orthonormal matrices. 
Then the principal angle distance between $B$ and $B^\star$ is
    \begin{align}
    \|\sin\Theta(B, B^\star)\|  = \|BB^\intercal  - B^\star (B^\star)^\intercal\|. \label{eq:def_sin_theta_dis}
    \end{align} 
\end{definition}
Originating from \cite{jordan1875essai}, the principal angle distance measures the separation between the column spaces of $B$ and $B^\star$, invariant to their rotations. Using this metric, we aim to learn an estimator $\widehat B$ that minimizes $\|\sin\Theta(B, B^\star)\|$ over $B\in\cO^{d\times k}$, so that its column space closely aligns with that of $B^\star$. We will analyze the statistical rate of the error $\|\sin\Theta(\widehat B, B^\star)\|$ in learning $B^\star$ from data $\{(x_{ij}, y_{ij})\}$ generated by the model in \eqref{eq:model-sup}. 


\section{Estimator and Error Upper Bound}\label{sec:up}
We propose a new estimator of $B^\star$, which achieves an improved error upper bound. 

\subsection{Our Estimator}  
\label{subsect:estimator} 
We review the limitations of existing estimators and introduce our proposed estimator.

\subsubsection{Limitations of Existing Estimators} Many recent works have designed estimators with provable error bounds~\citep{du2021fewshot,tripuraneni2021provable,thekumparampil2021statistically,collins2021exploiting,duchi2022subspace}. 
We review existing estimators and discuss their limitations.
The method-of-moments estimator (MoM) in \cite{tripuraneni2021provable} is formed by the top-$k$ eigenvectors of matrix,
\#\label{eq:z1}
Z_T = \sum_{i=1}^M \sum_{j=1}^{n_i} y_{ij}^2 x_{ij}x_{ij}^\intercal \, .
\#
Their analysis is limited to scenarios where $x_{ij} \iiddistr N(0, I_d)$, with the corresponding error upper bound provided in Table \ref{tab:bounds_comparison}. However, this upper bound is suboptimal compared to the lower bound 
\cite[Theorem 5]{tripuraneni2021provable}. 
A subsequent work \citep{duchi2022subspace}
assumes $n_i\ge 2$ and introduces another spectral estimator using the matrix, 
\#\label{eq:z2}
Z_D = \sum_{i=1}^M \frac{w_i}{n_i(n_i-1)}\sum_{j_1\neq j_2} y_{ij_1}y_{ij_2}x_{ij_1}x_{ij_2}^\intercal  \, ,
\#
where $w_i>0$ are weight parameters satisfying $\sum_{i=1}^M w_i=1$, chosen as $w_i \equiv 1/M$ in the work.
By excluding the diagonal terms $j_1=j_2$ in the summation, their estimator 
is designed to handle scenarios where the noise 
$\xi_{ij}$ may depend on 
$x_{ij}$ and shown to achieve the same suboptimal error bound of \cite{tripuraneni2021provable}. However, whether this estimator improves the error rates remains unclear.
Several works~\citep{thekumparampil2021statistically,collins2021exploiting,zhang2024sample} study the alternating minimization methods. 
However, their results rely on initialization via MoM 
and thus still suffer from the suboptimality inherent in MoM. 
In fact, the suboptimality of
MoM has been acknowledged in many works~\citep{tripuraneni2021provable,thekumparampil2021statistically,thaker2023leveraging,tian2023learning} and closing this gap has remained a well-recognized open problem.

\subsubsection{A Warm-up Example: Mean Estimation Problems} We will introduce our estimator to address the limitations of these existing ones and improve the error rate. To illustrate our ideas, we begin with a simpler mean estimation problem and show that a local averaging estimator is an optimal solution to the least squares problem. Specifically, 
suppose that each client $i$ observes $n_i$ data vectors $u_{ij}\in \R^{d}$ for $j\in [n_i]$, where
\$
u_{ij}= \theta_i^\star + \xi_{ij} = B^\star\alpha_i^\star + \xi_{ij}.
\$
Here $\xi_{ij}\in \R^d$ is an additive noise for the $j$-th sample and $\theta_i^\star = B^\star \alpha_i^\star$ is the ground-truth, where $B^\star \in \cO^{d\times k}$ with $k\le d$ and $\alpha_i^\star \in \R^k$. 
We can directly solve this mean estimation problem, given the observed data, by minimizing the non-convex least squares loss,
\#\label{eq:me-least-sqr}
\min_{B\in\cO^{d\times k},\{\alpha_i\}} \sum_{i=1}^M\sum_{j=1}^{n_i}\|u_{ij} - B\alpha_i\|^2.
\#
Let $\overline u_i = (\sum_{j=1}^{n_i} u_{ij})/n_i$ be the local average at client $i$, and $\widetilde B$ be the  top-$k$ eigenvectors of 
$\sum_{i=1}^M n_i \overline u_i \overline u_i^\intercal$. The next proposition shows that $\widetilde B$ solves this least squares problem.\footnote{Proofs of Proposition \ref{thm:me}-\ref{prop:x-gen} and Example \ref{lem:well-cond} are straightforward, and we omit them here. Interested readers may refer to \cite{niu2024collaborative} for the formal proofs.}
\begin{proposition}\label{thm:me}
After first optimizing over $\alpha_i$, the problem in \eqref{eq:me-least-sqr} is equivalent to 
\#\label{eq:pca-me-1}
\max_{B\in\cO^{d\times k}} \sum_{i=1}^M n_i \overline u_i^\intercal B B^\intercal \overline u_i.
\#
In addition, the estimator $\widetilde B$ formed by the  top-$k$ eigenvectors of the matrix $\sum_{i=1}^M n_i \overline u_i \overline u_i^\intercal$ is an optimal solution to problems in \eqref{eq:me-least-sqr} and \eqref{eq:pca-me-1}.
\end{proposition}
Proposition \ref{thm:me} demonstrates that $\widetilde B$, utilizing local averaging, is an optimal solution to the least squares problem for mean estimation. 

\subsubsection{Introducing Our Estimator}

Now, we return to tackle the original problem by leveraging the idea of local averaging discussed above.  Similar to mean estimation, we consider the non-convex least squares problem for linear regression, with $\theta_i = \Gamma_i^{-1} B\alpha_i$,
\#\label{eq:least-squares}
\min_{\{\theta_i\}}\sum_{i=1}^M\sum_{j=1}^{n_i}\big(y_{ij} - x_{ij}^\intercal \theta_i\big)^2 = \min_{B, \{\alpha_i\}}\sum_{i=1}^M\sum_{j=1}^{n_i}\big(y_{ij} - x_{ij}^\intercal \Gamma_i^{-1} B\alpha_i\big)^2.
\#
Let $\widehat z_i = \sum_{j=1}^{n_i}y_{ij} x_{ij}/n_i$ be the $i$-th local average, and $A^\dag$ be the pseudoinverse of matrix $A$.
\begin{proposition}\label{claim:app}
After first optimizing over $\alpha_i$, the problem in \eqref{eq:least-squares} is equivalent to the following one, with $\Lambda_i = \Gamma_i^{-1} B(B^\intercal\Gamma_i^{-1}\widehat \Gamma_i \Gamma_i^{-1} B)^{\dag} B^\intercal\Gamma_i^{-1}$ and $\widehat \Gamma_i = (\sum_{j=1}^{n_i}x_{ij} x_{ij}^\intercal)/n_i$, 
\begin{align}
\max_{B\in\cO^{d\times k}} \sum_{i=1}^M n_i \widehat z_i^\intercal \Lambda_i \widehat z_i. \label{eq:regression_pca_true}
\end{align}
\end{proposition}
Unfortunately, unlike \eqref{eq:pca-me-1} in Proposition \ref{thm:me}, the problem in Proposition \ref{claim:app} lacks a closed-form solution since the matrix $\Lambda_i$  involves the unknown covariance $\Gamma_i$ and the decision variable $B$. 
If we pretend $\widehat{\Gamma}_i \approx \Gamma_i \approx I_d$,
then  $\Lambda_i \approx  B B^\intercal$. Thus, 
we approximate $\Lambda_i$ using $B B^\intercal$ in \eqref{eq:regression_pca_true} and instead solve the approximated problem: 
\#\label{eq:approx}
\max_{B\in\cO^{d\times k}} \sum_{i=1}^M n_i \widehat z_i^\intercal BB^\intercal \widehat z_i.
\#
The problem \eqref{eq:approx} has the same form as \eqref{eq:pca-me-1} 
and therefore the matrix formed by the top-$k$ eigenvectors of $\sum_{i=1}^M n_i \widehat z_i \widehat z_i^\intercal$
is an optimal solution to \eqref{eq:approx}.
As a result, it is tempting to estimate $B^\star$ based on the top-$k$ eigenvectors of $\sum_{i=1}^M n_i \widehat z_i \widehat z_i^\intercal$. 

However, since $\{x_{ij}\}$ follows general sub-gaussian distributions with non-identity covariance $\Gamma_i$, 
Proposition \ref{prop:x-gen} shows that, without additional assumptions on the covariance matrix $\Gamma_i$ and the fourth-order moments, $\E[x_{ij}^\intercal \theta_i^\star (\theta_i^\star)^\intercal x_{ij} x_{ij}x_{ij}^\intercal]$, it is impossible to construct the column space of $B^\star$ solely by using the eigenvectors of $\sum_{i=1}^M n_i \widehat z_i \widehat z_i^\intercal$. 
\begin{proposition}\label{prop:x-gen}
    Under Assumptions \ref{ass:sub-gaussian-noise}-\ref{ass:sup-low-rank}, the matrix $\sum_{i=1}^M n_i \widehat z_i \widehat z_i^\intercal$ satisfies
    \$
    \E \Big[\sum_{i=1}^M n_i \widehat z_i \widehat z_i^\intercal\Big] & = B^\star\Big(\sum_{i=1}^M (n_i-1) \alpha_i^\star (\alpha_i^\star )^\intercal\Big) (B^\star)^\intercal \\
    & \qquad + \sum_{i=1}^M\frac{1}{n_i} \sum_{j=1}^{n_i} \E[x_{ij}^\intercal \theta_i^\star (\theta_i^\star)^\intercal x_{ij} x_{ij}x_{ij}^\intercal] + \sum_{i=1}^M \frac{1}{n_i} \sum_{j=1}^{n_i} \E[\xi_{ij}^2]\Gamma_i.
    \$ 
\end{proposition}

To resolve this issue, we construct two independent replicas, $\overline z_i$ and $\widetilde z_i$, in replace of the local average $\widehat z_i$. For convenience, suppose that $n_i \ge 2$ is an even number. For $i\in[M]$, let $\overline z_i =(2/n_i) \cdot\sum_{j=1}^{n_i/2}y_{ij}x_{ij}$ and $\widetilde z_i = (2/n_i) \cdot\sum_{j=n_i/2 + 1}^{n_i}y_{ij}x_{ij}$ be two independent replicas of local averages at client $i$. We consider the following matrix,
\#\label{eq:estimator}
Z = \sum_{i=1}^M n_i \overline z_i \widetilde z_i^\intercal = \sum_{i=1}^M n_i\bigg(\frac{2}{n_i}\sum_{j=1}^{n_i/2}y_{ij}x_{ij}\bigg)\bigg(\frac{2}{n_i}\sum_{j=n_i/2 + 1}^{n_i}y_{ij}x_{ij}^\intercal\bigg).
\#
 We take $\widehat B$ as the matrix formed by the right (or left) top-$k$ singular vectors of $Z$. 






Now, with two independent replicas, it is easy to see that 
    \$
    \E [Z] = \E \Big[\sum_{i=1}^M n_i \overline z_i \widetilde z_i^\intercal\Big] = 
    \sum_{i=1}^M n_i \E [\overline z_i] \E[\widetilde z_i^\intercal] =B^\star\Big(\sum_{i=1}^M n_i \alpha_i^\star (\alpha_i^\star )^\intercal\Big) (B^\star)^\intercal, 
  \$
and the column space of $\E Z$ recovers that of $B^\star$; thus, $\widehat B$, formed by the singular vectors of $Z$, provides a good estimate for $B^\star$, ensured by the classic perturbation theory for singular vectors \citep{wedin1972perturbation}. This highlights the benefits of using two replicas. Similar replica ideas have appeared in other related problems of mixed linear regression~\citep{kong2020meta,su2024global}, but the motivations and results therein are different from ours.

In summary, compared to $Z_T$ \eqref{eq:z1}, our estimator applies independent replicas of local averaging of $x_{ij}y_{ij}$ at each client, which reduces noise and achieves a tighter bound. Independent replicas allow us to handle general
covariate distributions, unlike \cite{tripuraneni2021provable}, which is restricted to standard Gaussian covariates. This is because $\E Z$ avoids fourth-order moments or $\Gamma_i$. In addition, local averaging effectively reduces noise. For example, in Lemma \ref{lem:noise-sup}, the error involves $M$ outer products of sub-exponential vectors with constant variance, giving an overall error of order $\sqrt{Md} + d$. In contrast, the corresponding error for $Z_T$ is $\sum_{i=1}^M\sum_{j=1}^{n_i} \xi_{ij}^2x_{ij}x_{ij}^\intercal$, which consists of $N$  outer products of the sub-exponential vector $\xi_{ij}x_{ij}$ with itself, also with constant variance. This leads to a larger error of order $\sqrt{Nd} + d$. 

For $Z_D$ in \eqref{eq:z2}, while excluding diagonal terms can be viewed as an alternative to our use of independent replicas, our approach provides significant advantages in privacy-sensitive settings, such as federated learning. To compute our estimator, each client can send only vectors of their local averages $\overline z_i$ and $\widetilde z_i$, or their variants with added noise, to the server, rather than transmitting any raw data $y_{ij}x_{ij}$. Thus, our design effectively prevents the leakage of local data. In addition, since our estimator approximates the least squares solution in \eqref{eq:least-squares}, we will show that it achieves the optimal statistical rate without the need for further refinement via alternating minimization.
\footnote{When the noise variance $\sigma^2$ is vanishing, further refinement via alternating minimization may improve the dependence on $\sigma^2$ and thus achieve a smaller estimation error, as shown in~\cite{thekumparampil2021statistically} for sufficiently fast diminishing $\sigma$  and~\cite{collins2021exploiting} for $\sigma=0$.}

\subsection{Error Upper Bound}


The following theorem establishes an error upper bound of our estimator $\widehat{B}$. 
\begin{theorem}\label{thm:upper-bound-sup}
Suppose that Assumptions \ref{ass:sub-gaussian-noise}-\ref{ass:client-diverse} hold and $n_i\ge 2$. For the estimator $\widehat{B}$ obtained in \eqref{eq:estimator}, with probability at least $1-O((d+N)^{-10})$,
\#\label{eq:main-up-error} 
    \|\sin\Theta(\widehat B, B^\star)\| = O\bigg(\bigg(\sqrt{\frac{d\lambda_{1}}{N\lambda_k^2}}  + \sqrt{\frac{Md}{N^2\lambda_k^2}}\bigg)\cdot\log^3(d+N)\bigg).
\#
\end{theorem}
Here the condition number $\lambda_1/\lambda_k$ and the smallest eigenvalue $\lambda_k$ appear in the numerator and denominator, respectively. This results aligns with our intuition that a larger $\lambda_1/\lambda_k$ or a smaller $\lambda_k$ causes more difficulty in estimating $B^\star$, since the client diversity matrix $D = \sum_{i=1}^M n_i \alpha_i^\star(\alpha_i^\star)^\intercal/N$ lacks information in certain directions.

Our error bound significantly improves over the previously best-known rate in the literature. Specifically, under the additional assumption $x_{ij} \iiddistr N(0, I_d)$, \cite{tripuraneni2021provable} shows that the method-of-moments estimator given in~\eqref{eq:z1} achieves an estimation error rate 
$\widetilde O(\sqrt{d
/(N\lambda_k^2)})$. However, their analysis crucially relies on the isotropy property of standard Gaussian vectors (see, e.g., the proof of Lemma 4 in~\cite{tripuraneni2021provable}). If $x_{ij}$'s were instead sub-gaussian, their error bound 
would become $\widetilde O(\sqrt{d\max\{1, k\lambda_1\}/(N\lambda_k^2)})$.
Subsequent work by~\cite{duchi2022subspace} analyzes a different spectral estimator in~\eqref{eq:z2} and  shows an error rate 
$\widetilde O(\sqrt{d
/(M\lambda_k^2\min_{i\in[M]} n_i)})$ for sub-gaussian $x_{ij}$'s. However, their bound matches the rate in~\cite{tripuraneni2021provable}) only under equal data partitioning, where $n_i \equiv n$. In cases with unequal data partitions, where $\min_{i\in[M]} n_i$ can be very small, their bound becomes less effective. In contrast, our result improves over the rate in \cite{tripuraneni2021provable} by a factor of $\sqrt{\min\{\lambda_1, M/N\}}$ and holds for any unequal data partitions $n_i$. 

In addition, our analysis requires only $n_i\ge 2$, which significantly relaxes the strict assumption $n_i\gg d$ imposed by \cite{du2021fewshot,chua2021fine,duan2023adaptive,tian2023learning,zhang2024sample}, where clients can already independently learn $\theta_i^\star$ accurately. 
This highlights the benefits of learning shared representations, especially when individual clients have limited data.

\subsubsection{Optimal Rates in Well-Represented Cases}
Our improvement is particularly significant in well-represented cases, where $\lambda_1/\lambda_k=\Theta(1)$. 
Before presenting the corollary, we note that this requirement is not overly restrictive and holds for many cases with unequal data partitions. As shown in the next example, it holds with high probability, so long as the data partition is not excessively unbalanced. 
\begin{example}\label{lem:well-cond}
Suppose that $\alpha_i$'s are i.i.d.\ sub-gaussian with $\E [\alpha_i\alpha_i^\intercal] = I_k/k$ for all $i$. Let $\overline{n} = (\sum_{i=1}^M n_i)/M = N/M$ be the average number of data per client. If ${\max_{i\in[M]} n_i}/{\overline{n}} \le c\sqrt{{M}/{k}}$ for a constant $c>0$,
then with probability at least $1-2\exp(-k)$, the matrix $D= \sum_{i=1}^M n_i \alpha_i\alpha_i^\intercal/N$ satisfies $\lambda_1/\lambda_k = \Theta(1)$.
\end{example}
When $M=\Theta(k)$, the condition in Example \ref{lem:well-cond} requires $\max_{i\in[M]} n_i= \Theta(\overline{n})$, allowing heterogeneous data partitions with a constant relative ratio. When $M\gg k$, Example \ref{lem:well-cond} shows that well-represented cases allow significantly greater heterogeneity on data partitions. 

\begin{corollary}\label{cor:up-sup}
    Suppose Assumptions \ref{ass:sub-gaussian-noise}-\ref{ass:client-diverse} hold, $n_i\ge 2$, and  $\lambda_1=\Theta(\lambda_k) = \Theta(1/k)$. For the estimator $\widehat{B}$ obtained in \eqref{eq:estimator}, with probability at least $1-O((d+N)^{-10})$,
    \$ 
    \|\sin\Theta(\widehat B, B^\star)\| = O\bigg(\bigg(\sqrt{\frac{dk}{N}} + \sqrt{\frac{Mdk^2}{N^2}}\bigg)\cdot\log^3(d+N)\bigg). 
    \$
\end{corollary}

This rate applies to unequal data partitions and improves the rate $\widetilde O(\sqrt{dk^2/N})$ given by \cite{tripuraneni2021provable,duchi2022subspace}. More strikingly,  our rate is order-wise optimal, matching the minimax lower bound up to a polylogarithmic factor, as shown in the next section. This resolves the challenging open problem of characterizing the optimal estimation error rate, and our estimator is the first in the literature to achieve this optimal rate. Figure \ref{fig:diag} plots the phase diagram for Corollary \ref{cor:up-sup} under a specific scaling parameterization.

\section{Minimax Lower Bound}\label{sec:low}
This section establishes an information-theoretic lower bound for any estimator  $\widehat{B}.$
For fixed $M$ and $N$, we use the eigenvalues $\lambda_1$ and $\lambda_k$ of the client diversity matrix $\sum_{i=1}^M n_i \alpha_i^\star(\alpha_i^\star)^\intercal/N$ to capture the complexity of the estimation problem. In particular, we analyze the minimax estimation error against the worst possible choice of the model parameters $B$, $\{\alpha_i\}_{i=1}^M$, and $\{n_i\}_{i=1}^M$ from a parameter space.  The problem of estimating $B$ can then be represented as the following Markov chain, 
\$
\left( B, \{\alpha_i\}_{i=1}^M, \{n_i\}_{i=1}^M \right) \to \{\{(x_{ij}, y_{ij})\}_{j=1}^{n_i}\}_{i=1}^M \to \widehat{B}.
\$
Here the data volumes $\{n_i\}_{i=1}^M$, satisfying $\sum_{i=1}^M n_i = N$, can be observed from the data and hence are nuisance parameters. 

We now define the parameter space. We take $B\in \cO^{d\times k}$ to be any $d\times k$ orthogonal matrix. 
Let $\alpha= (\alpha_1,\cdots,\alpha_M)$ be the matrix whose columns are the client-specific parameters $\alpha_i$ and $\vec{n} = (n_1, \cdots, n_M)^\intercal$ be the vector with entries $n_i$.
For any $\lambda_1\ge \lambda_k>0$ and a fixed $\Vec{n}=(n_1, \ldots, n_M)$, we define the parameter space of $\alpha$ satisfying Assumption \ref{ass:client-diverse} as follows:
\$ 
\Psi_{\lambda_1, \lambda_k}^{n_1, \cdots, n_M} = \Big\{\alpha\in\R^{k\times M} \colon \|\alpha_i\|=O(1) \ \forall i \in [M], \Omega(\lambda_k) I_k \preceq \frac{1}{N} \sum_{i=1}^M n_i \alpha_i \alpha_i^\intercal \preceq O(\lambda_1) I_k \Big\}.
\$
We consider only $\lambda_k >0$, since otherwise $B^\star$ is not identifiable.\footnote{If $\lambda_k =0$, $\{\alpha_i^\star\}$ spans only an $r$-dimensional subspace of $\R^k$ with $r<k$, so the parameters $\{\theta_i^\star\}$ and data $\{(x_{ij}, y_{ij})\}$ contain information only about an $r$-dimensional subspace of $B^\star$'s columns. The remaining $k-r$ columns of $B^\star$ may be any vectors in the $(d-r)$-dimensional complement, making $B^\star$ unidentifiable.} This also implies $M\ge k$. Recall that Assumption \ref{ass:client-diverse} yields $k\lambda_k=O(1)$
and $\lambda_1=O(1) $; otherwise, the parameter space is empty. 
Henceforth, we assume $\lambda_k>0$, 
$M\ge k$,
$k\lambda_k=O(1)$, and 
$\lambda_1=O(1)$. 
The following theorem presents the minimax error lower bound. Here
$\wedge$ is a shorthand notation for the minimum operation. 
\begin{theorem}\label{thm:lb-sup}
Consider a system with $M$ clients and $N$ samples. 
Assume $x_{ij}\sim N(0, I_d)$ and $\xi_{ij}\sim N(0, 1)$ independently for all $i,j$, and Assumptions \ref{ass:sup-low-rank} and \ref{ass:client-diverse} hold. When 
$k=\Omega(\log M)$, $d\ge (1+\rho_1)k$, and $M\ge(1+\rho_2)k$ for constants $\rho_1,\rho_2>0$, we have
    \$
    \inf_{\widehat{B}\in\cO^{d\times k}}
    \sup_{B\in\cO^{d\times k}} 
    \sup_{\substack{n_1,\cdots, n_M\\ \sum_{i=1}^M n_i = N} } \sup_{\alpha\in\Psi_{\lambda_1, \lambda_k}^{n_1, \cdots, n_M}} \E\Big[\big\|\sin\Theta(\widehat B, B)\big\|\Big] = \Omega\bigg(\bigg(\sqrt{\frac{d}{N\lambda_k}} + \sqrt{\frac{Md}{N^2\lambda_k^2}} \bigg)\wedge1\bigg).
    \$   
\end{theorem}

Theorem \ref{thm:lb-sup} establishes an error lower bound, improving the state-of-the-art result from \cite{tripuraneni2021provable}, which is of order $\Omega(\sqrt{1/(N\lambda_k)} + \sqrt{dk/N})$. Specifically, we achieve a $\sqrt{d}$ improvement in the first term 
and derive a new second term. 
Our lower bound matches the upper bound presented in Theorem \ref{thm:upper-bound-sup}, differing only by a condition number $\sqrt{\lambda_1/\lambda_k}$ in the first term and a logarithmic factor. This gap in the condition number also arises in other related problems 
in statistical learning. 
For example, for the Lasso estimator in sparse linear regression, a gap between existing upper and lower bounds remains in the condition number of the data design matrix $X$ \citep{raskutti2011minimax,zhang2014lower,chen2016bayes}. Similarly, for the spectral estimator in the principal component analysis, a gap remains in the condition number of the low-dimensional factor matrix $B^\star$ under a well-conditioned client diversity matrix $D$ 
\citep[Theorem 3.6]{chen2021spectral}.

Notably, when $\lambda_1/\lambda_k =\Theta(1)$, our upper and lower bounds match up to a polylogarithmic factor, as shown in the following corollary. 


\begin{corollary}\label{cor:lb-sup}
Under the conditions in Theorem \ref{thm:lb-sup}, when $\lambda_1=\Theta(\lambda_k) = \Theta(1/k)$, 
we have 
    \$
    \inf_{\widehat{B}\in\cO^{d\times k}}\sup_{B\in\cO^{d\times k}}  
    \sup_{\substack{n_1,\cdots, n_M\\ \sum_{i=1}^M n_i = N} } \sup_{\alpha\in\Psi_{\lambda_1, \lambda_k}^{n_1, \cdots, n_M}} 
    \E\Big[\big\|\sin\Theta(\widehat B, B)\big\|\Big] = \Omega\bigg(\bigg(\sqrt{\frac{dk}{N}} + \sqrt{\frac{Mdk^2}{N^2}}\bigg)\wedge 1 \bigg).
        \$
    \end{corollary}
    Corollary \ref{cor:lb-sup} shows the error lower bound for well-represented cases and improves that from \cite{tripuraneni2021provable} of order $\Omega(\sqrt{dk/N})$. This result matches the upper bound in Corollary \ref{cor:up-sup} up to a polylogarithmic factor, thereby giving the optimal statistical rate.

\section{Applications}\label{sect:applications}
Having identified the statistical rate for estimating $B^\star$, we now apply this result to learn the model parameters for a newly joined client or an unseen private task and provide a more precise characterization of when collaboration benefits a new client. 

\subsection{Transferring Representations to New Clients}
Consider a new client $M+1$ with $n_{M+1}$ independent samples $\{(x_{M+1,j}, y_{M+1,j})\}_{j=1}^{n_{M+1}}$ generated from model \eqref{eq:model-sup} with 
$\Gamma_{M+1} = I_d$ and local parameter $\theta_{M+1}^\star$. There exists $\alpha_{M+1}^\star\in \R^k$ with $\|\alpha_{M+1}^\star\| = O(1)$, such that $\theta_{M+1}^\star = B^\star \alpha_{M+1}^\star$. The new client aims to learn 
$\theta_{M+1}^\star$. 

If we substitute an estimator $\widehat B$, learned from clients $1$ to $M$, in place of  the shared $B^\star$, the problem reduces to 
\#\label{eq:alpha-M1}
\widehat\alpha_{M+1} =\argmin_{\widehat \alpha} \sum_{j=1}^{n_{M+1}}\|x_{M+1, j}^\intercal \widehat{B} \widehat \alpha - y_{M+1, j}\|^2.
\#
Equivalently, we first project $x_{M+1, j}$ onto the $k$-dimensional subspace to obtain $\widehat{B}^\intercal x_{M+1, j}$, and then estimate $\alpha_{M+1}^\star$ by the standard least-squares on the projected data. 
\cite{tripuraneni2021provable} provides an error upper bound for $\widehat B\widehat\alpha_{M+1}$ when $\|\sin\Theta(\widehat{B}, B^\star)\| \le \delta^2$. Since Corollary \ref{cor:up-sup} establishes an error bound for our estimator $\widehat B$ obtained by \eqref{eq:estimator} 
in well-represented 
cases, 
directly applying their results yields the following result.\footnote{Proofs of Corollaries \ref{cor:transfer}-\ref{cor:private} follow directly by applying results from \cite{tripuraneni2020theory} and \cite{thaker2023leveraging}, and we omit them here.} 
\begin{corollary}[Transfer learning]\label{cor:transfer}
Suppose that Assumptions \ref{ass:sub-gaussian-noise}-\ref{ass:client-diverse} hold and $\lambda_1=\Theta(\lambda_k) =\Theta(1/k)$. For $\widehat{B}$ obtained by \eqref{eq:estimator} and $\widehat\alpha_{M+1}$ given by \eqref{eq:alpha-M1}, with probability at least $1-O((d+N)^{-10})-O(n_{M+1}^{-100})$, 
    \$    
    \big\|\widehat B\widehat\alpha_{M+1} - B^\star\alpha_{M+1}^\star\big\| = \widetilde O\bigg(\sqrt{\frac{dk}{N}} + \sqrt{\frac{Mdk^2}{N^2}} + \sqrt{\frac{k}{n_{M+1}}}\bigg). 
    \$
\end{corollary}
Corollary \ref{cor:transfer} decomposes the estimation error into two parts: $\widetilde{O}(\sqrt{{dk}/{N}} + \sqrt{{Mdk^2}/{N^2}})$ captures the error for estimating $B^\star$, and  $\widetilde{O}(\sqrt{k/n_{M+1}})$ is the error for estimating $\alpha_{M+1}^\star$ given $\widehat{B}$. 
Compared to the previous literature~\citep{tripuraneni2021provable}, our optimal rate provides a more precise characterization of when collaboration benefits the new client. If the new client estimates $\theta_{M+1}^\star\in\R^d$ from scratch based on the local data, the resulting error rate will be $\widetilde{O}(\sqrt{d/n_{M+1}})$. Therefore, it is advantageous to first learn the shared representation when $dk/N \ll d/n_{M+1}$ and $Mdk^2/N^2 \ll d/n_{M+1}$. Such conditions are satisfied when $n_{M+1} \ll \min\{N/k, N^2/(Mk^2)\}$. Conversely, if $n_{M+1} \gg \min\{N/k, N^2/(Mk^2)\}$, then the collaboration is unhelpful and the new client would be better off learning individually.

\subsection{Private Fine-tuning for New Clients}
\cite{thaker2023leveraging} studies a differentially private variant of learning $\alpha_{M+1}^\star$ in the same setting further with $x_{ij} \iiddistr N(0,I_d)$. 
We present an additional corollary under $(\varepsilon,\delta)$-differential privacy (see the formal definition in \citep{dwork2006our,dwork2006calibrating}), building upon \cite
{thaker2023leveraging} and derived using our estimator.
\begin{corollary}[Private transfer learning]\label{cor:private}
Under Assumptions \ref{ass:sub-gaussian-noise}-\ref{ass:client-diverse}, suppose $\lambda_1=\Theta(\lambda_k) =\Theta(1/k)$, and $x_{ij} \iiddistr N(0,I_d)$. Given $\widehat{B}$ from \eqref{eq:estimator}, there exists an $(\varepsilon,\delta)$-differentially private estimator $\widehat \alpha_{M+1}^{\varepsilon,\delta}$ such that, with probability at least $1-O((d+N)^{-10})-O(n_{M+1}^{-100})$,
\$
\big\| \widehat B\widehat\alpha_{M+1}^{\varepsilon,\delta}  - B^\star\alpha_{M+1}^\star\big\|= \widetilde{O}\bigg(\sqrt{\frac{dk}{N}} + \sqrt{\frac{Mdk^2}{N^2}} + \sqrt{\frac{k}{n_{M+1}}} + \frac{k\sqrt{\log(1/\delta)}}{n_{M+1}\varepsilon}\bigg).
\$
\end{corollary}
For comparison, if the client privately estimates its $d$-dimensional parameter $\theta_{M+1}^\star$ from scratch,  the resulting tight error rate is $\widetilde{O}(\sqrt{d/n_{M+1}} + {d\sqrt{\log(1/\delta)}}/(n_{M+1}\varepsilon))$ \citep{pmlr-v178-varshney22a,cai2021cost}. 
Thus, when the estimation error of $\widehat{B}$, $\widetilde{O}(\sqrt{{dk}/{N}} + \sqrt{{Mdk^2}/{N^2}})$, is smaller, learning $\widehat{B}$ first will reduce the error rates.

\section{Extensions to General Models}\label{sect:non}
Moving beyond linear regression, we extend our analysis to settings where the responses $y_{ij}$ are generated by general non-linear models. Let $B^\star \in \cO^{d\times k}$ be the shared low-dimension representation. For client $i$, let $h_i \colon \R^k \to \R$ be a local function. Given a covariate $x_{ij}\in\R^d$ for the $j$-th sample, the response $y_{ij}$ in conditional expectation satisfies
\#\label{model:glm}
\E\big[y_{ij}\given x_{ij}\big] = h_i\big((B^\star)^\intercal x_{ij}\big).
\#
This formulation captures many canonical models, as outlined below. 

\paragraph{Generalized linear models (GLMs).} Let $\theta_i^\star\in\R^d$ be the ground-truth parameter for client $i$. In GLMs \citep{nelder1972generalized}, given $x\in\R^d$, the response $y_i$ follows a conditional distribution from the exponential family, 
\#\label{model:exp}
p(y_{i} \given x) = g(y_{i})\exp\Big(\frac{y_{i}\cdot x^\intercal \theta_i^\star - \psi(x^\intercal \theta_i^\star)}{c(\sigma)}\Big),
\#
where $g\colon \R\to \R_+$ is a carrier function, $\sigma$ is a scale parameter, and $\psi$ is the cumulant generating function defined by $\psi(t) = \log \int_{y\in\R} \exp(ty) g(y) dy$.
Standard results for exponential families \citep{mccullagh1989generalized} yield
$
\E[y_{i}\given x] = \psi^\prime(x^\intercal \theta_i^\star)$,
where $\psi^\prime$ is also known as the inverse link function. 
When $\theta_i^\star = B^\star\alpha_i^\star$ share a low-dimensional representation, 
we have $\E[y_{i}\given x] = \psi^\prime(x^\intercal \theta_i^\star) = \psi^\prime(x^\intercal B^\star\alpha_i^\star)$, in the general form with $h_i(u) = \psi^\prime(u^\intercal \alpha_i^\star)$.

Notable special cases of GLMs include linear regression, where $y_{i} \given x\sim N(x^\intercal \theta_i^\star, \sigma^2)$, and logistic regression, where $
    p(y_{i} \given x) = \exp(y_{i}x^\intercal \theta_i^\star - \log(1 + e^{x^\intercal \theta_i^\star}))$.

\paragraph{One-Hidden-Layer Neural networks.} One-hidden-layer neural networks also fit the framework of \eqref{model:glm}. We consider two-layer models $f_i(x) = \phi(B^\intercal x)^\intercal \alpha_i$, where clients share the first hidden layer with weights $B$ 
and activation function $\phi$, 
while each client has a personalized head $\alpha_i$ on top of the shared output $\phi(B^\intercal x)$. Assume the data are generated by a teacher network with parameters $B^\star \in \cO^{d\times k}$, $\alpha_i^\star\in\R^k$, and noise $\xi_{ij}$,
\$
y_{ij} = \phi((B^\star)^\intercal x_{ij})^\intercal \alpha_i^\star + \xi_{ij}.
\$
This corresponds to \eqref{model:glm} with $h_i(u) = \phi(u)^\intercal \alpha_i^\star$. For example, we will show that our analysis applies to the ReLU activation $(z)_+ = \max \{0, z\}$ despite its nonsmoothness.



For simplicity, we assume in this section the covariates $x_{ij}$ are i.i.d. standard Gaussian.
\begin{assumption}\label{ass:iid-x}
The covariates $x_{ij}$ are sampled i.i.d. from $N(0, I_d)$ for all $i$ and $j$.   
\end{assumption}

We now show that the same spectral estimator $\widehat B$ formed by the right (or left) top-$k$ singular vectors of matrix $Z$ in \eqref{eq:estimator} also recovers the column space of $B^\star$ for general models. 
The following lemma explains the reason.\footnote{Lemma \ref{lem:e-zi} follows by decomposing $x_{ij}$ into two orthogonal components,
one in the subspace $B$ and the other in its orthogonal complement. We omit the formal proof here; interested readers may refer to \cite{niu2024collaborative} for details.}  
Note that $(B^\star)^\intercal x_{i1} \sim N(0, I_k)$ when $x_{i1}\sim N(0, I_d)$.
\begin{lemma}\label{lem:e-zi}
Under the model \eqref{model:glm} and Assumption \ref{ass:iid-x}, for all $i$, we have $\E[h_i((B^\star)^\intercal x_{i1})x_{i1}] = B^\star \E_{U\sim N(0, I_k)}[h_i(U)U]$ and
    \$
    \E[\overline z_i] = \E[\widetilde z_i] = B^\star \E_{U\sim N(0, I_k)}\big[h_i(U)U\big].
    \$   
\end{lemma}
The vector $\E_{U\sim N(0, I_k)}[h_i(U)U]\in\R^k$ acts as $\alpha_i^\star$ in the linear model. 
Let the client diversity matrix for general models be $D = \sum_{i=1}^Mn_i \E_{U\sim N(0, I_k)}[h_i(U)U]\E_{U\sim N(0, I_k)}[h_i(U)U^\intercal]/N$. Lemma \ref{lem:e-zi} yields that 
$\E[Z] =  B^\star D(B^\star)^\intercal$,
so our estimator $\widehat B$, formed by the leading singular vectors of $Z$ is a good estimate for $B^\star$, as in the linear case. 
To show the error upper bound, we impose the following assumptions on the model in \eqref{model:glm}.
    
\begin{assumption}\label{ass:non-normal}
In the model \eqref{model:glm}, the following conditions hold:
\begin{enumerate}
\item For all $i\in [M]$, the function $h_i$ satisfies $\|\E_{U\sim N(0, I_k)}[h_i(U)U]\| = O(1)$;
\item For all $u,v\in\R^k$ and $i\in [M]$, it holds that $\|h_i(u) - h_i(v)\| = O( \|u- v\|)$;
\item Errors $y_{ij}\given x_{ij} - \E[y_{ij}\given x_{ij}]$ are independent sub-gaussian with constant variance proxy.
\end{enumerate}
\end{assumption}
Assumption \ref{ass:non-normal} serves as the counterpart of Assumptions \ref{ass:sub-gaussian-noise}-\ref{ass:client-diverse} for the general model. 
Specifically, the Lipschitz continuity of $h_i$ implies that $h_i(U)$ is sub-gaussian for $U\sim N(0, I_k)$ \cite[Theorem 5.2.2]{vershynin2018high}, which is crucial for deriving the tail bounds. 
Several previous works on GLMs \citep{loh2013regularized} require similar assumptions like $\psi''(\cdot) \le c$ for a constant $c>0$.\footnote{For differentiable $h_i(u) = \psi^\prime(u^\intercal \alpha_i^\star)$, we have $\nabla h_i(u) = \psi''(u^\intercal \alpha_i^\star)\alpha_i^\star$; thus $\|\nabla h_i(u)\| = \psi''(u^\intercal \alpha_i^\star) \|\alpha_i^\star\|$ due to the convexity of the cumulant generating function $\psi$ such that $\psi''(\cdot)\ge 0$. Since $\|\alpha_i^\star\| = O(1)$, their condition $\psi''(\cdot) \le c$ implies the boundedness of $\|\nabla h_i(u)\|$ and thus the Lipschitz continuity of $h_i$. 
} This Lipschitz condition holds in several important cases, including linear regression,
logistic regression, and one-hidden-layer ReLU networks. 

The next theorem shows that the estimator achieves the same error rate as in the linear regression case (Theorem \ref{thm:upper-bound-sup}) for general models. It extends the upper bound to a broader class of models, including 
logistic regression and one-hidden-layer ReLU networks.

\begin{theorem}\label{thm:non-up}
Given the model \eqref{model:glm} and Assumptions \ref{ass:iid-x} and \ref{ass:non-normal}, the estimator $\widehat{B}$ from \eqref{eq:estimator} achieves the error bound in \eqref{eq:main-up-error} with probability at least $1-O((d+N)^{-10})$, where $\lambda_1$ and $\lambda_k$ are the largest and smallest eigenvalues of $D$ for model \eqref{model:glm}, respectively.
\end{theorem}

\section{Proof of the Upper Bounds in Theorems \ref{thm:upper-bound-sup} and \ref{thm:non-up}}\label{sec:proof-up}
To unify the analysis of both linear and non-linear models discussed in Sections \ref{sec:up} and \ref{sect:non}, we introduce a general model and derive the error rate. 
Theorems \ref{thm:upper-bound-sup} and \ref{thm:non-up} follow as special cases. 
Suppose we observe data $\{(x_{ij}, y_{ij})_{j=1}^{n_i}\}_{i=1}^M$ to estimate $B^\star$, where $x_{ij}\in\R^d$, $y_{ij} = \zeta_{ij} + \xi_{ij}\in\R$, and $\E[\zeta_{ij} x_{ij}] = B^\star\alpha_i^\star$. Moreover, we impose the following assumptions.
\begin{assumption}\label{ass:general}
The tuples $(x_{ij}, \zeta_{ij}, \xi_{ij})$ are mutually independent across different $i,j$.
For all $i,j$, $x_{ij}\in\R^d$ is sub-gaussian with a constant variance proxy; $\zeta_{ij}$ is sub-gaussian, where $\zeta_{ij}/\|\alpha_i^\star\|$ has a constant variance proxy; and $\xi_{ij}\given x_{ij}$ is zero-mean sub-gaussian with a constant variance proxy.
\end{assumption}
We will show that, given the general model above and Assumption \ref{ass:general}, 
the estimator $\widehat{B}$ obtained by \eqref{eq:estimator} achieves the error bound in \eqref{eq:main-up-error} with probability at least $1-O((d+N)^{-10})$. This result directly
implies Theorems \ref{thm:upper-bound-sup} and \ref{thm:non-up} as special cases. 

\medskip
\begin{proof}[Proof of Theorem \ref{thm:upper-bound-sup}]
Let $\zeta_{ij} = x_{ij}^\intercal \theta_i^\star$. Then $\E[\zeta_{ij} x_{ij}] = \E[x_{ij}x_{ij}^\intercal]\theta_i^\star = B^\star\alpha_i^\star$. Moreover, $\xi_{ij}$ is sub-gaussian by Assumption \ref{ass:sub-gaussian-noise}, and  $\zeta_{ij} = (x_{ij}^\intercal \theta_i^\star/\|\theta_i^\star\|)\cdot \|\theta_i^\star\|$ is sub-gaussian with variance proxy $O(\|\theta_i^\star\|^2) = O(\|\Gamma_i^{-1} B^\star\alpha_i^\star\|^2) = O(\|\alpha_i^\star\|^2) =O(1)$ by Assumptions \ref{ass:covariates}-\ref{ass:client-diverse}. 
\end{proof}

\begin{proof}[Proof of Theorem \ref{thm:non-up}]
Given the model in \eqref{model:glm}, let $\alpha_i^\star = \E_{U\sim N(0,I_k)}[h_i(U)U]$, $\zeta_{ij} = h_i((B^\star)^\intercal x_{ij})$, and $\xi_{ij} = y_{ij} - h_i((B^\star)^\intercal x_{ij})$. Note that $\E[\zeta_{ij} x_{ij}] =\E[h_i((B^\star)^\intercal x_{i1})x_{i1}] = B^\star \E_{U\sim N(0, I_k)}[h_i(U)U] = B^\star\alpha_i^\star$ by Lemma \ref{lem:e-zi}. We now verify that $\zeta_{ij}$ and $\xi_{ij}$ satisfy Assumption \ref{ass:general}. First, $\zeta_{ij} =h_i((B^\star)^\intercal x_{ij})$ is sub-gaussian due to the Lipschitz continuity of $h_i$ in Assumption \ref{ass:non-normal} and $(B^\star)^\intercal x_{ij}\sim N(0, I_k)$ \citep
{vershynin2018high}. Next, the conditional variable $\xi_{ij}\given x_{ij}$ is zero-mean since $\E[\xi_{ij}\given x_{ij}] = \E[y_{ij}- \E[y_{ij}\given x_{ij}] \given x_{ij}] =0$. Finally, $\xi_{ij}\given x_{ij}$ is sub-gaussian since both $y_{ij}\given x_{ij}$ (Assumption \ref{ass:non-normal}) and $\zeta_{ij} =h_i((B^\star)^\intercal x_{ij})$ are sub-gaussian. 
\end{proof}   

We now prove the error bound in \eqref{eq:main-up-error} under the general model. 
For all $i$, let 
\begin{align*}
a_i = \frac{2}{\sqrt{n_i}} \sum_{j=1}^{n_i/2} \zeta_{ij} x_{ij}, \
b_i = \frac{2}{\sqrt{n_i}} \sum_{j=1}^{n_i/2} \xi_{ij} x_{ij}, \ 
u_i = \frac{2}{\sqrt{n_i}} \sum_{j=n_i/2+1}^{n_i} \zeta_{ij} x_{ij}, \
v_i = \frac{2}{\sqrt{n_i}} \sum_{j=n_i/2+1}^{n_i} \xi_{ij} x_{ij}.
\end{align*}
Then $(a_i, b_i)$ is independent of $(u_i, v_i)$ since they are from independent data subsets. Moreover, $\sqrt{n_i}\overline z_i = 2/\sqrt{n_i} \cdot \sum_{j=1}^{n_i/2} y_{ij} x_{ij} = a_i + b_i$ and $\sqrt{n_i}\widetilde z_i = u_i + v_i$.

\medskip
\begin{proof}[Proof of the error bound in \eqref{eq:main-up-error}]
Recall that $\E[\zeta_{ij}x_{ij}] = B^\star\alpha_i^\star$ and $\E[\xi_{ij} x_{ij}] = \E[\E[\xi_{ij}\given x_{ij}] x_{ij}] = 0$ for all $i,j$. Thus, by $\sqrt{n_i}\overline z_i = a_i + b_i$ and $\sqrt{n_i}\widetilde z_i = u_i + v_i$, we have
\$
\E \overline z_i = \frac{1}{\sqrt{n_i}}\big(\E a_i + \E b_i\big) = \frac{2}{n_i}\sum_{i=1}^{n_i/2}\E[\zeta_{ij}x_{ij}] + \frac{2}{n_i}\sum_{i=1}^{n_i/2}\E[\xi_{ij}x_{ij}]=  B^\star\alpha_i^\star.
\$ 
Similarly, we have $\E \widetilde z_i = B^\star\alpha_i^\star$ and thus
    \$
    \E Z = \sum_{i=1}^M n_i (\E \overline z_i) (\E \widetilde z_i^\intercal)  = B^\star\Big(\sum_{i=1}^M n_i\alpha_i^\star(\alpha_i^\star)^\intercal \Big) (B^\star)^\intercal.
    \$
By substituting $\sqrt{n_i}\overline z_i = a_i + b_i$ and $\sqrt{n_i}\widetilde z_i = u_i + v_i$ into the estimator $Z$, we have
\$
& Z - \E Z = \sum_{i=1}^M n_i \overline z_i \widetilde z_i^\intercal - B^\star\Big(\sum_{i=1}^M n_i\alpha_i^\star(\alpha_i^\star)^\intercal \Big) (B^\star)^\intercal \\
&\qquad = \sum_{i=1}^M (a_i + b_i)(u_i + v_i)^\intercal - B^\star\Big(\sum_{i=1}^M n_i\alpha_i^\star(\alpha_i^\star)^\intercal \Big) (B^\star)^\intercal \\
&\qquad = \sum_{i=1}^M a_iu_i^\intercal -  B^\star\Big(\sum_{i=1}^M n_i\alpha_i^\star(\alpha_i^\star)^\intercal \Big) (B^\star)^\intercal
+ \sum_{i=1}^M a_i v_i^\intercal 
+ \sum_{i=1}^M b_i u_i^\intercal 
+ \sum_{i=1}^M b_i v_i^\intercal.
\$
Lemmas \ref{lem:main-sup}-\ref{lem:noise-sup} will bound the fluctuations of each term in spectral norm using random matrix tools. Since $b_iu_i^\intercal$ and $a_iv_i^\intercal$ are identically distributed, Lemma \ref{lem:inner-prod-sup} also applies to $\|\sum_{i=1}^M b_iu_i^\intercal \|$. Substituting these results, we have, with probability at least $1-O((d+N)^{-10})$,
\begin{align}
    \|Z - \E Z\| & \le \Big\|\sum_{i=1}^M a_iu_i^\intercal -  B^\star\Big(\sum_{i=1}^M n_i\alpha_i^\star(\alpha_i^\star)^\intercal \Big) (B^\star)^\intercal \Big\| 
    +  \Big\| \sum_{i=1}^M a_iv_i^\intercal\Big\| 
    + \Big\|\sum_{i=1}^M b_iu_i^\intercal \Big\| + \Big\|\sum_{i=1}^M b_iv_i^\intercal\Big\| \label{eq:error_decomp} \notag \\
& = O\big(\big(\sqrt{Md} + \sqrt{Nd\lambda_1} + d \big)\cdot\log^3(d+N) \big). 
\end{align}
Finally, we apply Wedin's $\sin\Theta$ theorem \citep{wedin1972perturbation} to bound the subspace error. Note that $\E Z= B^\star(\sum_{i=1}^M n_i\alpha_i^\star(\alpha_i^\star)^\intercal) (B^\star)^\intercal$ is rank-$k$ with $\lambda_{k}(\E Z) =N\lambda_{k}$ and $\lambda_{k+1}(\E Z)=0$. Observe that when $N \lambda_k =O(d)$, our upper bound trivializes. Thus, we assume $N\lambda_k =\Omega(d)$ henceforth and obtain 
\$
\|\sin\Theta(\widehat{B}, B^\star)\| & \le \frac{2\|Z - \E Z\|}{\lambda_{k}(\E Z)} = \frac{O((\sqrt{Md} + \sqrt{Nd\lambda_1} + d )\cdot\log^3(d+N))}{N\lambda_{k}} \\
& = O\bigg(\bigg(\sqrt{\frac{d\lambda_{1}}{N\lambda_k^2}}  + \sqrt{\frac{Md}{N^2\lambda_k^2}}\bigg)\cdot\log^3(d+N)\bigg),
\$
where the last line holds from $N\lambda_k=\Omega(d)$ and $d/(N\lambda_k) \le \sqrt{d/(N\lambda_k)} \le \sqrt{d\lambda_{1}/(N\lambda_k^2)}$. 
\end{proof}

It remains to bound each error term in \eqref{eq:error_decomp}, 
summarized below, which completes the proof of the error bound.
\begin{lemma}\label{lem:main-sup}
Under Assumption \ref{ass:general}, with probability at least $1-O((d+N)^{-10})$, we have
\$
\Big\|\sum_{i=1}^M a_iu_i^\intercal -  B^\star\Big(\sum_{i=1}^M n_i\alpha_i^\star(\alpha_i^\star)^\intercal \Big) (B^\star)^\intercal \Big\| = O\big(\big(\sqrt{Md} + \sqrt{Nd\lambda_1} + d \big)\cdot\log^3(d+N) \big).
\$
\end{lemma}
\begin{lemma}\label{lem:inner-prod-sup}
Under Assumption \ref{ass:general}, with probability at least $1-O((d+N)^{-10})$, we have
\$
\Big\| \sum_{i=1}^M a_i v_i^\intercal\Big\| = O\big(\big(\sqrt{Md} + \sqrt{Nd\lambda_1} + d \big)\cdot\log^3(d+N) \big).
\$
\end{lemma}
\begin{lemma}\label{lem:noise-sup}
Under Assumption \ref{ass:general}, with probability at least $1-O((d+N)^{-10})$, we have
\$
\Big\|\sum_{i=1}^M b_iv_i^\intercal\Big\| = O\big(\big(\sqrt{Md} + d \big)\cdot\log^3(d+N) \big).
\$
\end{lemma}


We prove these lemmas using the truncated matrix Bernstein inequality, and we defer their proofs to the appendix.

\section{Proof of the Lower Bound in Theorem \ref{thm:lb-sup}}\label{sec:proof-low}
This section proves the lower bound by an information-theoretic argument via Fano's method. We will establish the two terms, $\sqrt{d/(N\lambda_k)}$ and $\sqrt{Md}/(N\lambda_k)$, in Theorem \ref{thm:lb-sup} separately. While the proofs of these two
terms share a similar structure, they differ in the construction of $\alpha_i$ and the way to bound the mutual
information in Steps 2 and 4. 

We begin with a preliminary result that constructs a $(c\sqrt{\varepsilon})$-separated packing set $\cB\subset\cO^{d\times k}$ for a constant $c>0$ and any $0<\varepsilon<1/2$. Let $\{e_r\}_{r=1}^k$ be the standard basis in $\R^k$.
\begin{lemma}[Packing set]\label{lem:packing}
    There exists a constant $c>0$ such that one can construct a packing set $\widetilde\cB = \{b_1, \cdots, b_K\} \subset \cO^{(d-k)\times 1}$ with $K \ge 10^{(d-k)}$ that is $(\sqrt{2}c)$-separated, \ie, $\|\sin\Theta(b_r, b_s)\| \ge \sqrt{2}c$ for any $r\neq s$.
Moreover, given $0<\varepsilon<1/2$, we define $B_r\in\R^{d\times k}$ as
    \$
     B_r = \begin{pmatrix}
         & e_1, & \cdots, &e_{k-1}, & \sqrt{1-\varepsilon}e_k \\
         & 0, & \cdots, & 0, &\sqrt{\varepsilon} b_r
     \end{pmatrix},
     \$
     for all $r\in[K]$. 
     Then the set $\cB= \{B_1, \cdots, B_K\}$ forms a packing set of $\cO^{d\times k}$, with $\|\sin\Theta(B_r, B_s)\| \ge c\sqrt{\varepsilon}$ for any $r \neq s$.
\end{lemma}
\begin{proof}
    First, \cite{liu2022optimal} proves the first statement that there exists a $(\sqrt{2}c)$-separated packing set $\widetilde\cB\subset\cO^{(d-k)\times 1}$ with cardinality  $K=|\widetilde\cB|\ge 10^{(d-k)}$. 
    We now study the properties of $\cB$. For any $r\in[K]$, we have $B_r^\intercal B_r=I_k$; thus $\cB\subset\cO^{d\times k}$. In addition, we compute that
\$
B_rB_r^\intercal = \begin{pmatrix}
    I_k-\varepsilon e_ke_k^\intercal & \sqrt{\varepsilon(1-\varepsilon)} e_k b_r^\intercal \\
    \sqrt{\varepsilon(1-\varepsilon)} b_re_k^\intercal & \varepsilon b_rb_r^\intercal \end{pmatrix}.
\$
Fix any $r\neq s$ and recall $\|\sin\Theta(B_r, B_s)\| = \|B_rB_r^\intercal - B_sB_s^\intercal\|$. 
Then we have 
\$
\|\sin\Theta(B_r, B_s)\| &= \|B_rB_r^\intercal - B_sB_s^\intercal\|
        \ge  \sqrt{\varepsilon(1-\varepsilon)}\|e_k (b_r - b_s)^\intercal\| \\
        &= \sqrt{\varepsilon(1-\varepsilon)} \|b_r -b_s\| \ge \sqrt{\varepsilon(1-\varepsilon)} \|\sin\Theta(b_r, b_s)\| \\
& \ge \sqrt{\varepsilon(1-\varepsilon)} \times \sqrt{2}c \ge c\sqrt{\varepsilon},
\$
where the first inequality holds because the spectral norm of a matrix is no smaller than that of its submatrix, the second equality holds because $e_k(b_r-b_s)^\intercal$ is rank-one, and the second inequality follows from $\|b_r-b_s\|=2 \|\sin(\Theta(b_r, b_s)/2) \|\ge\|\sin\Theta(b_r, b_s)\|$. 
\end{proof}

We obtain the first term, $\sqrt{{d}/({N\lambda_k})}$, by considering deterministic $\alpha_i$'s. The key lies in Step 2, which we construct
hard-to-distinguish problem instances. 

\medskip
\begin{proof}[Proof of the first term]
\paragraph{Step 1: Reduction to a probability bound.}
    For any $\widehat{B}$, $B$, $\{n_i\}$, and $\alpha$, Markov's inequality gives that
\$
\E\big[\big\|\sin\Theta(\widehat B, B)\big\|\big] \ge \frac{c\sqrt{\varepsilon}}{2} \pr\Big(\big\|\sin\Theta(\widehat B, B)\big\| \ge \frac{c\sqrt{\varepsilon}}{2}\Big).
\$
Thus, to show the first term, it suffices to show the following for $\sqrt{\varepsilon} = \Theta(\sqrt{d/(N\lambda_k)})$,
\#\label{eq:to-prove}
\inf_{\widehat{B}\in\cO^{d\times k}}\sup_{B\in\cO^{d\times k}} \sup_{\substack{n_1,\cdots, n_M\\ \sum_{i=1}^M n_i = N} } \sup_{\alpha\in\Psi_{\lambda_1, \lambda_k}^{n_1, \cdots, n_M}} \pr\big(\big\|\sin\Theta(\widehat B, B)\big\| \ge c\sqrt{\varepsilon}/2 \big) \ge 1/2.
\#
\paragraph{Step 2: Constructing hard-to-distinguish problem instances.}
The learnability of $B$ is determined by the smallest eigenvalue $\lambda_k$ of the client-diversity matrix $D$, as a smaller $\lambda_k$ indicates less information in the data along its associated eigenvector direction. Accordingly, we choose $B$ and $\alpha_i$ to capture $\lambda_k$ in the lower bound. 

We fix $\alpha_i$ for $i\in[M]$ such that $\alpha_i=O(1)$ and $\sum_{i=1}^M n\alpha_i \alpha_i^\intercal = N \lambda_k I_k$. As an example, we write each $i$ as $i=rk+s$ with $r,s\in\Z^+$, $r \le \floor{M/k}$ and $s<k$. Let $\alpha_i =\sqrt{M\lambda_k/\floor{M/k}} e_{s+1}$ if $r \le \floor{M/k}$, and $\alpha_i=0$ otherwise. Given $0<k\lambda_k=O(1)$ and $M\ge k$, such $\alpha_i$'s satisfy the required conditions. 

We then construct $\cB= \{B_1, \cdots, B_K\}\subset\cO^{d\times k}$ as the $(c\sqrt{\varepsilon})$-separated packing set from Lemma \ref{lem:packing}, where we keep the first $k-1$ columns of $B_r$ fixed and vary only the last column corresponding to $\lambda_k$. Here $K\ge 10^{(d-k)}$, and $\log K \ge c_1(d-k) \ge c_2d$ when $d\ge (1+\rho_1)k$ for constants $c_1, c_2,\rho_1>0$. 
We sample $B \sim \text{Unif}(\cB)$ and set $n_i \equiv n = {N}/{M}$. 

Let $x_i = (x_{ij})_{j\in[n]}$ and $y_i = (y_{ij})_{j\in[n]}$ be the local data for client $i$, where $x_{ij}\sim N(0,I_d)$ and $\xi_{ij}\sim N(0,1)$ independently, and $y_{ij} = x_{ij}^\intercal B\alpha_i + \xi_{ij}$. Since $\alpha_i$ is deterministic, we have \$P_{y_i\given x_i, B} = N(x_i^\intercal B\alpha_i, I_{n}).\$ Let $X = (x_i)_{i\in[M]}$, $Y = (y_i)_{i\in[M]}$ be the entire dataset, and $\pr_{B, (X, Y)}(\cdot)$ denote the joint distribution of $(B, (X, Y))$. 
For any $\widehat{B}$, we lower-bound the supremum by an average:
\#\label{eq:fano-step1-2}
&\sup_{B\in\cO^{d\times k}} \sup_{\substack{n_1,\cdots, n_M\\ \sum_{i=1}^M n_i = N} } \sup_{\alpha\in\Psi_{\lambda_1, \lambda_k}^{n_1, \cdots, n_M}} \pr\big(\big\|\sin\Theta(\widehat B, B)\big\| \ge c\sqrt{\varepsilon}/2 \big) \notag\\
& \qquad \qquad \ge \pr_{B, (X, Y)}\big(\big\|\sin\Theta(\widehat B, B)\big\| \ge c\sqrt{\varepsilon}/2 \big).
\#
\paragraph{Step 3: Applying Fano’s inequality.}
Let $\phi\colon \cO^{d\times k} \to \cB$ be a quantizer that maps any $B\in \cO^{d\times k}$ to the closet point in $\cB$. Recall that $\|\sin\Theta(B_r, B_s)\| \ge c\sqrt{\varepsilon}$ for any $r \neq s$. For any $B_i \in\cB$, if $\phi(\widehat B) \neq B_i$, then we have $\|\sin\Theta(\widehat B, B_i)\| \ge c\sqrt{\varepsilon}/2$. Thus, we obtain
\#\label{eq:fano-2}
    \pr_{B, (X, Y)}\big(\big\|\sin\Theta(\widehat B, B)\big\| \ge c\sqrt{\varepsilon}/2\big) & \ge \pr_{B, (X, Y)}\big(\phi(\widehat B) \neq B \big) \notag\\
    & \ge 1 - \frac{I(B; (X, Y)) + \log 2}{\log K},
\#
where the last inequality follows from Fano's inequality. 

\paragraph{Step 4: Bounding the mutual information.}
It remains to establish an upper bound for $I(B; (X, Y))$. Since $(x_1, y_1), \cdots, (x_M, y_M)$ are independent conditioned on $B$, we have
$
I(B; (X, Y)) 
\le \sum_{i=1}^M I(B; (x_i, y_i)).$
Since $B$ and $x_i$ are independent, we have
\$
I\big(B; (x_i, y_i)\big) & = \E_B \E_{x_i}\big[\kl{P_{y_i\given x_i, B}}{P_{y_i\given x_i}}\big] \notag\\
& = \E_B \E_{x_i}\big[\kl{P_{y_i\given x_i, B}}{\E_{B^\prime}[P_{y_i\given x_i, B^\prime}]}\big] \notag\\
&\le \E_{B^\prime}\E_B \E_{x_i}\big[\kl{P_{y_i\given x_i, B}}{P_{y_i\given x_i, B^\prime}}\big],
\$
where the inequality follows from the convexity of KL-divergence.
Combining the above two inequalities, we obtain
\#\label{eq:I-B-xy}
I\big(B; (X,Y)\big) & \le \E_{B^\prime}\E_B \Big(\sum_{i=1}^M\E_{x_i}\big[\kl{P_{y_i\given x_i, B}}{P_{y_i\given x_i, B^\prime}}\big]\Big) \notag\\
& \le \max_{B_r, B_s\in \cB} \sum_{i=1}^M\E_{x_i}\big[\kl{P_{y_i\given x_i, B_r}}{P_{y_i\given x_i, B_s}}\big].
\#
We now compute the divergence for fixed $B_r \neq B_s$. Recall that $P_{y_i\given x_i, B_r} = N(x_i^\intercal B_r\alpha_i, I_{n})$. The KL-divergence between multivariate Gaussian distributions, along with $\E[ x_ix_i^\intercal] = \sum_{j=1}^{n} \E[ x_{ij}x_{ij}^\intercal] = nI_d$, yields that
\$
&\E_{x_i}\big[\kl{P_{y_i\given x_i, B_r}}{P_{y_i\given x_i, B_s}}\big] = \frac{1}{2} \alpha_i^\intercal(B_r-B_s)^\intercal\E[ x_ix_i^\intercal](B_r-B_s)\alpha_i  \\
& \qquad \qquad\qquad = \frac{1}{2} n\alpha_i^\intercal(B_r-B_s)^\intercal (B_r-B_s)\alpha_i  = \frac{1}{2}\tr\big((B_r-B_s)^\intercal (B_r-B_s)\cdot n\alpha_i\alpha_i^\intercal\big).
\$ 
Let $\Delta_{r,s} = (B_r-B_s)^\intercal (B_r-B_s)/2 = I_k - B_r^\intercal B_s/2 - B_s^\intercal B_r/2$. For $B_r,B_s$ defined in Lemma \ref{lem:packing}, we compute that
$
B_r^\intercal B_s = B_s^\intercal B_r = \diag{1,\cdots,1,1-\varepsilon+\varepsilon b_r^\intercal b_s}$. Thus, $\Delta_{r,s} = \diag{0,\cdots,0,\varepsilon(1 - b_r^\intercal b_s)}$. Substituting this into the above equation and recalling that $\sum_{i=1}^{M} n\alpha_i\alpha_i^\intercal = N\lambda_k I_k$, we have
\$
& \sum_{i=1}^M\E_{x_i}\big[\kl{P_{y_i\given x_i, B_r}}{P_{y_i\given x_i, B_s}}\big] = \tr\Big(\Delta_{r,s}\cdot\Big(\sum_{i=1}^{M} n\alpha_i\alpha_i^\intercal\Big)\Big) \\
& \qquad \qquad \qquad = N\lambda_k \cdot\tr\big(\diag{0,\cdots,0,\varepsilon(1 - b_r^\intercal b_s)}\big)  = \varepsilon(1 - b_r^\intercal b_s)N\lambda_k \le 2\varepsilon N\lambda_k,
\$
where the last inequality holds since $- b_r^\intercal b_s \le \|b_r\|\|b_s\|=1$.
Substituting the above into \eqref{eq:I-B-xy}, we obtain $I\big(B; (X,Y)\big) \le 2\varepsilon N\lambda_k$. Thus, when $\varepsilon >0$ satisfies
\$
\varepsilon =  \frac{c_1d}{6N\lambda_k} \wedge 1 = \Theta\Big(\frac{d}{N\lambda_k} \wedge 1 \Big),
\$
we have $I(B; (X, Y))\le 2\varepsilon N\lambda_k = c_1d/3$.
Thus, \eqref{eq:fano-2} implies that $\pr_{B,\alpha, X}(\|\sin\Theta(\widehat B, B)\| \ge c\sqrt{\varepsilon}/2)\ge 1/2$ since $\log K\ge c_1d$. Finally, substituting \eqref{eq:fano-2} into \eqref{eq:fano-step1-2} completes the proof. 
\end{proof} 

The proof of the second term, $\sqrt{Md}/(N\lambda_k)$, differs from that of the first term in the choice of $\alpha_i$ and the way to bound $I(B; (x_i,y_i))$. Instead of deterministic $\alpha_i$, we will consider random $\alpha_i$. The next lemma shows the concentration of $\alpha$ with Gaussian-generated columns.
\begin{lemma}\label{lem:bound-a}
Assume  $k=\Omega(\log M)$ and $M\ge(1+\rho_2)k$ for a constant $\rho_2>0$, and fix $n_i =n =N/M$ for $i\in[M]$. 
When generating $\alpha_i\sim N(0, \lambda_k I_k)$ independently, we have
\$
\pr\big( \alpha\in\Psi_{\lambda_1, \lambda_k}^{n,\cdots, n} \big) \ge 3/4.
\$
\end{lemma}
\begin{proof}
By the definition of the parameter space $\Psi_{\lambda_1, \lambda_k}^{n,\cdots, n}$, 
we have
\$
\pr\big( \alpha \in \Psi_{\lambda_1, \lambda_k}^{n,\cdots, n} \big) \ge 
\pr\big(\|\alpha_i\|= O(1), \forall i \big) + \pr\Big(\Omega(\lambda_k) I_k \preceq \frac{1}{N}\sum_{i=1}^M n \alpha_i\alpha_i^\intercal  \preceq O(\lambda_1) I_k\Big) -1.
\$
We first show that $\pr(\|\alpha_i\|= O(1), \forall i ) \ge 7/8$.
By the union bound, 
$
\pr(\exists i,\, \|\alpha_i\|\neq O(1) ) 
\le \sum_{i=1}^M 
\pr(\|\alpha_i\| \neq  O(1))
= M \pr(\|\alpha_1\| \neq  O(1))$.
Thus, it suffices to show $\pr(\|\alpha_1\| \neq O(1) )\le 1/(8M)$. Since $\alpha_1/\sqrt{\lambda_k}$ is a $k\times 1$ standard Gaussian matrix, the concentration inequality \citep
{vershynin2018high} gives, for any $t 
\ge 0$, $
\pr(\|\alpha_1\| \le \sqrt{\lambda_k}(\sqrt{k} + 1 +t)) \ge 1- 2\exp(-t^2/2) 
$.
Taking $t = \sqrt{2\log(16M)}$ yields, with probability at least $1-1/(8M)$,
$
\|\alpha_1\| \le  \sqrt{ \lambda_k}(\sqrt{k}+1+\sqrt{2\log(16M)})
\le O(1)$,
since $k\lambda_k=O(1)$  and $k=\Omega(\log M)$ together imply $\lambda_k \log M = O(\log(M)/k)=O(1)$. 
Thus, $
\pr(\|\alpha_1\| \neq O(1) )
\le 1/(8M)$. It remains to show that
$$
\pr \Big( \Omega(\lambda_k) I_k \preceq \frac{1}{N}\sum_{i=1}^M n \alpha_i\alpha_i^\intercal  \preceq O(\lambda_1) I_k \Big) \ge 7/8.
$$
Let $\sigma_r(\cdot)$ be the $r$-th largest singular value of a matrix. Since $\lambda_r(\sum_{i=1}^M n \alpha_i\alpha_i^\intercal) = n\sigma_r^2(\alpha)$, it reduces to proving that 
$
\pr ( \Omega(\sqrt{M\lambda_k}) \le
\sigma_k(\alpha) \le \sigma_1(\alpha) \le O(\sqrt{M\lambda_1}) ) \ge 7/8$.
We now bound $\sigma_1(\alpha)$ and $\sigma_k(\alpha)$. 
The concentration properties of the standard Gaussian matrix $\alpha/
\sqrt{\lambda_k}$ \cite
{vershynin2018high} yields that, for any $t 
\ge 0$, with probability at least $1-2\exp(-t^2/2)$,
$
\sqrt{\lambda_k}(\sqrt{M} - \sqrt{k} - t)\le \sigma_k(\alpha) \le \sigma_1(\alpha) \le \sqrt{\lambda_k}(\sqrt{M} + \sqrt{k} +t)$.
Thus, by picking $t=\sqrt{2\log(16)}$, with probability at least $7/8$,
\$
\sqrt{\lambda_k}(\sqrt{M} - \sqrt{k} - \sqrt{2\log(16)})\le \sigma_k(\alpha) \le \sigma_1(\alpha) \le \sqrt{ \lambda_k}(\sqrt{M} + \sqrt{k} + \sqrt{2\log(16)}).  
\$
Finally, since $M\ge(1+\rho_2)k$ for a constant $\rho_2>0$, we have
$$
\sqrt{M}-\sqrt{k}-\sqrt{2\log(16)} \ge \Omega(\sqrt{M}), \quad
\sqrt{M}+\sqrt{k}+\sqrt{2\log(16)} \le O(\sqrt{M}). \qedhere
$$
\end{proof}

We now prove the second term $\sqrt{Md}/(N\lambda_k)$ similarly, 
but with Gaussian-generated $\alpha$.

\medskip
\begin{proof}[Proof of the second term]

\paragraph{Step 1: Reduction to a probability bound.}
Similar to the proof of the first term, we only need to show that the following holds for $\varepsilon = \Theta({Md}/(N^2\lambda_k^2))$,
\#\label{eq:to-prove-term2}
\inf_{\widehat{B}\in\cO^{d\times k}}\sup_{B\in\cO^{d\times k}} \sup_{\substack{n_1,\cdots, n_M\\ \sum_{i=1}^M n_i = N} } \sup_{\alpha\in\Psi_{\lambda_1, \lambda_k}^{n_1, \cdots, n_M}} \pr\big(\big\|\sin\Theta(\widehat B, B)\big\| \ge c\sqrt{\varepsilon}/2 \big) \ge 1/2.
\#

\paragraph{Step 2: Constructing hard-to-distinguish problem instances.} Similarly, we choose $B$ and $\alpha_i$ to capture $\lambda_k$ in the lower bound. We use the same packing set $\cB$ from Lemma \ref{lem:packing} with $\log K\ge c_2d$. We sample $B \sim \text{Unif}(\cB)$.
But here we generate $\alpha_i\sim N(0, \lambda_k I_k)$ independently for $i\in[M]$. Given $B_r\in\cB$ and $\alpha_i\sim N(0, \lambda_k I_k)$, the model $y_i = x_i^\intercal B_r\alpha_i + \xi_i$ implies that $P_{y_i\given x_i, B_r} = N(0, \Sigma_{ir})$ with $$
\Sigma_{ir} = \lambda_k x_i^\intercal B_r B_r^\intercal x_i + I_n, 
\quad \text{ where }
B_r B_r^\intercal = \begin{pmatrix}
    I_k-\varepsilon e_ke_k^\intercal & \sqrt{\varepsilon(1-\varepsilon)} e_k b_r^\intercal \\
    \sqrt{\varepsilon(1-\varepsilon)} b_re_k^\intercal & \varepsilon b_rb_r^\intercal \end{pmatrix}.
$$
Let $\pr_{B, \alpha, (X, Y)}(\cdot)$ denote the joint distribution of $(B,\alpha, (X, Y))$. 
For any $\widehat{B}$ and $\varepsilon>0$, we lower-bound the supremum and obtain
\#\label{eq:fano-step1}
    & \sup_{B\in\cO^{d\times k}} \sup_{\substack{n_1,\cdots, n_M\\ \sum_{i=1}^M n_i = N} } \sup_{\alpha\in\Psi_{\lambda_1, \lambda_k}^{n_1, \cdots, n_M}} \pr\big(\big\|\sin\Theta(\widehat B, B)\big\| \ge c\sqrt{\varepsilon}/2 \big) \notag\\
    &\qquad \qquad\ge \pr_{B,\alpha, (X, Y)}\big(\big\|\sin\Theta(\widehat B, B)\big\| \ge c\sqrt{\varepsilon}/2 \biggiven \alpha\in\Omega_{\lambda_1, \lambda_k}^{n, \cdots, n} \big) \notag\\
    &\qquad \qquad\ge \pr_{B,\alpha, (X, Y)}\big(\big\|\sin\Theta(\widehat B, B)\big\| \ge c\sqrt{\varepsilon}/2, \alpha\in\Omega_{\lambda_1, \lambda_k}^{n, \cdots, n} \big) \notag\\
    &\qquad \qquad \ge \pr_{B,\alpha, (X, Y)}\big(\big\|\sin\Theta(\widehat B, B)\big\| \ge c\sqrt{\varepsilon}/2\big) + \pr\big( \alpha\in\Omega_{\lambda_1, \lambda_k}^{n, \cdots, n} \big) - 1,
\#
where the inequalities hold since for any events $\cE$ and $\cA$, we have $\pr(\cE\given\cA)=\pr(\cE\cap\cA)/\pr(\cA) \ge \pr(\cE\cap\cA)$ and $\pr(\cE\cap\cA) \ge \pr(\cE) + \pr(\cA) - 1$. Lemma \ref{lem:bound-a} gives $ \pr\big( \alpha\in\Omega_{\lambda_1, \lambda_k}^{n, \cdots, n} \big) - 1 \ge -1/4$. Thus, it remains to lower bound the first term of~\eqref{eq:fano-step1}. 

\paragraph{Step 3: Applying Fano’s inequality.}
Fano's inequality yields that
\#\label{eq:fano}
    \pr_{B,\alpha, (X, Y)}\big(\big\|\sin\Theta(\widehat B, B)\big\| \ge c\sqrt{\varepsilon}/2\big) \ge 1 - \frac{I(B; (X, Y)) + \log 2}{\log K}.
\#
\paragraph{Step 4: Bounding the mutual information.}
Since $(x_1, y_1), \cdots, (x_M, y_M)$ are independent, we have $
I(B; (X, Y)) \le \sum_{i=1}^M I(B; (x_i, y_i))$.
It remains to upper bound $I\big(B; (x_i, y_i)\big)$ for a fixed $i$. 
 We define $Q_{\cdot\given x_i} = N(0, \Sigma_{Q_i})$ with $\Sigma_{Q_i} = \lambda_k  x_i^\intercal \begin{pmatrix}
    I_k & 0\\
    0 & 0
\end{pmatrix}_{d\times d} x_i + I_n$.
Since $B$ and $x_i$ are independent, we can bound $I(B; (x_i, y_i))$ as follows,
\#\label{eq:mI-B-xi}
I\big(B; (x_i, y_i)\big) & = \E_B \E_{x_i}\big[\kl{P_{y_i\given x_i, B}}{P_{y_i\given x_i}}\big] \notag\\
& = \E_B \E_{x_i}\big[\kl{P_{y_i\given x_i, B}}{Q_{\cdot\given x_i}}\big] - \kl{P_{y_i\given x_i}}{Q_{\cdot\given x_i}} \notag\\
&\le \E_B \E_{x_i}\big[\kl{P_{y_i\given x_i, B}}{Q_{\cdot\given x_i}}\big] \notag\\
&\le \max_{B_r\in \cB}\E_{x_i}\big[\kl{P_{y_i\given x_i, B_r}}{Q_{\cdot\given x_i}}\big].
\#
We now compute $\E_{x_i}[\kl{P_{y_i\given x_i, B_r}}{Q_{\cdot\given x_i}}]$ for a fixed $B_r$. The KL-divergence between multivariate Gaussian distributions yields that
\$
\E_{x_i}\big[\kl{P_{y_i\given x_i, B_r}}{Q_{\cdot\given x_i}}\big] = \frac{1}{2}\E_{x_i}\Big[\log\frac{|\Sigma_{Q_i}|}{|\Sigma_{ir}|} + \tr\big(\Sigma_{Q_i}^{-1}\Sigma_{ir} - I_n \big) \Big].
\$
The non-negativity of the KL divergence: 
$
2\kl{Q_{\cdot\given x_i}}{P_{y_i\given x_i, B_r}} = \log\frac{|\Sigma_{ir}|}{|\Sigma_{Q_i}|} + \tr\big(\Sigma_{ir}^{-1}\Sigma_{Q_i} - I_n \big)\ge 0,
$
gives 
$\log\frac{|\Sigma_{Q_i}|}{|\Sigma_{ir}|} \le \tr\big(\Sigma_{ir}^{-1}\Sigma_{Q_i} - I_n \big).
$ Thus, we can bound the above equation as
\$
\E_{x_i}\big[\kl{P_{y_i\given x_i, B_r}}{Q_{\cdot\given x_i}}\big] \le \frac{1}{2}\E_{x_i}\Big[\tr\big(\Sigma_{ir}^{-1}\Sigma_{Q_i} + \Sigma_{Q_i}^{-1}\Sigma_{ir} - 2I_n \big) \Big].
\$
Let $\Delta = \Sigma_{ir} - \Sigma_{Q_i}$. Since $\Sigma_{ir} \succeq I_n$ and $\Sigma_{Q_i} \succeq I_n$, we have
\$
& \tr\big(\Sigma_{ir}^{-1}\Sigma_{Q_i} + \Sigma_{Q_i}^{-1}\Sigma_{ir} - 2I_n \big) = \tr\Big(\Sigma_{ir}^{-1}\big(\Sigma_{Q_i} + \Sigma_{ir}\Sigma_{Q_i}^{-1}\Sigma_{ir} - 2\Sigma_{ir} \big) \Big) \\
&\qquad \le \tr\big(\Sigma_{Q_i} + \Sigma_{ir}\Sigma_{Q_i}^{-1}\Sigma_{ir} - 2\Sigma_{ir} \big) = \tr\big(\Delta\Sigma_{Q_i}^{-1}\Delta \big) = \tr\big(\Sigma_{Q_i}^{-1}\Delta^2 \big) \le \tr(\Delta^2).
\$
Combining the above two equations, we have
\#\label{eq:bound-KL-sup}
\E_{x_i}\big[\kl{P_{y_i\given x_i, B_r}}{Q_{\cdot\given x_i}}\big] \le \frac{1}{2}\E_{x_i}\big[\tr(\Delta^2 )\big] = \frac{1}{2}\tr\big(\E_{x_i}(\Delta^2 )\big).
\#
Split $x_i = (a_i; u_i)$ with $a_i = (a_{i1}, \cdots, a_{in})\in\R^{k\times n}$ and $u_i = (u_{i1}, \cdots, u_{in})\in\R^{(d-k)\times n}$. Then $a_{ij}\sim N(0,I_k)$, $u_{ij}\sim N(0,I_{d-k})$, and all $a_{ij}$ and $u_{ij}$ are mutually independent. Note that
\$
\Delta & = \lambda_k \begin{pmatrix}
    a_i^\intercal & u_i^\intercal
\end{pmatrix}
\begin{pmatrix}
    -\varepsilon e_ke_k^\intercal & \sqrt{\varepsilon(1-\varepsilon)} e_k b_r^\intercal \\
    \sqrt{\varepsilon(1-\varepsilon)} b_re_k^\intercal & \varepsilon b_rb_r^\intercal \end{pmatrix}
\begin{pmatrix}
    a_i^\intercal \\
     u_i^\intercal
\end{pmatrix} \\
& = \lambda_k (-\varepsilon a_i^\intercal e_ke_k^\intercal a_i + \sqrt{\varepsilon(1-\varepsilon)}u_i^\intercal b_re_k^\intercal a_i + \sqrt{\varepsilon(1-\varepsilon)}a_i^\intercal e_kb_r^\intercal u_i + \varepsilon u_i^\intercal b_rb_r^\intercal u_i).
\$ 
Let $\widetilde{a}_i = a_i^\intercal e_k$ and $\widetilde{u}_i = u_i^\intercal b_r$; thus, $\widetilde{a}_i, \widetilde{u}_i\sim N(0, I_n)$. By the symmetry and independence, 
\#\label{eq:lb-sup-exp}
\E_{x_i}(\Delta^2) & = 2 \lambda_k^2 \E_{x_i}\big[\varepsilon^2\widetilde{a}_i\widetilde{a}_i^\intercal\widetilde{a}_i\widetilde{a}_i^\intercal - \varepsilon^2 \widetilde{a}_i\widetilde{a}_i^\intercal\widetilde{u}_i\widetilde{u}_i^\intercal + (1-\varepsilon)\varepsilon \widetilde{u}_i\widetilde{a}_i^\intercal\widetilde{u}_i\widetilde{a}_i^\intercal + (1-\varepsilon)\varepsilon \widetilde{u}_i\widetilde{a}_i^\intercal\widetilde{a}_i\widetilde{u}_i^\intercal \big].
\#
By the linearity, we compute the trace of each term above. Since $\widetilde{a}_i^\intercal\widetilde{a}_i\sim \chi^2(n)$, we have
\$
& \tr\big(\E_{x_i}\big(\widetilde{a}_i\widetilde{a}_i^\intercal\widetilde{a}_i\widetilde{a}_i^\intercal\big)\big) = \E_{a_i}\big[(\widetilde{a}_i^\intercal\widetilde{a}_i)^2\big] = \var\big(\widetilde{a}_i^\intercal\widetilde{a}_i\big) + \big(\E_{a_i}\big[\widetilde{a}_i^\intercal\widetilde{a}_i\big]\big)^2 = n^2 + 2n; \\
&\tr\big(\E_{x_i}\big(\widetilde{a}_i\widetilde{a}_i^\intercal\widetilde{u}_i\widetilde{u}_i^\intercal\big)\big) = \tr\big([\E_{a_i}(\widetilde{a}_i\widetilde{a}_i^\intercal)]\cdot[\E_{u_i}(\widetilde{u}_i\widetilde{u}_i^\intercal)]\big) = \tr(I_n) = n; \\
& \tr\big(\E_{x_i}\big(\widetilde{u}_i\widetilde{a}_i^\intercal\widetilde{u}_i\widetilde{a}_i^\intercal\big)\big) = \tr\big(\E_{a_i}\big(\E_{u_i}\big(\widetilde{a}_i^\intercal\widetilde{u}_i\widetilde{a}_i^\intercal\widetilde{u}_i\biggiven \widetilde a_i\big)\big)\big) \\
& \qquad \qquad \qquad \qquad  = \tr\big(\E_{a_i}\big(\var_{u_i}\big(\widetilde{a}_i^\intercal\widetilde{u}_i\biggiven \widetilde a_i\big)\big)\big) = \tr\big(\E_{\widetilde{a}_i}\big(\widetilde{a}_i^\intercal\widetilde{a}_i\big)\big) = n; \\
& \tr\big(\E_{x_i}\big(\widetilde{u}_i\widetilde{a}_i^\intercal\widetilde{a}_i\widetilde{u}_i^\intercal\big)\big) = \tr\big(\E_{u_i}\big(\E_{a_i}\big(\widetilde{u}_i\widetilde{a}_i^\intercal\widetilde{a}_i\widetilde{u}_i^\intercal\biggiven \widetilde u_i\big)\big)\big) = n \tr\big(\E_{u_i}\big(\widetilde{u}_i\widetilde{u}_i^\intercal\big)\big) =  n \tr(I_n) = n^2.
\$
Substituting the above four terms into \eqref{eq:lb-sup-exp}, we have
\$
\tr\big(\E_{x_i}(\Delta^2 )\big) = 2 \lambda_k^2\big[\varepsilon^2 (n^2 + 2n) - \varepsilon^2n + (1-\varepsilon)\varepsilon n + (1-\varepsilon)\varepsilon n^2\big] = 2  \varepsilon \lambda_k^2 (n^2 + n).
\$
Combining the above with \eqref{eq:mI-B-xi} and \eqref{eq:bound-KL-sup}, we have
\$
I\big(B; (x_i, y_i)\big) \le \varepsilon \lambda_k^2 (n^2 + n).
\$
Recall that $N = Mn$. When $\varepsilon>0$ satisfies
\$
    \varepsilon = \frac{c_1d}{5M\lambda_k^2(n^2 + n)}\wedge 1  = \Theta\Big(\frac{Md}{N^2\lambda_k^2}\wedge 1 \Big),
\$
we bound the mutual information by 
$ I(B; (X, Y)) \le \sum_{i=1}^M I(B; (x_i, y_i)) \le \varepsilon M \lambda_k^2 (n^2 + n) = c_1d/5$.
Recall that $\log K\ge c_1d$. Thus, \eqref{eq:fano} yields that 
\#\label{eq:xxx}
\pr_{B,\alpha, X}(\|\sin\Theta(\widehat B, B)\| \ge c\sqrt{\varepsilon}/2)\ge 3/4.
\# Finally, substituting Lemma \ref{lem:bound-a} and \eqref{eq:xxx} into \eqref{eq:fano-step1}, we conclude the proof. 
\end{proof}

\section{Numerical Experiments}\label{sect:exp}

We complement the theory with numerical studies to highlight the strength of our estimator and the benefit of learning good representations. Beyond the estimator based on $Z$ in \eqref{eq:estimator} with two independent local averages, we consider extensions using multiple averages. Given a splitting strategy $\vec{g} = (g_1, \cdots, g_M)^\intercal$, client $i$ with $n_i$ samples divides its local data into $g_i \le n_i$ groups $G_{i1}, \cdots G_{ig_i}$, each of size $\floor{n_i/g_i}$. For the $r$-th group, we define the local average $\overline z_{ir} = \sum_{j\in G_{ir}} x_{ij}y_{ij} / \sqrt{\floor{n_i/g_i}}$. Using these local averages, we construct the matrix
\#\label{eq:Z-g}
Z_{g} = \sum_{i=1}^M \frac{1}{\sqrt{g_i(g_i-1)}} \sum_{r\neq s\in [g_i]} \overline z_{ir} \overline z_{is}^\intercal, 
\#
and define the estimator $\widehat B_{g}$ as 
the top-$k$ eigenvectors of $Z_{g}$. The scaling factor 
$1/\sqrt{g_i(g_i-1)}$ normalizes the sum over $g_i(g_i-1)$ cross-group terms $\overline z_{ir} \overline z_{is}^\intercal$. When $g_i \equiv 2$, $Z_{g}$ reduces to the symmetric version of \eqref{eq:estimator}, $(Z+Z^\intercal)/2\sqrt{2}$, and when  $g_i \equiv n_i$, it yields new weights $w_i = \sqrt{n_i(n_i-1)}$ for \cite{duchi2022subspace} in \eqref{eq:z2}. As our experiments show, these new weights help handle unequal data partitions. The $Z_g$ framework also captures the privacy level of each client: smaller $g_i$ implies less data leakage and hence stronger privacy protection.

\subsection{Synthetic Data}
We consider an experimental setup in which $x_{ij}\sim N(0,\Gamma_i)$ has non-standard Gaussian covariates $\Gamma_i$,\footnote{We generate each covariance matrix $\Gamma_i$ by first sampling a random matrix $A_i$ with i.i.d. entries uniformly on $[0,1)$, forming $\Gamma_i = (A_i + A_i^\intercal)/2 + 6I_d$, and normalizing its trace so the average eigenvalue is one.} and the client sample sizes $n_i$ are unequal, sampled uniformly from $[2, 118]$. We set $d=120$ and $k=10$, and present results averaged over $10$ repetitions. Fixing parameters $B^\star$, $\alpha_i^\star\sim N(0, I_k/k)$, $\Gamma_i$, and $n_i$, each repetition generates new data $\{x_{ij}, y_{ij}\}$ with $y_{ij} = x_{ij}^\intercal  \Gamma_i^{-1} B^\star\alpha_i^\star + \epsilon_{ij}$ and $\epsilon_{ij}\sim N(0,1)$. We compare the performance of estimators $Z_g$ in \eqref{eq:Z-g} (with varying $g$), MoM \citep{tripuraneni2021provable}, and the estimator (DFHT) from \cite{duchi2022subspace}.



\begin{figure}[ht]
        \centering
        \begin{minipage}[b]{0.45\textwidth}
\includegraphics[width=\textwidth]{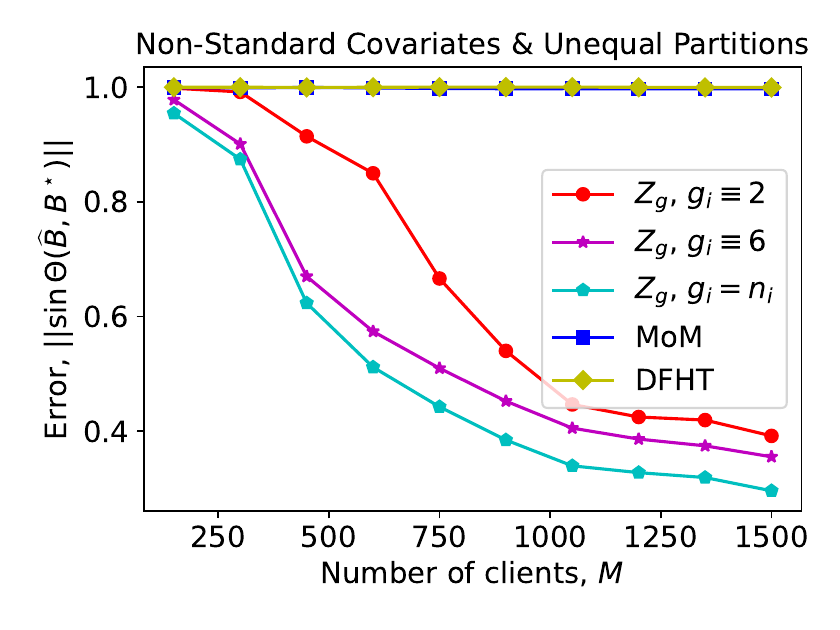}
        \caption{\small Subspace estimation error.}
        \label{fig:fig4}
    \end{minipage}
 \hfill
    \begin{minipage}[b]{0.45\textwidth}
        \centering
\includegraphics[width=\textwidth]{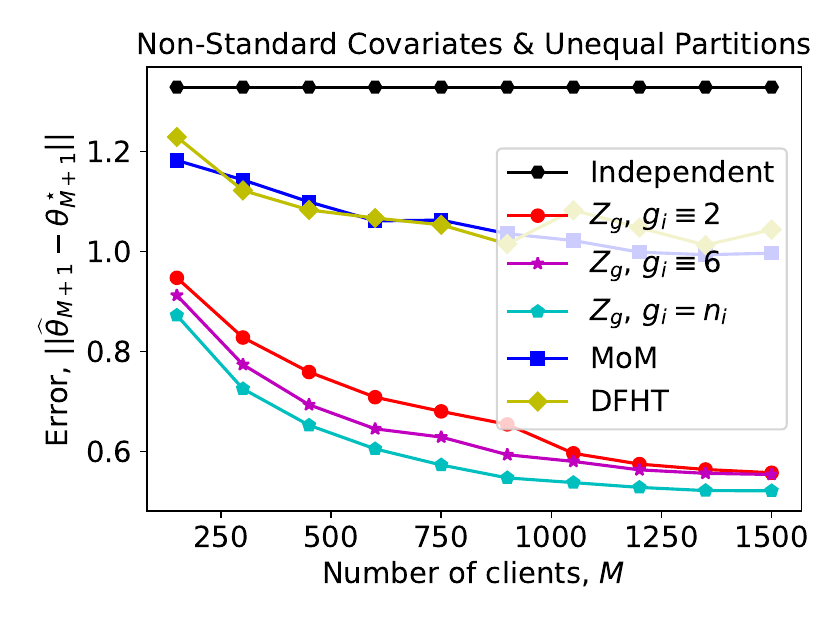}
        \caption{\small Estimation error at a new client.}
        \label{fig:fig6}
    \end{minipage}
\end{figure}


In Figure \ref{fig:fig4}, the y-axis shows the estimation error of $\widehat B$, and the x-axis is the number of clients $M$. 
As 
$M$ increases, the total number of samples grows, and consequently the subspace estimation error decreases for efficient estimators such as $Z_g$.
 Although both MoM and DFHT perform poorly, our estimator $Z_g$ remains robust and effective even in this heterogeneous setting with non-standard covariates and unequal local data sizes.
Moreover, there is a privacy-efficiency tradeoff for $Z_g$, where increasing $g_i$ improves model performance at the cost of greater data leakage. 
Fortunately, in practice, a small $g_i$ suffices to achieve performance comparable to transmitting all $n_i$ local data. For example, 
with a mean local dataset size of $60$, transmitting $g_i \equiv 6$ or even $g_i\equiv 2$ groups of local averages is enough.

Next, we learn 
$\theta_{M+1}^\star$ for a new client with $\Gamma_{M+1}=I_d$ and $n_{M+1} = 60$. 
Using $\widehat B$ learned from clients $1$-$M$, we estimate $\widehat\alpha_{M+1}$ 
via \eqref{eq:alpha-M1} and take
$\widehat\theta_{M+1} = \widehat B\widehat\alpha_{M+1}$. 
Figure \ref{fig:fig6} plots  $\|\widehat\theta_{M+1} - \theta_{M+1}^\star\|$ versus the number of clients $M$ used to learn $\widehat B$. We also compare with the baseline ``Independent'' that runs linear regression only on client $(M+1)$'s local data. Notably, well-estimated shared representations can provide better results than independent learning since the local data size is too small. While MoM and DFHT perform poorly, our estimators consistently outperform them and remain robust and effective.

\subsection{Risk Prediction in Healthcare}
We apply our method to a real-world dataset for early diabetes detection. Healthcare data is high-dimensional, and individual hospitals often lack sufficient local samples to train reliable models. Learning shared representations across hospitals can thus improve prediction. 

The original dataset consists of $9947$ 
medical records from $379$ hospitals. We use data from the $M = 103$ hospitals with at least $20$ local samples. 
Each record includes a response $y_{ij}\in\{0,1\}$ indicating diabetes diagnosis for patient $j$ at hospital $i$, 
and $d=185$ features after preprocessing. 
Although the ground-truth parameters are unavailable, Figure \ref{fig:fig7} plots the logarithm of the eigenvalues of the $d \times M$ matrix $[\widehat{\theta}_1, \ldots, \widehat{\theta}_M]$ as a proxy, where $\widehat{\theta}_i$ is the locally trained linear regression coefficients for 
hospital $i$. The rapid decay of eigenvalues around $20$ indicates that these parameters mainly lie in a shared low-dimensional subspace. 

 To identify the best linear classifier, 
 we 
 evaluate several methods using out-of-sample AUC at two hospitals: a large one ($n_A = 206$) and a small one ($n_B = 95$). 
Beyond MoM, DFHT, and $Z_g$, which are designed for the subspace model \eqref{eq:model-sup}, we implement additional methods for comparison. ``Global'' is trained using data from all hospitals; ``Local'' is trained only on the local data; ``RME'' \citep{xu2025multitask} is designed for a sparse heterogeneity model, which first computes a trimmed mean of locally trained parameters by removing the top and bottom $\omega$ quantiles, and then applies
LASSO locally to debias the estimate; and ``ARMUL'' \citep{duan2023adaptive} is an alternating minimization method with default parameters for a low-rank robust multi-task learning model, where clients share similar but non-identical subspaces. 
We plot the averaged results over $100$ trials, along with the confidence intervals. In each trial, we randomly split the datasets at each hospital into training ($0.8$) and testing ($0.2$) sets. 
   
\begin{figure}[ht]
    \centering
    
    \begin{minipage}[b]{0.32\textwidth}
        \centering
        \includegraphics[width=\textwidth]{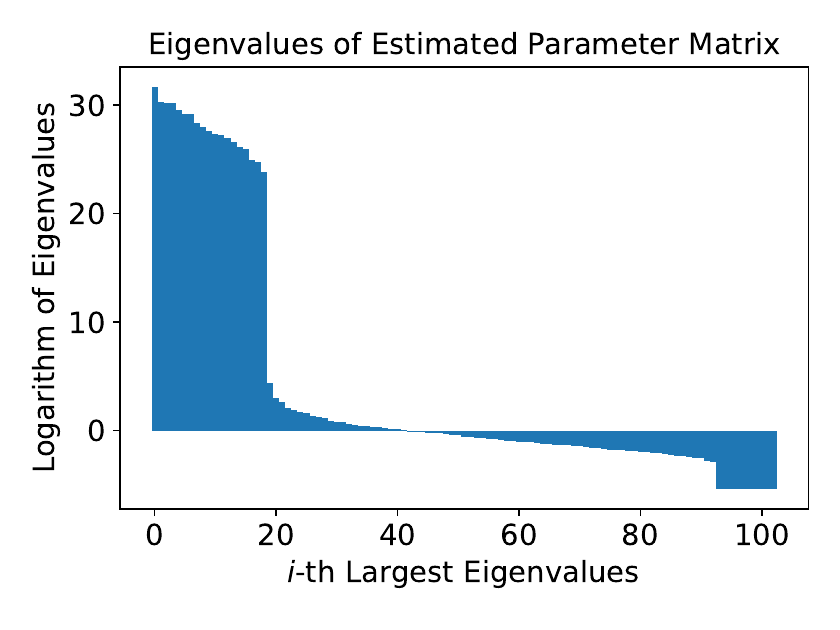}
        \caption{\small Subspace structure}
        \label{fig:fig7}
    \end{minipage}
    \hfill
    \begin{minipage}[b]{0.33\textwidth}
        \centering
        \includegraphics[width=\textwidth]{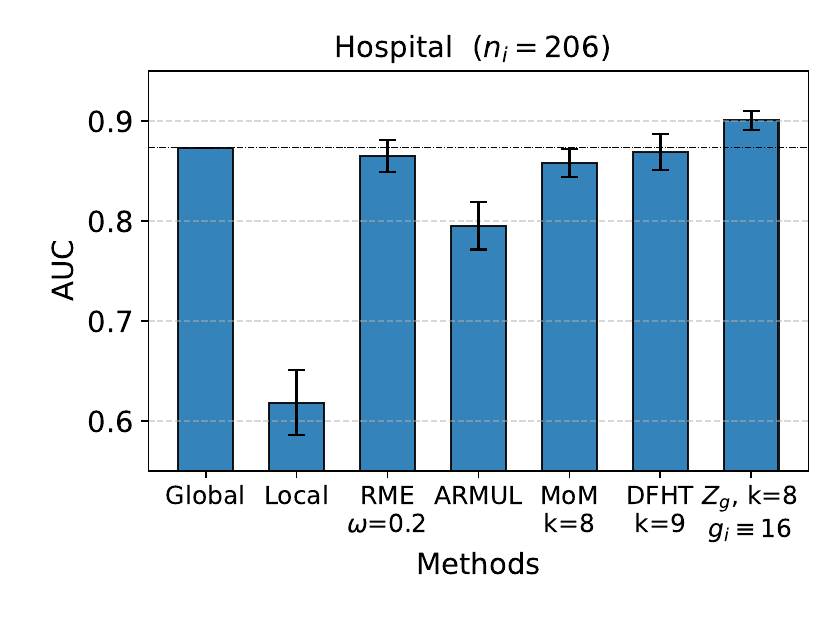}
        \caption{\small Large hospital}
        \label{fig:fig8}
    \end{minipage}
    \hfill
    \begin{minipage}[b]{0.33\textwidth}
        \centering
\includegraphics[width=\textwidth]{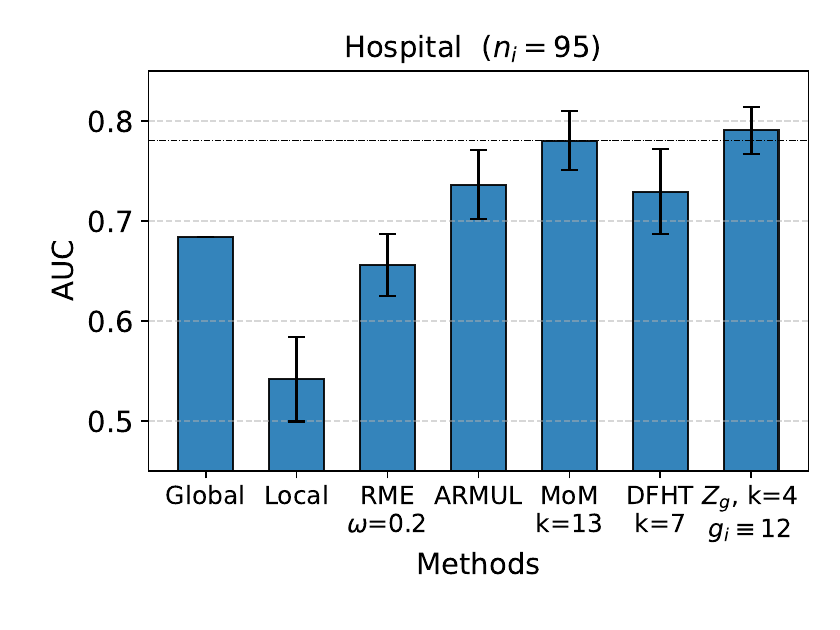}
        \caption{\small Small hospital}
        \label{fig:fig9}
    \end{minipage}
\end{figure}

Figures \ref{fig:fig8} and \ref{fig:fig9} show that our estimator $Z_g$ 
achieves the best performance at both hospitals. 
Local models perform the worst because the local datasets are too small: only around $206\times 0.8$ and $95\times 0.8$ training samples at the two hospitals.
At the large hospital, $Z_g$ improves upon the second-best method, Global, by $3.2\%$; at the small hospital, $Z_g$ improves upon the second-best method, MoM, by $1.3\%$.
Overall, our estimator $Z_g$ outperforms all others at both hospitals, highlighting its effectiveness. Moreover, it preserves patient privacy by transmitting only local averages to the server, which is particularly important in healthcare applications.

\section{Conclusion and Future Directions} \label{sect:conclusion}
We propose a spectral estimator with independent replicas of local averages, establish an improved error upper bound for learning shared representations under heterogeneous data partitions and covariate or model shifts, and extend the results to non-linear models.
We also tighten the minimax lower bound, identifying the optimal rate in well-represented cases. 

A future direction is to extend our analysis to more general scenarios where clients share similar but not identical low-dimensional representations. For example, can we obtain similarly tight results for the models in \cite{chua2021fine,tian2023learning,duan2023adaptive}? Moreover, 
it would be interesting to explore effective shared representation-learning methods for deep architectures, like transformers. 

\section*{Acknowledgments}
J. Xu is deeply grateful to Laurent Massouli\'{e} for highlighting the gaps between the existing upper and lower bounds in estimating the shared linear representation  $B^\star$, and for the insightful discussions on how to address these gaps. X. Niu and J. Xu are supported in part by an NSF CAREER award CCF-2144593. L. Su is supported in part by NSF CAREER award CCF-2340482. 
P. Yang is supported in part by an NSFC grant 12101353.

\bibliographystyle{apalike} 
\bibliography{reference}  

@article{even2022sample,
  title={On sample optimality in personalized collaborative and federated learning},
  author={Even, Mathieu and Massouli{\'e}, Laurent and Scaman, Kevin},
  journal={Advances in Neural Information Processing Systems},
  volume={35},
  pages={212--225},
  year={2022}
}

@article{bai2008large,
  title={Large dimensional factor analysis},
  author={Bai, Jushan and Ng, Serena},
  journal={Foundations and Trends{\textregistered} in Econometrics},
  volume={3},
  number={2},
  pages={89--163},
  year={2008},
  publisher={Now Publishers, Inc.}
}

@article{ghosh2020efficient,
  title={An efficient framework for clustered federated learning},
  author={Ghosh, Avishek and Chung, Jichan and Yin, Dong and Ramchandran, Kannan},
  journal={Advances in Neural Information Processing Systems},
  volume={33},
  pages={19586--19597},
  year={2020}
}

@article{duan2023adaptive,
  title={Adaptive and robust multi-task learning},
  author={Duan, Yaqi and Wang, Kaizheng},
  journal={The Annals of Statistics},
  volume={51},
  number={5},
  pages={2015--2039},
  year={2023},
  publisher={Institute of Mathematical Statistics}
}

@article{tian2023learning,
  title={Learning from similar linear representations: Adaptivity, minimaxity, and robustness},
  author={Tian, Ye and Gu, Yuqi and Feng, Yang},
  journal={Journal of Machine Learning Research},
  volume={26},
  number={187},
  pages={1--125},
  year={2025}
}

@article{thekumparampil2021statistically,
  title={Statistically and computationally efficient linear meta-representation learning},
  author={Thekumparampil, Kiran K and Jain, Prateek and Netrapalli, Praneeth and Oh, Sewoong},
  journal={Advances in Neural Information Processing Systems},
  volume={34},
  pages={18487--18500},
  year={2021}
}

@article{su2024global,
  title={Global Convergence of Federated Learning for Mixed Regression},
  author={Su, Lili and Xu, Jiaming and Yang, Pengkun},
  journal={IEEE Transactions on Information Theory},
  year={2024},
  publisher={IEEE}
}

@inproceedings{kong2020meta,
  title={Meta-learning for mixed linear regression},
  author={Kong, Weihao and Somani, Raghav and Song, Zhao and Kakade, Sham and Oh, Sewoong},
  booktitle={International Conference on Machine Learning},
  pages={5394--5404},
  year={2020},
  organization={PMLR}
}

@inproceedings{tripuraneni2021provable,
  title={Provable meta-learning of linear representations},
  author={Tripuraneni, Nilesh and Jin, Chi and Jordan, Michael},
  booktitle={International Conference on Machine Learning},
  pages={10434--10443},
  year={2021},
  organization={PMLR}
}

@inproceedings{collins2021exploiting,
  title={Exploiting shared representations for personalized federated learning},
  author={Collins, Liam and Hassani, Hamed and Mokhtari, Aryan and Shakkottai, Sanjay},
  booktitle={International Conference on Machine Learning},
  pages={2089--2099},
  year={2021},
  organization={PMLR}
}

@inproceedings{hopkins2016fast,
  title={Fast spectral algorithms from sum-of-squares proofs: tensor decomposition and planted sparse vectors},
  author={Hopkins, Samuel B and Schramm, Tselil and Shi, Jonathan and Steurer, David},
  booktitle={Proceedings of the forty-eighth annual ACM symposium on Theory of Computing},
  pages={178--191},
  year={2016}
}

@article{chen2021spectral,
  title={Spectral methods for data science: A statistical perspective},
  author={Chen, Yuxin and Chi, Yuejie and Fan, Jianqing and Ma, Cong and others},
  journal={Foundations and Trends{\textregistered} in Machine Learning},
  volume={14},
  number={5},
  pages={566--806},
  year={2021},
  publisher={Now Publishers, Inc.}
}

@article{liu2022optimal,
  title={Optimal estimation for lower bound of the packing number},
  author={Liu, Youming and Qi, Xinyu},
  journal={Statistics \& Probability Letters},
  volume={186},
  pages={109487},
  year={2022},
  publisher={Elsevier}
}

@book{vershynin2018high,
  title={High-dimensional probability: An introduction with applications in data science},
  author={Vershynin, Roman},
  volume={47},
  year={2018},
  publisher={Cambridge university press}
}

@inproceedings{
du2021fewshot,
title={Few-Shot Learning via Learning the Representation, Provably},
author={Simon Shaolei Du and Wei Hu and Sham M. Kakade and Jason D. Lee and Qi Lei},
booktitle={International Conference on Learning Representations},
year={2021}
}

@article{duchi2022subspace,
  title={Subspace recovery from heterogeneous data with non-isotropic noise},
  author={Duchi, John C and Feldman, Vitaly and Hu, Lunjia and Talwar, Kunal},
  journal={Advances in Neural Information Processing Systems},
  volume={35},
  pages={5854--5866},
  year={2022}
}

@inproceedings{thaker2023leveraging,
  title={On the benefits of public representations for private transfer learning under distribution shift},
  author={Thaker, Pratiksha and Setlur, Amrith and Wu, Steven and Smith, Virginia},
  booktitle={The Thirty-eighth Annual Conference on Neural Information Processing Systems},
  year={2023}
}

@inproceedings{dwork2006calibrating,
  title={Calibrating noise to sensitivity in private data analysis},
  author={Dwork, Cynthia and McSherry, Frank and Nissim, Kobbi and Smith, Adam},
  booktitle={Theory of Cryptography: Third Theory of Cryptography Conference, TCC 2006, New York, NY, USA, March 4-7, 2006. Proceedings 3},
  pages={265--284},
  year={2006},
  organization={Springer}
}

@article{wedin1972perturbation,
  title={Perturbation bounds in connection with singular value decomposition},
  author={Wedin, Per-{\AA}ke},
  journal={BIT Numerical Mathematics},
  volume={12},
  pages={99--111},
  year={1972},
  publisher={Springer}
}

@InProceedings{pmlr-v178-varshney22a,
  title = 	 {(Nearly) Optimal Private Linear Regression for Sub-Gaussian Data via Adaptive Clipping},
  author =       {Varshney, Prateek and Thakurta, Abhradeep and Jain, Prateek},
  booktitle = 	 {Proceedings of Thirty Fifth Conference on Learning Theory},
  pages = 	 {1126--1166},
  year = 	 {2022},
  volume = 	 {178},
  month = 	 {02--05 Jul},
  publisher =    {PMLR},
  url = 	 {https://proceedings.mlr.press/v178/varshney22a.html},
}

@article{cai2021cost,
  title={The cost of privacy: Optimal rates of convergence for parameter estimation with differential privacy},
  author={Cai, T Tony and Wang, Yichen and Zhang, Linjun},
  journal={The Annals of Statistics},
  volume={49},
  number={5},
  pages={2825--2850},
  year={2021},
  publisher={Institute of Mathematical Statistics}
}

@inproceedings{zhang2024sample,
  title={Sample-Efficient Linear Representation Learning from Non-IID Non-Isotropic Data},
  author={Zhang, Thomas TCK and Toso, Leonardo Felipe and Anderson, James and Matni, Nikolai},
  booktitle={The Twelfth International Conference on Learning Representations},
  year={2024}
}

@article{baxter2000model,
  title={A model of inductive bias learning},
  author={Baxter, Jonathan},
  journal={Journal of artificial intelligence research},
  volume={12},
  pages={149--198},
  year={2000}
}

@article{maurer2016benefit,
  title={The benefit of multitask representation learning},
  author={Maurer, Andreas and Pontil, Massimiliano and Romera-Paredes, Bernardino},
  journal={Journal of Machine Learning Research},
  volume={17},
  number={81},
  pages={1--32},
  year={2016}
}

@article{ando2005framework,
  title={A framework for learning predictive structures from multiple tasks and unlabeled data.},
  author={Ando, Rie Kubota and Zhang, Tong and Bartlett, Peter},
  journal={Journal of machine learning research},
  volume={6},
  number={11},
  year={2005}
}

@article{caruana1997multitask,
  title={Multitask learning},
  author={Caruana, Rich},
  journal={Machine learning},
  volume={28},
  pages={41--75},
  year={1997},
  publisher={Springer}
}

@incollection{thrun1998learning,
  title={Learning to learn: Introduction and overview},
  author={Thrun, Sebastian and Pratt, Lorien},
  booktitle={Learning to learn},
  pages={3--17},
  year={1998},
  publisher={Springer}
}

@article{fallah2020personalized,
  title={Personalized federated learning with theoretical guarantees: A model-agnostic meta-learning approach},
  author={Fallah, Alireza and Mokhtari, Aryan and Ozdaglar, Asuman},
  journal={Advances in neural information processing systems},
  volume={33},
  pages={3557--3568},
  year={2020}
}

@inproceedings{lee2018gradient,
  title={Gradient-based meta-learning with learned layerwise metric and subspace},
  author={Lee, Yoonho and Choi, Seungjin},
  booktitle={International Conference on Machine Learning},
  pages={2927--2936},
  year={2018},
  organization={PMLR}
}

@article{khodak2019adaptive,
  title={Adaptive gradient-based meta-learning methods},
  author={Khodak, Mikhail and Balcan, Maria-Florina F and Talwalkar, Ameet S},
  journal={Advances in Neural Information Processing Systems},
  volume={32},
  year={2019}
}

@article{hanneke2019value,
  title={On the value of target data in transfer learning},
  author={Hanneke, Steve and Kpotufe, Samory},
  journal={Advances in Neural Information Processing Systems},
  volume={32},
  year={2019}
}

@article{cai2021transfer,
  title={Transfer learning for nonparametric classification: Minimax rate and adaptive classifier},
  author={Cai, T Tony and Wei, Hongji},
  journal={The Annals of Statistics},
  volume={49},
  number={1},
  year={2021}
}

@inproceedings{dwork2006our,
  title={Our data, ourselves: Privacy via distributed noise generation},
  author={Dwork, Cynthia and Kenthapadi, Krishnaram and McSherry, Frank and Mironov, Ilya and Naor, Moni},
  booktitle={Advances in Cryptology-EUROCRYPT 2006: 24th Annual International Conference on the Theory and Applications of Cryptographic Techniques, St. Petersburg, Russia, May 28-June 1, 2006. Proceedings 25},
  pages={486--503},
  year={2006},
  organization={Springer}
}

@article{loh2013regularized,
  title={Regularized M-estimators with nonconvexity: Statistical and algorithmic theory for local optima},
  author={Loh, Po-Ling and Wainwright, Martin J},
  journal={Advances in Neural Information Processing Systems},
  volume={26},
  year={2013}
}

@article{nelder1972generalized,
  title={Generalized linear models},
  author={Nelder, John Ashworth and Wedderburn, Robert WM},
  journal={Journal of the Royal Statistical Society Series A: Statistics in Society},
  volume={135},
  number={3},
  pages={370--384},
  year={1972},
  publisher={Oxford University Press}
}

@article{mccullagh1989generalized,
  title={Generalized linear models},
  author={McCullagh, P and Nelder, JA},
  journal={Monographs on statistics and applied probability},
  year={1989}
}

@article{tripuraneni2020theory,
  title={On the theory of transfer learning: The importance of task diversity},
  author={Tripuraneni, Nilesh and Jordan, Michael and Jin, Chi},
  journal={Advances in neural information processing systems},
  volume={33},
  pages={7852--7862},
  year={2020}
}

@article{kairouz2021advances,
  title={Advances and open problems in federated learning},
  author={Kairouz, Peter and McMahan, H Brendan and Avent, Brendan and Bellet, Aur{\'e}lien and Bennis, Mehdi and Bhagoji, Arjun Nitin and Bonawitz, Kallista and Charles, Zachary and Cormode, Graham and Cummings, Rachel and others},
  journal={Foundations and trends{\textregistered} in machine learning},
  volume={14},
  number={1--2},
  pages={1--210},
  year={2021},
  publisher={Now Publishers, Inc.}
}

@article{singhal2021federated,
  title={Federated reconstruction: Partially local federated learning},
  author={Singhal, Karan and Sidahmed, Hakim and Garrett, Zachary and Wu, Shanshan and Rush, John and Prakash, Sushant},
  journal={Advances in Neural Information Processing Systems},
  volume={34},
  pages={11220--11232},
  year={2021}
}

@article{smith2017federated,
  title={Federated multi-task learning},
  author={Smith, Virginia and Chiang, Chao-Kai and Sanjabi, Maziar and Talwalkar, Ameet S},
  journal={Advances in neural information processing systems},
  volume={30},
  year={2017}
}

@article{jordan1875essai,
  title={Essai sur la g{\'e}om{\'e}trie {\`a} $ n $ dimensions},
  author={Jordan, Camille},
  journal={Bulletin de la Soci{\'e}t{\'e} math{\'e}matique de France},
  volume={3},
  pages={103--174},
  year={1875}
}

@article{chua2021fine,
  title={How fine-tuning allows for effective meta-learning},
  author={Chua, Kurtland and Lei, Qi and Lee, Jason D},
  journal={Advances in Neural Information Processing Systems},
  volume={34},
  pages={8871--8884},
  year={2021}
}

@inproceedings{zhang2014lower,
  title={Lower bounds on the performance of polynomial-time algorithms for sparse linear regression},
  author={Zhang, Yuchen and Wainwright, Martin J and Jordan, Michael I},
  booktitle={Conference on Learning Theory},
  pages={921--948},
  year={2014},
  organization={PMLR}
}

@article{chen2016bayes,
  title={On {B}ayes risk lower bounds},
  author={Chen, Xi and Zhang, Yuchen},
  journal={Journal of Machine Learning Research},
  volume={17},
  number={218},
  pages={1--58},
  year={2016}
}

@article{raskutti2011minimax,
  title={Minimax rates of estimation for high-dimensional linear regression over $\ell_q$-balls},
  author={Raskutti, Garvesh and Wainwright, Martin J and Yu, Bin},
  journal={IEEE transactions on information theory},
  volume={57},
  number={10},
  pages={6976--6994},
  year={2011},
  publisher={IEEE}
}

@article{xu2025multitask,
  title={Multitask learning and bandits via robust statistics},
  author={Xu, Kan and Bastani, Hamsa},
  journal={Management Science},
  year={2025},
  publisher={INFORMS}
}

@article{niu2024collaborative,
  title={Collaborative learning with shared linear representations: Statistical rates and optimal algorithms},
  author={Niu, Xiaochun and Su, Lili and Xu, Jiaming and Yang, Pengkun},
  journal={arXiv preprint arXiv:2409.04919},
  year={2024}
}

\begin{appendix}
\section{Proofs of the Lemmas in Section \ref{sec:proof-up}}
We present the proofs of Lemmas \ref{lem:main-sup}-\ref{lem:noise-sup} here.
\subsection{Auxiliary Results}
We begin with several auxiliary results, 
including the truncated matrix Bernstein inequality as follows, and properties of the random vectors $a_i$, $u_i$, $b_i$, and $v_i$.
\begin{theorem}[\citep{chen2021spectral}]\label{cor:bernstein}
Let $Z_1, \cdots, Z_M\in\R^{d_1\times d_2}$ be independent random matrices. Suppose there exist positive constants $\beta$, $q$, and $\delta \le 1$, such that for all $i\in[M]$,
\$
& \pr\big(\|Z_i - \E Z_i \| \ge \beta \big) \le \delta, \quad \big\|\E Z_i - \E[Z_i\mathds{1}\{\|Z_i\|<\beta\}] \big\| \le q.
\$
In addition, let $v$ be the matrix variance statistic defined as
\$
v = \max\Big\{\Big\|\sum_{i=1}^M \big(\E [Z_i Z_i^\intercal]- (\E Z_i)(\E Z_i^\intercal)\big)\Big\|, \Big\|\sum_{i=1}^M \big(\E [Z_i^\intercal Z_i]- (\E Z_i^\intercal)(\E Z_i)\big)\Big\|\Big\}.
\$
Set $d=\max\{d_1, d_2\}$. For any $c\ge2$, with probability at least $1-2d^{-c+1} - M\delta$, we have
\$
\Big\|\sum_{i=1}^M \big(Z_i - \E Z_i\big)\Big\| \le \sqrt{2cv\log d} + \frac{2c}{3}\beta\log d + Mq.
\$
\end{theorem}

Recall that the error terms in \eqref{eq:error_decomp} involve matrices formed by the vectors $a_i$, $b_i$, $u_i$, $v_i$, where $a_i = 2/\sqrt{n_i} \cdot \sum_{j=1}^{n_i/2} \zeta_{ij} x_{ij}$, $b_i = 2/\sqrt{n_i} \cdot \sum_{j=1}^{n_i/2} \xi_{ij} x_{ij}$, and where $a_i$ and $u_i$ are identically distributed, as are $b_i$ and $v_i$. We next record several key properties of these vectors.
The following lemma will help identify an appropriate truncation level of the random vector $a_i$.

\begin{lemma}\label{lem:sub-exp-bernstein}
Let $\{\zeta_j\}_{j\in[n]}\subset\R$ and $\{x_j\}_{j\in[n]}\subset\R^d$ be sequences of sub-gaussian random variables and vectors, respectively, with constant variance proxies. 
Assume the pairs $(\zeta_j, x_j)$ are mutually independent across $j$, while $\zeta_j$ and $x_j$ may be dependent.
Define the normalized sum $\eta = \sum_{j=1}^{n} \zeta_j x_j/\sqrt{n}$. Then there exist constants $c_1,c_2>0$ such that for any $t>0$,
\$
\pr\big(\|\eta - \E \eta\| \ge t\big)\le 2d\exp\big(- \min\big\{c_1 t^2/d, c_2 t\sqrt{n/d}\big\}\big).
\$
\end{lemma}
\begin{proof}
    Let $x_{j}^{(r)}$ denote the $r$-th entry of vector $x_{j}$, and define $\eta^{(r)} = \sum_{j=1}^{n}\zeta_{j} x_{j}^{(r)}/\sqrt{n}$ as the $r$-th entry of $\eta$. Then we have $\|\eta-\E \eta\|^2 = \sum_{r=1}^d(\eta^{(r)}-\E \eta^{(r)})^2$ and thus 
    \$
    \pr\big(\|\eta - \E \eta\| \ge t\big) = \pr\Big(\sum_{r=1}^d\big(\eta^{(r)} - \E \eta^{(r)}\big)^2 \ge t^2\Big) \le \sum_{r=1}^d\pr\Big(\big|\eta^{(r)} - \E \eta^{(r)}\big| \ge t/\sqrt{d}\Big). 
    \$
    For a fixed $r$, the product of sub-gaussians, $\zeta_{j}x_{j}^{(r)}$, is sub-exponential; $\{\zeta_{j}x_{j}^{(r)} - \E\zeta_{j}x_{j}^{(r)}\}_{j\in[n]}$ is a sequence of independent, mean zero, sub-exponential random variables. Thus, Bernstein's inequality~\cite
    {vershynin2018high} yields the existence of constants $c_1, c_2>0$ such that 
    \$
    \pr\Big(\big|\eta^{(r)} - \E \eta^{(r)}\big| \ge t/\sqrt{d}\Big) &= \pr\bigg(\Big |\frac{1}{\sqrt{n}}\sum_{j=1}^{n} \Big(\zeta_{j} x_{j}^{(r)} - \E\zeta_{j} x_{j}^{(r)}\Big)\Big | \ge t/\sqrt{d}\bigg) \notag\\
    & \le 2\exp\big(- \min\big\{c_1 t^2/d, c_2 t\sqrt{n/d}\big\}\big).
    \$
    We conclude the proof by combining the above two equations.
\end{proof}

We now show the sub-gaussian property of $\xi_{ij}$, which enables us to similarly bound the truncation level of $b_i$. 
\begin{lemma}\label{lem:sub-gaussian-xi}
Under Assumption \ref{ass:general}, 
$\xi_{ij}$ is sub-gaussian with a constant variance proxy.
\end{lemma}
\begin{proof}
Let constant $c>0$ be the variance proxy of $\xi_{ij} \given x_{ij}$.
Since $\xi_{ij} \given x_{ij} \sim \text{subG}(c)$ and $\E[\xi_{ij} \given x_{ij}] = 0$, we have for any $t\in\R$, $
\E[\exp(t\xi_{ij})\given x_{ij}] = \E[\exp(t(\xi_{ij}-\E[\xi_{ij} \given x_{ij}])) \given x_{ij}] \le \exp(ct^2/2)$.
Then we have $\E\xi_{ij} = \E[\E[\xi_{ij}\given x_{ij}]] = 0$; thus the sub-gaussian property of $\xi_{ij}$ is shown as 
$
\E[\exp(t(\xi_{ij} - \E\xi_{ij}))] = \E[\exp(t\xi_{ij})] =\E[\E[\exp(t\xi_{ij})\given x_{ij}]] \le \exp(ct^2/2)$.
\end{proof}

The following lemma assists in bounding the mean shift after truncation.
\begin{lemma}\label{lem:bound-mean-shift} Let $p_i,q_i\in\R^d$ be independent random vectors and $Z_i=p_iq_i^\intercal$. For any $\beta>0$,
\$
\big\|\E Z_i - \E[Z_i\mathds{1}\{\|Z_i\|<\beta\}] \big\| \le \sqrt{\E\big[\|p_i\|^2\big]\E\big[\|q_i\|^2\big]\pr\big(\|Z_i\|\ge\beta\big)}.
\$
\end{lemma}
\begin{proof}
We bound the mean shift after truncation as follows,
\$
\big\|\E Z_i - \E[Z_i\mathds{1}\{\|Z_i\|<\beta\}] \big\| & 
\le \E\big[\|Z_i\|\mathds{1}\{\|Z_i\|\ge\beta\}\big] \le \sqrt{\E\big[\|Z_i\|^2\big]\E\big[\mathds{1}\{\|Z_i\|\ge\beta\}\big]},
\$
where Cauchy--Schwarz inequality gives the last inequality. Since $a_i$ and $b_i$ are independent, 
$\E[\|Z_i\|^2] = \E[\|p_i\|^2\|q_i\|^2] = \E[\|p_i\|^2]\E[\|q_i\|^2]$. Also note $\E[\mathds{1}\{\|Z_i\|\ge\beta\}]=\pr(\|Z_i\| \ge \beta)$. 
\end{proof}

Finally, the next two lemmas bound the variance of random vectors $a_i$ and $b_i$.
\begin{lemma}\label{lem:var-1}
Under Assumption \ref{ass:general}, for all $i$, the vector $a_i = 2/\sqrt{n_i}\cdot\sum_{j=1}^{n_i/2}\zeta_{ij} x_{ij}$ satisfies 
\$
\E a_i = \sqrt{n_i}B^\star\alpha_i^\star, \quad \|\E a_i\| \le \sqrt{N\lambda_1}, \quad \E\big[\|a_i\|^2\big] = O(d) + n_i\|\alpha_i^\star\|^2.
\$
In addition, for any $s\in\mathbb{S}^{d-1}$, we have
$
s^\intercal \E[a_i a_i^\intercal]s = O(1) + n_i s^\intercal B^\star\alpha_i^\star(\alpha_i^\star)^\intercal (B^\star)^\intercal s$.
\end{lemma}
\begin{proof} We first note that $\E a_i = \sqrt{n_i}B^\star\alpha_i^\star$, and $\E a_i\E a_i^\intercal = n_i B^\star\alpha_i^\star(\alpha_i^\star)^\intercal (B^\star)^\intercal$.
For any $s\in\mathbb{S}^{d-1}$, we have $\|(B^\star)^\intercal s\|\le 1$ and thus
\$
s^\intercal \E a_i\E a_i^\intercal s = n_i s^\intercal B^\star\alpha_i^\star(\alpha_i^\star)^\intercal (B^\star)^\intercal s \le s^\intercal B^\star \Big(\sum_{i=1}^Mn_i\alpha_i^\star(\alpha_i^\star)^\intercal\Big) (B^\star)^\intercal s \le N\lambda_1\|(B^\star)^\intercal s\| \le N\lambda_1.
\$    
Then we have $\|\E a_i\E a_i^\intercal\| \le N\lambda_1$, which bounds the vector norm $
\|\E a_i\| = \sqrt{\|\E a_i\E a_i^\intercal\|} \le \sqrt{N\lambda_1}$. Next, we compute $\E[\|a_i\|^2]$ using its definition,
    \#\label{eq:var-ai-non}
    \E\big[\|a_i\|^2\big] & = \E\Big[\Big(\frac{2}{\sqrt{n_i}}\sum_{j=1}^{n_i/2} \zeta_{ij}x_{ij}^\intercal\Big)\Big(\frac{2}{\sqrt{n_i}}\sum_{j=1}^{n_i/2}\zeta_{ij} x_{ij}\Big)\Big] \notag\\
    & = \frac{4}{n_i}\Big[ \E\Big(\sum_{j=1}^{n_i/2} \zeta_{ij}^2 \|x_{ij}\|^2\Big) +  \E\Big(\sum_{j=1}^{n_i/2} \sum_{r\neq j} \zeta_{ij} x_{ij}^\intercal \cdot \zeta_{ir} x_{ir}\Big)\Big] \notag\\
    & \le  \frac{4}{n_i}\sum_{j=1}^{n_i/2}\E\big[\zeta_{ij}^2 \|x_{ij}\|^2\big] +  \frac{4}{n_i}\sum_{j=1}^{n_i/2} \sum_{r\neq j}\big\|\E\big[\zeta_{ij}x_{ij}\big] \big\|\cdot\big\|\E\big[\zeta_{ir}x_{ir}\big] \big\|. 
    \#
For the first term, by Cauchy--Schwarz inequality, $
\E[\zeta_{ij}^2 \|x_{ij}\|^2] \le \sqrt{\E[\zeta_{ij}^4] \E[\|x_{ij}\|^4]}$ for all $j$.
Since the moments of sub-gaussian variables are bounded by constants relying on the variance proxy, we have $\E[\zeta_{ij}^4] = O(\|\alpha_i^\star\|^4) =O(1)$.
Given bounded moments $\E[(x_{ij}^{(s)})^4]$, we have $\E[\|x_{ij}\|^4] = \E[(\sum_{s=1}^d(x_{i1}^{(s)})^2)^2] \le d \sum_{s=1}^d\E[(x_{ij}^{(s)})^4] = O(d^2)$. Thus, the first term in~\eqref{eq:var-ai-non} is of order, $
{4}/{n_i}\cdot \sum_{j=1}^{n_i/2}\E[\zeta_{ij}^2 \|x_{ij}\|^2] = O(d)$. 
In addition, each summand in the second term satisfies $\|\E[\zeta_{ij}x_{ij}] \| = \|\E[\zeta_{ir}x_{ir}] \|= \|B^\star \alpha_i^\star\| = \|\alpha_i^\star\|$ since $B\in \cO^{d\times k}$. As a summary, it follows from \eqref{eq:var-ai-non} that,
\$
\E\big[\|a_i\|^2\big] = O(d) + n_i\|\alpha_i^\star\|^2.
\$

Similarly, we compute $\E[a_i a_i^\intercal]$ as follows,
\#\label{eq:matrix-var-ai-non}
\E[a_i a_i^\intercal] & = \E\Big[\Big(\frac{2}{\sqrt{n_i}}\sum_{j=1}^{n_i/2} \zeta_{ij} x_{ij}\Big)\Big(\frac{2}{\sqrt{n_i}}\sum_{j=1}^{n_i/2}\zeta_{ij} x_{ij}^\intercal\Big)\Big] \notag\\
    & = \frac{4}{n_i}\sum_{j=1}^{n_i/2} \E \big[\zeta_{ij}^2 x_{ij}x_{ij}^\intercal\big] +  \frac{4}{n_i}\sum_{j=1}^{n_i/2}\sum_{r\neq j}\E\big[ \zeta_{ij} x_{ij}\big] \cdot \E\big[\zeta_{ir} x_{ir}^\intercal\big]. 
\#
In the first term, for any $j$ and $s\in\mathbb{S}^{d-1}$, Cauchy--Schwarz inequality gives
$
s^\intercal \E[\zeta_{ij}^2 x_{ij}x_{ij}^\intercal] s = \E[\zeta_{ij}^2 (s^\intercal x_{ij})^2]\le \sqrt{\E[\zeta_{ij}^4] \E[(s^\intercal x_{ij})^4]} = O(1)$,
where $\zeta_{ij}$ and $s^\intercal x_{ij}$ are both sub-gaussian with constant variance proxies. Recall that $\E[\zeta_{ij}x_{ij}] = \E[\zeta_{ir}x_{ir}] = B^\star \alpha_i^\star$ in the second term.
Thus, by \eqref{eq:matrix-var-ai-non}, we have for any $s\in\mathbb{S}^{d-1}$ that
$
s^\intercal \E[a_i a_i^\intercal] s = O(1) + n_is^\intercal B^\star\alpha_i^\star(\alpha_i^\star)^\intercal (B^\star)^\intercal s
$.
\end{proof}

\begin{lemma}\label{lem:var-2}
Under Assumption \ref{ass:general}, for any $i$, the vector $b_i = 2/\sqrt{n_i}\cdot\sum_{j=1}^{n_i/2}\xi_{ij}x_{ij}$ satisfies 
\$
\E b_i = 0, \quad \E\big[\|b_i\|^2\big] = O(d), \quad \big\|\E[b_ib_i^\intercal]\big\|= O(1).
\$   
\end{lemma}
\begin{proof}
Given $\E[\xi_{ij}\given x_{ij}] = 0$, $\E[\xi_{ij} x_{ij}] = \E[\E[\xi_{ij}\given x_{ij}] x_{ij}] = 0$ for any $i,j$. By definition, 
$\E b_i = ({2}/{\sqrt{n_i}})\cdot\sum_{j=1}^{n_i/2}\E\big[\xi_{ij}x_{ij}\big] = 0$.
In addition, we have 
\$
    \E\big[\|b_i\|^2\big] &= \E\Big[\Big(\frac{2}{\sqrt{n_i}}\sum_{j=1}^{n_i/2}\xi_{ij} x_{ij}\Big)^\intercal\Big(\frac{2}{\sqrt{n_i}}\sum_{j=1}^{n_i/2}\xi_{ij} x_{ij}\Big)\Big] \notag\\
    & = \frac{4}{n_i} \sum_{j=1}^{n_i/2}\sum_{r=1}^{n_i/2} \E\big[\xi_{ij} x_{ij}^\intercal \cdot \xi_{ir} x_{ir}\big] = \frac{4}{n_i} \sum_{j=1}^{n_i/2}\E[\xi_{ij}^2\|x_{ij}\|^2] = \frac{4}{n_i} \sum_{j=1}^{n_i/2}\E\big[\E[\xi_{ij}^2\given x_{ij}]\cdot\|x_{ij}\|^2\big] \\
    & = O(1)\cdot \frac{4}{n_i} \sum_{j=1}^{n_i/2}\E\big[\|x_{ij}\|^2\big] = O(d),
\$
where the last two equalities hold since $\E[\xi_{ij}^2\given x_{ij}] = O(1)$ and $\E[\|x_{ij}\|^2] = \sum_{r=1}^d\E(x_{ij}^{(r)})^2 = O(d)$. Similarly, it is easy to show that $\|\E[b_ib_i^\intercal]\| = O(1)$ since for any $s\in\mathbb{S}^{d-1}$,
$$
    s^\intercal\E[b_ib_i^\intercal]s= \frac{4}{n_i} \sum_{j=1}^{n_i/2}\E\big[\E[\xi_{ij}^2\given x_{ij}]\cdot (s^\intercal x_{ij})^2\big] = O(1). \qedhere
$$
\end{proof}
\vspace{-10pt}
\subsection{Proofs of Lemmas \ref{lem:main-sup}-\ref{lem:noise-sup}}
We are now ready to prove Lemmas \ref{lem:main-sup}-\ref{lem:noise-sup}.
\medskip
\begin{proof}[Proof of Lemma \ref{lem:main-sup}]
Let $Z_i = a_iu_i^\intercal$. Note that $a_i$ and $u_i$ are independent and identically distributed, with properties outlined in Lemma \ref{lem:var-1}. We aim to bound $\|\sum_{i=1}^M (a_iu_i^\intercal -  n_i B^\star\alpha_i^\star(\alpha_i^\star)^\intercal (B^\star)^\intercal) \| = \|\sum_{i=1}^M (Z_i -  \E Z_i)\|$ by the truncated matrix Bernstein's inequality. Since $\|Z_i-  \E Z_i\|$ might be unbounded, we first identify a truncation level. By adding and subtracting terms, we have
\$
\big\|Z_i -  \E Z_i\big\| & = \big\|a_iu_i^\intercal -  \E a_i \E u_i^\intercal \big\| \\
&\le \big\|(a_i - \E a_i) (u_i - \E u_i)^\intercal \big\| + \big\|(a_i - \E  a_i) \E u_i^\intercal \big\| + \big\|\E  a_i (u_i -\E u_i)^\intercal \big\|.
\$
Since $a_i$ and $u_i$ share identical distributions, the above equation yields that for any $\beta>0$,
\$
\pr\big(\|Z_i -  \E Z_i\| \ge \beta \big) & \le \pr\big(\|(a_i - \E a_i) (u_i - \E u_i)^\intercal \| \ge \beta/3\big) + 2\pr\big(\|(a_i - \E  a_i) \E u_i^\intercal \| \ge \beta/3\big)\\
& \le 2\pr\big(\|a_i - \E a_i\| \ge \sqrt{\beta/3}\big) + 2\pr\big(\|\E u_i\| \cdot \|a_i - \E a_i\| \ge \beta/3\big) \\
& \le 2\pr\big(\|a_i - \E a_i\| \ge \sqrt{\beta/3}\big) + 2\pr\big(\|a_i - \E a_i\| \ge \beta/\big(3\sqrt{N\lambda_1}\big)\big),
\$
where the last inequality holds since $\|\E u_i\| \le \sqrt{N\lambda_1}$ by Lemma \ref{lem:var-1}.  Note that $a_i = 2/\sqrt{n_i}\cdot\sum_{j=1}^{n_i/2}\zeta_{ij}x_{ij}$, where $\zeta_{ij}$ is a sub-gaussian variable and $x_{ij}$ is a sub-gaussian vector. Thus, following from the above equation, we apply the bound from Lemma \ref{lem:sub-exp-bernstein} to $\|a_i-\E a_i\|$, and have the existence of constants $c_1,c_2,c_3,c_4$ such that for any $\beta > 0$, 
\$
    \pr\big(\|Z_i-  \E Z_i\| \ge \beta\big) &\le 4d\exp\big(- \min\big\{c_1 \beta/d, c_2 \sqrt{n_i\beta/d}\big\}\big) \\
    &\qquad + 4d\exp\big(- \min\big\{c_3 \beta^2/(Nd\lambda_1), c_4\beta\sqrt{n_i/(Nd\lambda_1)}\big\}\big).
\$
For any $\delta >0$, we take a large enough $C$ and set $\beta = C\max\{d\log^2(d/\delta ), \sqrt{Nd\lambda_1}\log(d/\delta )\}$.
Thus, using $n_i\ge 1$, we obtain from the above equation,
\#\label{main-error-truncate}
\pr\big(\|Z_i -  \E Z_i\| \ge \beta\big) \le \delta .
\#

Next, we bound the mean shift after truncation by Lemma \ref{lem:bound-mean-shift}. For the above $\beta$, we have
$
\|\E Z_i - \E[Z_i\mathds{1}\{\|Z_i\|<\beta\}] \big\| \le \sqrt{\E[\|a_i\|^2]\E[\|u_i\|^2]}\sqrt{\delta }$.
Lemma \ref{lem:var-1} provides the bound that $\E[\|a_i\|^2]=\E[\|u_i\|^2]=O(d+N)$. Thus, we have
\#\label{main-error-q}
\big\|\E Z_i - \E[Z_i\mathds{1}\{\|Z_i\|<\beta\}] \big\| =  O\big((d+N)\sqrt{\delta}\big)\coloneqq q.
\#

It remains to determine the variance statistic. First, note that
$
\E[Z_iZ_i^\intercal] = \E[a_iu_i^\intercal u_ia_i^\intercal] = \E[\|u_i\|^2]\E[a_ia_i^\intercal]$, and $\E Z_i = n_i B^\star\alpha_i^\star(\alpha_i^\star)^\intercal (B^\star)^\intercal$. Lemma \ref{lem:var-1} yields that, for any $s\in \mathbb{S}^{d-1}$,
\$
& s^\intercal \big(\E[Z_iZ_i^\intercal] - (\E Z_i)(\E Z_i^\intercal) \big) s \\
&\quad = \E[\|u_i\|^2]\cdot s^\intercal\E[a_ia_i^\intercal]s  - n_i^2 \|B^\star\alpha_i^\star\|^2\cdot s^\intercal B^\star\alpha_i^\star(\alpha_i^\star)^\intercal (B^\star)^\intercal s  \\
&\quad = \big( O(d) + n_i\|\alpha_i^\star\|^2 \big)\cdot \big(O (1) + n_i s^\intercal B^\star\alpha_i^\star(\alpha_i^\star)^\intercal (B^\star)^\intercal s\big) - n_i^2 \|\alpha_i^\star\|^2\cdot s^\intercal B^\star\alpha_i^\star(\alpha_i^\star)^\intercal (B^\star)^\intercal s\\
&\quad = O\big(d + n_i \|\alpha_i^\star\|^2 + d n_is^\intercal B^\star\alpha_i^\star(\alpha_i^\star)^\intercal (B^\star)^\intercal s\big).
\$
Since $\sum_{i=1}^Mn_i \|\alpha_i^\star\|^2 = \tr(\sum_{i=1}^M n_i \alpha_i^\star(\alpha_i^\star)^\intercal)=O(Nk\lambda_1)$, summing over all $i\in[M]$ yields 
\#\label{eq:variance-stat-pre}
& s^\intercal \Big[\sum_{i=1}^M \big(\E[Z_iZ_i^\intercal] - (\E Z_i)(\E Z_i^\intercal) \big) \Big]s = O\Big(Md + \sum_{i=1}^M n_i\|\alpha_i^\star\|^2 + Nd \cdot s^\intercal B^\star D(B^\star)^\intercal s \Big)\notag\\
&\qquad = O\big(Md + Nk\lambda_1 + Nd\lambda_1\|(B^\star)^\intercal s\| \big)  = O(Md + Nd\lambda_1). 
\#
where $D=\sum_{i=1}^M n_i \alpha_i^\star(\alpha_i^\star)^\intercal/N$, the second equality follows from $\lambda_1(D)\le \lambda_1$, and the last one from $d\ge k$ and $\|(B^\star)^\intercal s\|\le 1$. In addition, note that $\E[Z_i^\intercal Z_i] = \E[u_ia_i^\intercal a_iu_i^\intercal ] = \E[\|a_i\|^2]\E[u_iu_i^\intercal]$. Since $a_i$ and $u_i$ are identically distributed, and $\E Z_i = \E Z_i^\intercal$, we observe $\E[Z_iZ_i^\intercal] - (\E Z_i)(\E Z_i^\intercal) = \E[Z_i^\intercal Z_i] - (\E Z_i^\intercal)(\E Z_i)$. Thus, following from \eqref{eq:variance-stat-pre}, the variance statistic is bounded such that
\#\label{main-error-v}
v &= \max\Big\{\Big\|\sum_{i=1}^M \big(\E [Z_i Z_i^\intercal]- (\E Z_i)(\E Z_i^\intercal)\big)\Big\|, \Big\|\sum_{i=1}^M \big(\E [Z_i^\intercal Z_i]- (\E Z_i^\intercal)(\E Z_i)\big)\Big\|\Big\} \notag\\
& = O(Md + Nd\lambda_1). 
\#
We apply Theorem \ref{cor:bernstein} with $c=11$, and $\beta$, $q$, and $v$ discussed in \eqref{main-error-truncate}, \eqref{main-error-q}, and \eqref{main-error-v}. Then with probability at least $1-2d^{-10} - M\delta$, we have
\$
& \Big\|\sum_{i=1}^M \big(Z_i - \E Z_i\big)\Big\| \le \sqrt{2cv\log d} + 2c\beta\log d/3 + Mq \\
&\qquad = O\big(\sqrt{(Md + Nd\lambda_1)\log d} + \big(d\log^2(d/\delta )+ \sqrt{Nd\lambda_1}\log(d/\delta )\big)\log d + \sqrt{\delta}M(d+N)\big).
\$
Note that $M\le N$. Taking $\delta = (d+N)^{-11}$, with probability at least $1-O((d+N)^{-10})$,
\[
\Big\|\sum_{i=1}^M \big(Z_i - \E Z_i\big)\Big\| = O\big(\big(\sqrt{Md} + \sqrt{Nd\lambda_1} + d \big)\cdot\log^3(d+N) \big).  \qedhere
\]
\end{proof}

\begin{proof}[Proof of Lemma \ref{lem:inner-prod-sup}]
Let $Z_i \triangleq a_iv_i^\intercal$. Lemma \ref{lem:var-2} summarizes the properties of $v_i$ and $b_i$. We have $\E v_i = 0$ and thus $\E Z_i = 0$. We first identify a truncation level of $\|Z_i\|$. Since $\|Z_i\| = \|a_iv_i^\intercal\| \le \|(a_i-\E a_i)v_i^\intercal\| + \|(\E a_i)v_i^\intercal\|$, we have for any $\beta\ge0$,
    \$
    \pr\big(\|Z_i\| \ge \beta\big) &\le \pr\big(\|(a_i-\E a_i)v_i^\intercal\| \ge \beta/2\big) + \pr\big(\|(\E a_i)v_i^\intercal\| \ge \beta/2\big) \notag\\
    & \le\pr\big(\|a_i-\E a_i\| \ge \sqrt{\beta/2}\big) + \pr\big(\|v_i\| \ge \sqrt{\beta/2}\big) + \pr\big(\|v_i\| \ge \beta/2\big(\sqrt{N\lambda_1}\big)\big), 
    \$
    where the last inequality holds since $\|\E a_i\| \le \sqrt{N\lambda_1}$ by Lemma \ref{lem:var-1}. Note that $a_i = 2/\sqrt{n_i}\cdot\sum_{j=1}^{n_i/2}\zeta_{ij} x_{ij}$ and $b_i= 2/\sqrt{n_i}\cdot\sum_{j=n_i/2 + 1}^{n_i}\xi_{ij} x_{ij}$ with sub-gaussian variables $\zeta_{ij}$ and $\xi_{ij}$ and sub-gaussian vector $x_{ij}$, following from Assumption \ref{ass:general} and Lemma \ref{lem:sub-gaussian-xi}. Thus, applying the bounds on $\|a_i-\E a_i\|$ and $\|v_i\|$ from Lemma \ref{lem:sub-exp-bernstein} to the above equation, there are constants $c_1,c_2,c_3,c_4$ such that for any $\beta > 0$, 
    \$
        \pr\big(\|Z_i\| \ge \beta\big) &\le 4d\exp\big(- \min\big\{c_1 \beta/d, c_2 \sqrt{n_i\beta/d}\big\}\big) \\
        &\qquad + 2d\exp\big(- \min\big\{c_3 \beta^2/(Nd\lambda_1), c_4\beta\sqrt{n_i/(Nd\lambda_1)}\big\}\big).
    \$
    For any $\delta >0$, we take a large enough $C$ and set $\beta = C\max\{d\log^2(d/\delta ), \sqrt{Nd\lambda_1}\log(d/\delta )\}$ 
    thereby obtaining (noting $n_i\ge1$)
    \#\label{inner-prod-truncate}
    \pr\big(\|Z_i\| \ge \beta\big) \le \delta .
    \#
    
    Next, we bound the mean shift after truncation by Lemma \ref{lem:bound-mean-shift}. For the above $\beta$, we have
    $
    \|\E Z_i - \E[Z_i\mathds{1}\{\|Z_i\|<\beta\}] \| \le \sqrt{\E[\|a_i\|^2]\E[\|v_i\|^2]}\sqrt{\delta }$.
    Lemmas \ref{lem:var-1} and \ref{lem:var-2} show that $\E[\|a_i\|^2]=O(d+N)$ and $\E[\|v_i\|^2]=O(d)$. Thus, we have
    \#\label{inner-prod-q}
    \big\|\E Z_i - \E[Z_i\mathds{1}\{\|Z_i\|<\beta\}] \big\| =  O\big(\sqrt{d(d+N)}\sqrt{\delta}\big)\coloneqq q.
    \#

    It remains to determine the variance statistic. We first have
    $
    \E[Z_iZ_i^\intercal] = \E[a_iv_i^\intercal v_ia_i^\intercal] = \E[\|v_i\|^2]\E[a_ia_i^\intercal]$. 
Then Lemmas \ref{lem:var-1} and \ref{lem:var-2} yield that, for any $s\in \mathbb{S}^{d-1}$, 
    \$
    s^\intercal \E[Z_iZ_i^\intercal]s & = \E[\|v_i\|^2]\cdot s^\intercal\E[a_ia_i^\intercal]s   
    = O(d) + O(d)\cdot n_i s^\intercal B^\star\alpha_i^\star(\alpha_i^\star)^\intercal (B^\star)^\intercal s.
    \$
    Therefore, when summing over $i\in [M]$, we have
    \#\label{eq:inner-prod-variance-stat-pre}
    s^\intercal \Big(\sum_{i=1}^M \E[Z_iZ_i^\intercal] \Big)s & = O(Md) + O(d)\cdot  s^\intercal B^\star \Big(\sum_{i=1}^M n_i \alpha_i^\star(\alpha_i^\star)^\intercal \Big) (B^\star)^\intercal s \notag\\
    & = O(Md) + O(Nd\lambda_1)\cdot\|(B^\star)^\intercal s\| = O(Md + Nd\lambda_1),
    \#
    where the second line follows from $\lambda_1(\sum_{i=1}^M n_i \alpha_i^\star(\alpha_i^\star)^\intercal)\le N\lambda_1$ and the last from $\|(B^\star)^\intercal s\|\le 1$. Moreover, $\E[Z_i^\intercal Z_i] = \E[v_ia_i^\intercal a_iv_i^\intercal ] = \E[\|a_i\|^2]\E[v_iv_i^\intercal]$. Then Lemmas \ref{lem:var-1} and \ref{lem:var-2} give that
    \$
    \big\|\E[Z_i^\intercal Z_i]\big\| \le \E[\|a_i\|^2]\cdot\big\|\E[v_iv_i^\intercal]\big\| = \big(O(d) + n_i\|\alpha_i^\star\|^2\big) \cdot O(1) = O(d + n_i\|\alpha_i^\star\|^2\big).
    \$
    Thus, since $\sum_{i=1}^M n_i\|\alpha_i^\star\|^2 = \tr(\sum_{i=1}^M n_i \alpha_i\alpha_i^\intercal) = O(Nk\lambda_1)$, we have
    \#\label{eq:inner-prod-variance-stat-1}
    \Big\|\sum_{i=1}^M \E [Z_i^\intercal Z_i]\Big\| & \le \sum_{i=1}^M\big\|\E[Z_i^\intercal Z_i]\big\| = O(Md) + O\Big(\sum_{i=1}^M n_i\|\alpha_i^\star\|^2\Big) = O(Md + Nk\lambda_1).
    \#
    By \eqref{eq:inner-prod-variance-stat-pre} and \eqref{eq:inner-prod-variance-stat-1}, and recalling that $\E Z_i=0$ and $k\le d$, we bound the variance statistic as 
    \#\label{inner-prod-v}
    v &= \max\Big\{\Big\|\sum_{i=1}^M \E [Z_i Z_i^\intercal]\Big\|, \Big\|\sum_{i=1}^M \E [Z_i^\intercal Z_i]\Big\|\Big\} 
    = O(Md + Nd\lambda_1).
    \#
    We apply Theorem \ref{cor:bernstein} with $c=11$, and $\beta$, $q$, and $v$ discussed in \eqref{inner-prod-truncate}, \eqref{inner-prod-q}, and \eqref{inner-prod-v}. Then with probability at least $1-2d^{-10} - M\delta$, we have
    \$
    & \Big\|\sum_{i=1}^M \big(Z_i - \E Z_i\big)\Big\| \le \sqrt{2cv\log d} + 2c\beta\log d/3 + Mq \\
    &\quad = O\big(\sqrt{(Md + Nd\lambda_1)\log d} + \big(d\log^2(d/\delta )+ \sqrt{Nd\lambda_1}\log(d/\delta )\big)\log d + \sqrt{\delta}\sqrt{d(d+N)}M\big).
    \$
    Note that $M\le N$. Taking $\delta = (d+N)^{-11}$, with probability at least $1-O((d+N)^{-10})$,
    $$
    \Big\|\sum_{i=1}^M \big(Z_i - \E Z_i\big)\Big\| = O\big(\big(\sqrt{Md} + \sqrt{Nd\lambda_1} + d \big)\cdot\log^3(d+N) \big). \qedhere
    $$
\end{proof}
\begin{proof}[Proof of Lemma \ref{lem:noise-sup}]
Let $Z_i = b_iv_i^\intercal$, where $\E b_i=\E v_i =0$ from Lemma \ref{lem:var-2} and thus $\E Z_i =0$. We first identify a truncation level of $\|Z_i\|$. 
Note that $\xi_{ij}$ is sub-gaussian from Lemma \ref{lem:sub-gaussian-xi}. By applying Lemma \ref{lem:sub-exp-bernstein} to bound the norms of $b_i = 2/\sqrt{n_i}\cdot\sum_{j=1}^{n_i/2}\xi_{ij} x_{ij}$ and $v_i=2/\sqrt{n_i}\cdot\sum_{j=n_i/2+1}^{n_i}\xi_{ij} x_{ij}$, there are $c_1$ and $c_2$ such that for any $\beta\ge0$,  
\$
\pr\big(\|Z_i\| \ge \beta\big) & \le \pr\big(\|b_i\| \|v_i\| \ge \beta\big) \le \pr\big(\|b_i\| \ge \sqrt{\beta}\big) + \pr\big(\|v_i\| \ge \sqrt{\beta}\big) \notag\\
& \le 4d\exp\big(- \min\big\{c_1 \beta/d, c_2 \sqrt{n_i\beta/d}\big\}\big). 
\$
For any $\delta >0$, we take $\beta = Cd\log^2(d/\delta)$ with a large enough $C$ in the above equation such that $\beta\ge C\max\{d\log(d/\delta ), d\log^2(d/\delta )/\min_i\{n_i\}\}$ and thus obtain
\#\label{eq:error-z}
\pr\big(\|Z_i\| \ge \beta\big) \le \delta .
\#

Next, for $\beta$ defined above, we establish bounds on the mean shift after truncation using Lemma \ref{lem:bound-mean-shift}. Here we substitute $\E[\|b_i\|^2]=\E[\|v_i\|^2]=O(d)$ from Lemma \ref{lem:var-2} to obtain,
    \#\label{eq:error-q}
    \big\|\E Z_i - \E[Z_i\mathds{1}\{\|Z_i\|<\beta\}] \big\| 
    \le \sqrt{\E\big[\|b_i\|^2\big]\E\big[\|v_i\|^2\big]}\sqrt{\delta }
    \le  O\big(\sqrt{\delta}d\big)\coloneqq q.
    \#

Then it remains to determine the variance statistic. By definition, we have $
\E[Z_iZ_i^\intercal] = \E[b_iv_i^\intercal v_ib_i^\intercal] = \E[\|v_i\|^2]\E[b_ib_i^\intercal]$.
Since $b_i$ and $v_i$ are identically distributed, it holds that $\E[Z_i^\intercal Z_i] = \E[v_ib_i^\intercal b_iv_i^\intercal ] = \E[\|b_i\|^2]\E[v_iv_i^\intercal] = \E[Z_i^\intercal Z_i]$. Then Lemma \ref{lem:var-2} yields
\$
\big\|\E[Z_iZ_i^\intercal]\big\| = \big\|\E[Z_i^\intercal Z_i]\big\| \le \E[\|v_i\|^2]\cdot\big\|\E[b_ib_i^\intercal]\big\| = O(d).
\$
Since $Z_i$ is mean zero, we further have
\#\label{eq:error-v}
v = \max\Big\{\Big\|\sum_{i=1}^M \E [Z_i Z_i^\intercal]\Big\|, \Big\|\sum_{i=1}^M \E [Z_i^\intercal Z_i]\Big\|\Big\} = O(Md).
\#
Applying Theorem \ref{cor:bernstein} with $c=11$, and $\beta$, $q$, and $v$ discussed in \eqref{eq:error-z}, \eqref{eq:error-q}, and \eqref{eq:error-v}, we obtain, with probability at least $1-2d^{-10} - M\delta$,
    \$
    \Big\|\sum_{i=1}^M \big(Z_i - \E Z_i\big)\Big\| &\le \sqrt{2cv\log d} + 2c\beta\log d/3 + Mq \\
    &= O\big(\sqrt{Md\log d} + d\log^2(d/\delta )\log d + \sqrt{\delta}Md\big).
    \$
    Note that $M\le N$. Taking $\delta = (d+N)^{-11}$, 
    with probability at least $1-O((d+N)^{-10})$,
    $$
    \Big\|\sum_{i=1}^M \big(Z_i - \E Z_i\big)\Big\| = O\big(\big(\sqrt{Md} + d \big)\cdot\log^3(d+N) \big). \qedhere
    $$
\end{proof}

\section{Proofs of the Remaining Results}
In this section, we prove the remaining results in the paper.
\subsection{Proofs of Propositions in Section \ref{subsect:estimator}}
We prove Proposition \ref{thm:me} by first fixing $B$ and optimizing for $\alpha_i$.
\medskip
\begin{proof}[Proof of Proposition \ref{thm:me}]
We first partially optimize for $\alpha_i$ given a fixed $B\in\cO^{d\times k}$. By taking partial derivatives and solving $\sum_{j=1}^{n_i}  B^\intercal(u_{ij} - B\widehat\alpha_i) = 0$ with $B^\intercal B = I_k$, we have
\$
\widehat\alpha_i = B^\intercal \overline u_i.
\$
Substituting the optimal $\widehat\alpha_i$ into the original problem, this leaves us to find $B\in\cO^{d\times k}$ to minimize:
\$
\sum_{i=1}^M\sum_{j=1}^{n_i}\big\|u_{ij} - B B^\intercal \overline u_i \big\|^2 & = \sum_{i=1}^M\sum_{j=1}^{n_i}\big\|u_{ij} - \overline u_i + \overline u_i - B B^\intercal \overline u_i \big\|^2 \\
& = \sum_{i=1}^M\sum_{j=1}^{n_i}\big\|u_{ij} - \overline u_i \big\|^2 + \sum_{i=1}^M n_i \big\|\overline u_i - B B^\intercal \overline u_i \big\|^2 \\
& = \sum_{i=1}^M\sum_{j=1}^{n_i}\big\|u_{ij} - \overline u_i \big\|^2 + \sum_{i=1}^M n_i \big\|\overline u_i \big\|^2 - \sum_{i=1}^M n_i (\overline u_i)^\intercal B B^\intercal \overline u_i, 
\$
where the second equality holds since $\sum_{j=1}^{n_i}(u_{ij} - \overline u_i)^\intercal (\overline u_i- B B^\intercal \overline u_i) = 0$ for $i\in[M]$ and the last equality holds due to $B^\intercal B = I_k$.
Thus, the least-squares problem is equivalent to the following one, 
\$
\max_{B\in\cO^{d\times k}} \sum_{i=1}^M n_i \overline u_i^\intercal B B^\intercal \overline u_i.
\$
In addition, we have
\$
\sum_{i=1}^M n_i \overline u_i^\intercal B B^\intercal \overline u_i = \sum_{i=1}^M n_i \tr\big(B^\intercal \overline u_i\overline u_i^\intercal B\big) = \tr\Big(B^\intercal \Big(\sum_{i=1}^M n_i \overline u_i\overline u_i^\intercal \Big) B\Big).
\$
We define $Z = \sum_{i=1}^M n_i \overline u_i\overline u_i^\intercal$. Solving the PCA problem $\max_{B\in \R^{d\times k}} \tr(BZ B)$ s.t. $B^\intercal B =I_k$, we obtain the optimal $\widetilde{B}$ as the top-$k$ eigenvectors of $Z$.    
\end{proof}

Similarly, we prove Proposition \ref{claim:app}.
\medskip
\begin{proof}[Proof of Proposition \ref{claim:app}]
We first partially optimize for $\alpha_i$ for a fixed $B\in\cO^{d\times k}$. 
By taking partial derivatives and solving $\sum_{j=1}^{n_i} B^\intercal\Gamma_i^{-1}x_{ij}(y_{ij} - x_{ij}^\intercal \Gamma_i^{-1} B\alpha_i) = 0$ with $B^\intercal B = I_k$, we have
\$
\widehat\alpha_i = \big(B^\intercal\Gamma_i^{-1}\widehat \Gamma_i \Gamma_i^{-1} B\big)^{\dag} B^\intercal\Gamma_i^{-1}\widehat z_i.
\$
Let $\Lambda_i = \Gamma_i^{-1} B\big(B^\intercal\Gamma_i^{-1}\widehat \Gamma_i \Gamma_i^{-1} B\big)^{\dag} B^\intercal\Gamma_i^{-1}$. Substituting the optimal $\widehat\alpha_i$ into the original problem, this leaves us to find $B\in\cO^{d\times k}$ to minimize:
\$
\sum_{i=1}^M\sum_{j=1}^{n_i}\big(y_{ij} - x_{ij}^\intercal \Gamma_i^{-1} B\widehat\alpha_i\big)^2 & = \sum_{i=1}^M\sum_{j=1}^{n_i}\big(y_{ij} - x_{ij}^\intercal \Lambda_i\widehat z_i\big)^2 \\
& = \sum_{i=1}^M\Big(\sum_{j=1}^{n_i} y_{ij}^2 - 2n_i\widehat z_i^\intercal \Lambda_i \widehat z_i+ n_i\widehat z_i^\intercal \Lambda_i \widehat\Gamma_i \Lambda_i\widehat z_i\Big) \\
& = \sum_{i=1}^M\Big(\sum_{j=1}^{n_i} y_{ij}^2 - n_i\widehat z_i^\intercal \Lambda_i \widehat z_i\Big),
\$
where the last equality holds since it is easy to compute that $\Lambda_i \widehat\Gamma_i \Lambda_i = \Lambda_i$. Thus, the least squares problem in \eqref{eq:least-squares} is equivalent to the following one,
$$
\max_{B\in\cO^{d\times k}} \sum_{i=1}^M n_i \widehat z_i^\intercal \Lambda_i \widehat z_i. \qedhere
$$
\end{proof}

We now compute $\E[\sum_{i=1}^M n_i \widehat z_i \widehat z_i^\intercal]$ in Proposition \ref{prop:x-gen}.
\medskip
\begin{proof}[Proof of Proposition \ref{prop:x-gen}]
We begin with $\E[\sum_{i=1}^M n_i \widehat z_i \widehat z_i^\intercal]$. By the definition of $\widehat z_i$, we have
\#\label{eq:e-gen}
\E \Big[\sum_{i=1}^M n_i \widehat z_i \widehat z_i^\intercal\Big] & = \E \Big[\sum_{i=1}^M n_i \Big(\frac{1}{n_i}\sum_{j=1}^{n_i}y_{ij} x_{ij}\Big) \Big(\frac{1}{n_i}\sum_{j=1}^{n_i}y_{ij} x_{ij}^\intercal\Big)\Big] \notag\\
& = \sum_{i=1}^M \frac{1}{n_i} \Big(\sum_{j_1\neq j_2}\E[y_{ij_1}x_{ij_1}]\E[ y_{ij_2} x_{ij_2}^\intercal ] + \sum_{j=1}^{n_i} \E[y_{ij}^2 x_{ij}x_{ij}^\intercal]\Big).
\#
For the first term, we have $\E[y_{ij_1}x_{ij_1}] = \E[x_{ij_1}x_{ij_1}^\intercal]\theta_i^\star = B^\star\alpha_i^\star$. It remains to compute the second term above. Since $\xi_{ij}$ and $x_{ij}$ are independent, we have
\$
\sum_{i=1}^M \frac{1}{n_i} \sum_{j=1}^{n_i} \E[y_{ij}^2 x_{ij}x_{ij}^\intercal] & = \sum_{i=1}^M\frac{1}{n_i} \sum_{j=1}^{n_i} \E[(x_{ij}^\intercal \theta_i^\star + \xi_{ij})^2 x_{ij}x_{ij}^\intercal] \\
& = \sum_{i=1}^M\frac{1}{n_i} \sum_{j=1}^{n_i} \Big(\E[x_{ij}^\intercal \theta_i^\star (\theta_i^\star)^\intercal x_{ij} x_{ij}x_{ij}^\intercal] + \E[\xi_{ij}^2] \E[x_{ij}x_{ij}^\intercal]\Big) \\
& = \sum_{i=1}^M\frac{1}{n_i} \sum_{j=1}^{n_i} \E[x_{ij}^\intercal \theta_i^\star (\theta_i^\star)^\intercal x_{ij} x_{ij}x_{ij}^\intercal] + \sum_{i=1}^M \frac{1}{n_i} \sum_{j=1}^{n_i} \E[\xi_{ij}^2]\Gamma_i.
\$
We conclude the proof by substituting these results into \eqref{eq:e-gen}.
\end{proof}

\subsection{Proof of Example \ref{lem:well-cond}}
Let $\|\cdot\|_{\psi_2}$ and $\|\cdot\|_{\psi_1}$ be the sub-gaussian and sub-exponential norms of random variables, respectively.
\medskip
\begin{proof}[Proof of Example \ref{lem:well-cond}]
Note that $\E D = I_k/k$ when $\E [\alpha_i\alpha_i^\intercal] = I_k/k$ for all $i$. We first use an $\varepsilon$-net argument to show that there exists a constant $c_1>0$ such that for any $t>0$, with probability at least $1-2\exp(-t^2)$,
\#\label{eq:D-eig}
\frac{N}{M} \big\| D- I_k/k\big\| \le c\max_{i\in[M]}n_i\cdot \big(\sqrt{k/M} + t/\sqrt{M}\big)/k,
\#
Let $\varepsilon = 1/4$. The covering number of the unit sphere $\mathbb{S}^{k-1}$ allows us to find an $\varepsilon$-net for the sphere, denoted by $\cU$, with cardinality $|\cU|\le 9^k$. With this $\varepsilon$-net, we can bound the matrix norm as follows,
\#\label{eq:inner-prod-main}
 \frac{N}{M} \big\| D- I_k/k\big\| &=  \Big\|\frac{1}{M} \sum_{i=1}^M n_i \alpha_i \alpha_i^\intercal - \frac{N}{Mk} I_k\Big\| = \max_{u\in \mathbb{S}^{k-1}}     \Big|\frac{1}{M}\sum_{i=1}^M n_i u^\intercal \alpha_i \alpha_i^\intercal u - \frac{N}{Mk}\Big| \notag\\
 & \le 2\max_{u\in \cU}     \Big|\frac{1}{M}\sum_{i=1}^M n_i u^\intercal \alpha_i \alpha_i^\intercal u - \frac{N}{Mk}\Big|.
\#    
Now we fix $u\in\cU$. Let $v_i = \sqrt{n_i}\alpha_i^\intercal u$. By assumption, $\sqrt{k}\alpha_i$ are independent, isotropic, and sub-gaussian random vectors with $\|\sqrt{k}\alpha_i\|_{\psi_2} \le c_2$ for a constant $c_2>0$. Thus, $v_i = \sqrt{n_i}\alpha_i^\intercal u$ are independent sub-gaussian random variables with $\E v_i^2 = n_i/k$ and $\|v_i\|_{\psi_2} \le c_1\sqrt{n_i/k}$. Therefore, $v_i^2 - n_i/k$ are independent, mean zero, and sub-exponential random variables with $\|v_i^2 - n_i/k\|_{\psi_1} \le w^2$, where $w^2 = c_3\max_{i\in[M]}n_i/k$ for a constant $c_3>0$. Let $\delta = C(\sqrt{k/M} + t/\sqrt{M})$ with a constant $C>0$ to be defined, and
$s = w^2\max\{\delta, \delta^2\}$. Then Bernstein's inequality implies that
\#\label{eq:bern1-s}
\pr\bigg(\Big|\frac{1}{M}\sum_{i=1}^M n_i u^\intercal \alpha_i \alpha_i^\intercal u - \frac{N}{Mk}\Big| \ge \frac{s}{2}\bigg) &= \pr\bigg(\Big|\frac{1}{M}\sum_{i=1}^M \Big(v_i^2 - \frac{n_i}{k}\Big)\Big| \ge \frac{s}{2}\bigg) \notag\\
& \le 2\exp\big[-c_4\min\big\{s^2/w^4, s/w^2\big\} M\big]. 
\#
We take $s$ as follows, with a sufficiently large constant $C$ to be chosen,
\# \label{eq:def-s}
s = Cw^2\big(\sqrt{k/M} + t/\sqrt{M}\big) = c_3C\max_{i\in[M]}n_i\cdot \big(\sqrt{k/M} + t/\sqrt{M}\big)/k,
\#
which implies that $\min\{s^2/w^4, s/w^2\} \ge C (k/M + t^2/M)$ since $k\le M$. Therefore, we have
\#\label{eq:bern1-2}
2\exp\big[-c_4\min\big\{s^2/w^4, s/w^2\big\} M\big] \le 2\exp\big[-c_4C(k + t^2)\big].
\#
Next, we take a union bound over $u\in\cU$. Recalling that $|\cU| \le 9^k$ and choosing the constant $C$ large enough, we obtain from \eqref{eq:bern1-s} and \eqref{eq:bern1-2} that
\$
\pr\bigg(\max_{u\in\cU}\Big|\frac{1}{M}\sum_{i=1}^M n_i u^\intercal \alpha_i \alpha_i^\intercal u - \frac{N}{Mk}\Big| \ge \frac{s}{2}\bigg) \le 9^k \cdot 2\exp\big[-c_4C(k + t^2)\big] \le 2\exp(-t^2).
\$ 
Combining the above with \eqref{eq:inner-prod-main}, we conclude the proof for \eqref{eq:D-eig}. The result in \eqref{eq:D-eig} further implies that with probability at least $1-2\exp(-t^2)$, we have $\lambda_1 \le 1/k + {Ms}/N$ and $\lambda_k \ge 1/k - {Ms}/N$. Thus, to satisfy the well-represented condition that there is a constant $C_1 > 0$ such that, 
\$
C_1\ge \frac{\lambda_1}{\lambda_k} = \frac{1/k + {Ms}/N}{1/k - {Ms}/N} = \frac{1 + {Mks}/N}{1 - {Mks}/N},
\$
we only need to ensure that $Mks/N \le c_5$ for a constant $c_5>0$. By the definition of $s$ in \eqref{eq:def-s}, this is equivalent to require
\$
\max_{i\in[M]} n_i\le c\frac{N}{M} \frac{1}{\sqrt{k/M} + t/\sqrt{M}}.
\$
We conclude the proof by choosing $t=\sqrt{k}$.
\end{proof}

\subsection{Proofs of Lemma \ref{lem:e-zi}}
We first prove Lemma \ref{lem:e-zi} under the assumption that $x_{ij}$ follows a standard Gaussian distribution.
\medskip
\begin{proof}[Proof of Lemma \ref{lem:e-zi}]
For fixed $i$ and $j$, we decompose the vector $x_{ij}$ into two orthogonal components, one in the subspace $B$ and the other in its orthogonal complement. Formally, we define $x_{ij}^B = BB^\intercal x_{ij}$ and $x_{ij}^{B_\perp} = (I-BB^\intercal)x_{ij}$ such that $x_{ij}^B + x_{ij}^{B_\perp} = x_{ij}$, $B^\intercal x_{ij}^B = B^\intercal x_{ij}$, and $(x_{ij}^B)^\intercal x_{ij}^{B_\perp} = 0$. It is easy to show that $\Cov(x_{ij}^B, x_{ij}^{B_\perp}) =0$, implying that the Gaussian vectors $x_{ij}^B$ and $x_{ij}^{B_\perp}$ are independent. Thus, given \eqref{model:glm}, we have
\$
    \E[y_{ij}x_{ij}] &= \E\big[\E[y_{ij}\given x_{ij}]\cdot x_{ij}\big] = \E\big[h_i(B^\intercal x_{ij})\cdot x_{ij}\big] \\
    & = \E\big[h_i(B^\intercal x_{ij}^B)\cdot (x_{ij}^B + x_{ij}^{B_\perp})\big] \\
    & = \E\big[h_i(B^\intercal x_{ij}^B)\cdot x_{ij}^B\big] \\
    & = BB^\intercal \E\big[h_i(B^\intercal x_{ij})x_{ij}\big],
\$
where the fourth equality follows from the independence of $x_{ij}^B$ and $x_{ij}^{B_\perp}$ and $\E[x_{ij}^{B_\perp}]=0$. Let $u = B^\intercal x_{ij}$ and note that $u\sim N(0, I_k)$. We conclude the proof by obtaining
\$
\E[y_{ij}x_{ij}] = BB^\intercal \E\big[h_i(B^\intercal x_{ij})x_{ij}\big] = B \E\big[h_i(u)u\big];
\$
thus $\E[\overline z_i] = 2/n_i \cdot \sum_{j=1}^{n_i/2}\E[y_{ij}x_{ij}] = B \E[h_i(u)u]$. A similar statement holds for $\E[\widetilde z_i]$.
\end{proof}

\end{appendix}

\end{document}